\documentclass[twoside]{article}

\usepackage[accepted]{aistats2026}
\usepackage[round]{natbib}

\usepackage{amsmath}
\usepackage{amssymb}
\usepackage{mathtools}
\usepackage{amsthm}
\usepackage{bm}
\usepackage{wrapfig} % add this in preamble
\usepackage{enumitem}
\usepackage{xcolor}
\usepackage{hyperref}
\usepackage[capitalize,noabbrev]{cleveref}

\definecolor{citecolor}{HTML}{2C4A73}
\hypersetup{
    colorlinks=true,
    citecolor=citecolor,
    linkcolor=citecolor,
    urlcolor=citecolor
}
\usepackage{booktabs} % for professional tables
\usepackage{multirow}
\usepackage{arydshln}

\usepackage{algorithm}
\usepackage{algpseudocode}

\newcommand\X{\boldsymbol{X}}

\newcommand\x{\boldsymbol{x}}

\DeclareMathOperator*{\argmin}{argmin}

\newcommand{\Xcal}{\mathcal{X}}
\newcommand{\Ycal}{\mathcal{Y}}

\newcommand{\Hcal}{\mathcal{H}}
\newcommand{\Gcal}{\mathcal{G}}
\newcommand{\Loss}{\ell}
\newcommand{\E}{\mathbb{E}}
\newcommand{\Prb}{\mathbb{P}}
\newcommand{\Ind}{\mathds{1}}
\newcommand{\ind}{\mathds{1}}
\newcommand{\R}{\mathbb{R}}
\newcommand{\KL}{\mathrm{KL}}
\newcommand{\Var}{\mathrm{Var}}

\newcommand{\Rcal}{\mathcal{R}}
\newcommand{\ip}[2]{\left\langle #1,\, #2\right\rangle}

\DeclareMathOperator{\Span}{span}
\newcommand{\disc}{\mathrm{disc}_{\Loss}}
\newcommand{\cX}{\mathcal{X}}

\newcommand{\cF}{\mathcal{F}}

\newcommand{\bbR}{\mathbb{R}}
\newcommand{\Risk}[2]{\epsilon_{#1}(#2)}

\newcommand{\Rk}{\mathbb{R}^k}
\newcommand{\dTV}{d_{\text{TV}}}
\usepackage{parskip}
\usepackage{dsfont}
\usepackage{mathrsfs}

\theoremstyle{plain}
\newtheorem{theorem}{Theorem}[section]
\newtheorem{proposition}[theorem]{Proposition}
\newtheorem{lemma}[theorem]{Lemma}
\newtheorem{corollary}[theorem]{Corollary}
\theoremstyle{definition}
\newtheorem{definition}[theorem]{Definition}
\newtheorem{assumption}[theorem]{Assumption}
\newtheorem{example}[theorem]{Example}
\theoremstyle{remark}
\newtheorem{remark}[theorem]{Remark}

\usepackage[textsize=tiny]{todonotes}

\usepackage{dblfloatfix}

\usepackage{tikz}
\definecolor{redd}{HTML}{b30326}
\definecolor{bloo}{HTML}{3a4cc0}

\definecolor{ours}{HTML}{219675}

\definecolor{circlefill}{HTML}{bce8da}
\newcommand*\annot[1]{\tikz[baseline=(char.base)]{
\node[shape=circle,draw=ours,fill=circlefill,minimum size=10pt,inner sep=0pt] (char) {#1};}}

\definecolor{kmm}{HTML}{2E5939} % Darker teal for KMM
\newcommand{\kmm}{\textcolor{kmm}{\textbf{KMM}}}

\definecolor{ours}{HTML}{1B7F5E} % Darker green for EPA
\newcommand{\epr}{\textcolor{ours}{\textbf{EPA}}}

\definecolor{seesc}{HTML}{C3752F} % Darker orange for SEES-C
\newcommand{\seesc}{\textcolor{seesc}{\textbf{SEES-C}}}

\definecolor{seesd}{HTML}{CC5C4A} % Darker peach for SEES-D
\newcommand{\seesd}{\textcolor{seesd}{\textbf{SEES-D}}}

\definecolor{covcolor}{HTML}{5E3D7F} % Darker purple for COV-CLF
\newcommand{\covariate}{\textcolor{covcolor}{\textbf{COV-CLF}}}

\definecolor{caluni}{HTML}{114D73} % Darker blue for CAL-UNI
\newcommand{\caluni}{\textcolor{caluni}{\textbf{SRC-UNI}}}

\definecolor{produni}{HTML}{7A1A1A} % Darker red for PROD-UNI
\newcommand{\produni}{\textcolor{produni}{\textbf{TAR-UNI}}}

\begin{document}

\runningtitle{EPA: Estimating, Explaining, and Improving Model
Performance Under Distribution Shift}


\twocolumn[

\aistatstitle{Entropic Projection Alignment: Estimating, Explaining, and Improving Model
Performance Under Distribution Shift}
\aistatsauthor{%
\parbox{\textwidth}{\centering
Salim I. Amoukou \quad
Emanuele Albini \quad
Tom Bewley \quad
Saumitra Mishra \quad
Manuela Veloso
}
}

\aistatsaddress{J.P. Morgan AI Research} 
]

\begin{abstract}
   We propose a unified framework for addressing three key challenges of distribution shift: \annot{1} estimating a model’s performance on an unlabeled target domain, \annot{2} explaining the shift by identifying the features responsible, and \annot{3} improving the target domain performance. Our method, Entropic Projection Alignment (\epr{}), aligns the source distribution to the target by matching carefully selected moments while simultaneously minimising the $\KL$ divergence from the source. This formulation yields a unique closed-form solution for importance weights, achieving robustness through implicit variance control. Drawing on domain adaptation theory, we establish that moment matching is sufficient for reliable estimation and adaptation, avoiding the need for full density ratio recovery. Extensive experiments, together with strong theoretical guarantees, demonstrate that \epr{} consistently outperforms state-of-the-art baselines while offering substantial computational efficiency.
\end{abstract}

\section{Introduction}

Machine learning models are often trained on data from a \textit{source} domain and deployed in a \textit{target} domain where the data distribution may differ. This discrepancy between training and deployment distributions can degrade model performance. This work considers a scenario where both features and labels are available in the source domain, but only unlabeled features are observed in the target domain. This mirrors applications such as medical diagnosis or credit scoring, where predictions involve future outcomes and immediate access to labels during deployment is not feasible.

When the target distribution can differ arbitrarily from the source distribution, estimating target performance or adapting the model becomes an ill-posed problem. To address this, the literature often imposes invariance assumptions about the nature of the distribution shift \citep{tasche2023invariance}. Common assumptions include changes in the marginal feature distribution $P_{\mathbf{X}}$ while the conditional label distribution $P_{Y|\mathbf{X}}$ remains unchanged, known as \textit{covariate shift} \citep{shimodaira2000improving}; shifts in the marginal label distribution $P_Y$ while $P_{\mathbf{X}|Y}$ remains constant, referred to as \textit{label shift} \citep{saerens2002adjusting, lipton2018detecting}; or changes in the joint distribution $P_{\mathbf{X}, Y}$ under the condition that there exists a subset of features $S$ such that $P_{\mathbf{X}_{\bar{S}} | \mathbf{X}_S, Y}$ is invariant, known as \textit{sparse joint shift} \citep{chen2022estimating, tasche2023invariance, tasche2023sparse}.

Instead of committing to these specific shift assumptions, we adopt the perspective of domain adaptation theory \citep{ben2006analysis, mansour2009domain}, which assumes only that there exists a hypothesis that performs reasonably well across both domains. In addition, we focus on machine learning models that process semantically meaningful features, such as those in tabular data, rather than unstructured data like images. Within this context, we introduce a unified framework that: \annot{1} estimates a model's performance on the unlabeled target domain; \annot{2} explains performance changes by identifying the features responsible under sparse joint shifts; and \annot{3} improves the model to enhance its target-domain performance. 

Our method transforms the source distribution by solving a constrained optimization problem: it matches key statistics of the target domain while minimizing the KL divergence from the source. This formulation yields a unique closed-form solution for the importance weights, which guarantees efficient computation and reduced estimation variance. The central theoretical insight is that accurate performance estimation and adaptation don't require recovering the full distribution ratio, but matching carefully chosen moments suffices. Our framework is supported by strong theoretical foundations and extensive empirical validation: across several distribution shifts, EPA achieves $30-70\%$ better target error estimation, more precise feature attribution, and superior model adaptation over state-of-the-art methods, all while running an order of magnitude faster.

 % Our approach extends the framework introduced by \citet{bachoc2023explaining}, originally developed for explainability. This method transforms a distribution by imposing constraints on certain statistics (e.g., feature means or correlations) while minimising the KL divergence from the original distribution. A key property of this optimisation problem is that, under mild assumptions, its solution is a reweighting of the original distribution, and the weights have an efficient and closed-form solution. We leverage this efficient framework to align statistics motivated by domain adaptation bounds. Crucially, this alignment is the key to addressing our three challenges.
\textbf{Related Works.} The work most closely related to ours is SEES \citep{chen2022estimating}, which also generates importance weights, estimates model performance, and identifies features driving sparse joint shifts. However, our approach diverges fundamentally in its objective. SEES attempts to directly learn the density ratio {\small \( w_S(x, y) = \frac{dQ}{dP}(\x_S, y) \)} by using a class of basis functions. In contrast, we circumvent the notoriously difficult task of density ratio estimation. Instead, we find the distribution closest to the source -- thereby reducing weight variance -- while satisfying selected statistics of the target domain, which ensures accuracy for downstream estimation and adaptation. This formulation is not only more robust but also yields a closed-form solution, making our method significantly faster. Furthermore, unlike SEES, our approach is not limited to classification tasks.

Our approach also shares conceptual similarities with kernel mean matching (KMM)  \citep{gretton2008covariate, sugiyama2012density}, which aligns feature means in an RKHS. KMM, however, requires solving a quadratic program whose complexity scales with the number of samples $n$, making it inefficient for large datasets. Our entropic projection approach scales with the feature dimension $k \ll n$. Moreover, the KL minimisation in our objective is equivalent to entropy maximisation, which provides implicit regularisation and prevents the unstable, high-variance weights that KMM can sometimes produce.

% While these prior works and our method both focus
% on tabular data, the underlying principles could, in theory, extend to unstructured domains such as images. However, applying our approach to such settings would require a non-trivial feature extraction step to obtain meaningful representations suitable for moment matching. Developing this extension is an
% important direction for future work, but beyond our
% current scope.

Finally, to our knowledge, no prior work has integrated all three aspects, estimating, explaining, and improving, into a single framework; most focus on only some. Refer to \ref{sec:related_works} for further discussion of prior works.

% Our objectives are as follows:
% \vspace{-0.2cm}
% \begin{itemize}
%     \item \annot{1} Estimate the unknown target domain error: {\small $\mathbb{E}_{\hat{Q}}[\mathcal{L}(f(\mathbf{X}), Y)].$}
%     % \vspace{-0.2cm}
%     \item \annot{2} Identify the features responsible for the performance differences between the source and target domain: {\small $\mathbb{E}_{\hat{P}}[\mathcal{L}(f(\mathbf{X}), Y)] \neq \mathbb{E}_{\hat{Q}}[\mathcal{L}(f(\mathbf{X}), Y)].$}
%     % \vspace{-0.2cm}
%     \item \annot{3} Adapt {\small \( f\) } to improve target domain performance via {\small \(\small\tilde{f} = \mathcal{A}(f, \hat{P}, \hat{Q}_\mathbf{X})\)}, where \(\small\mathcal{A}\) is an adaptation algorithm ensuring {\small $ \mathbb{E}_{\hat{Q}}[\mathcal{L}(\tilde{f}(\mathbf{X}), Y)] < \mathbb{E}_{\hat{Q}}[\mathcal{L}(f(\mathbf{X}), Y)].$}
% \end{itemize}

\vspace{-0.25cm}
\section{Entropic Projection Alignment}
\vspace{-0.1cm}
\textbf{Notations.} Let {\small$\small \mathcal{X}$} and {\small$\small\mathcal{Y}$} denote the input and label spaces. We consider a predictive model {\small\(\small f : \mathcal{X} \to \mathcal{Y} \)} from a hypothesis class $\Hcal$ and bounded loss function {\small\(\small \mathcal{L} : \mathcal{Y}^2 \to [0, 1] \)}. We are given a labeled dataset {\small\smash{$\small\mathcal{D}^{\text{src}} = \{(\mathbf{X}_i, Y_i)\}_{i=1}^{n}$}}, drawn i.i.d from a \textit{source} distribution {\small\(\small P = P_X \times P_{Y|X} \)}, and an unlabeled dataset {\small\smash{$\small \mathcal{D}^{\text{tar}} = \{\small\mathbf{X_i^{t}}\}_{i=1}^{m}$}} drawn from a \textit{target} distribution {\small\(\small Q = Q_X \times Q_{Y|X} \)}. We define the risk of a hypothesis $f \in \Hcal$ under distribution $P$ with respect to its true labeling function $c_P$ as  $\epsilon_P(f) := \E_{x \sim P_X}\big[\Loss(f(x), c_P(x))\big]$
and analogously for $Q$. We use the hat notation {\small$\small\hat{\cdot}$} to represent empirical distributions. 
% For example, the empirical distribution of the \textit{source} domain is defined as {\small\(\small \hat{P} = \frac{1}{n} \sum_{i=1}^n \delta_{\{\mathbf{X}_i, Y_i\}} \)}.

\epr{} casts alignment as an \emph{entropic projection} \citep{csiszar1984sanov} -- a classical optimisation problem that seeks the distribution closest in $\KL$ divergence to a reference distribution while satisfying linear constraints. This framework was recently used by \citet{bachoc2023explaining} for sensitivity analysis, where it addresses counterfactual questions such as: how would a model’s predictions change if the data distribution were projected to increase the mean or correlations of the features?

In contrast, we use entropic projection for a different goal: domain alignment. Our \epr{} method imposes constraints that align the source with the target distribution. Concretely, \epr{} seeks a distribution \( P^\star \) that is as close as possible to \( P \) while matching moment conditions derived from the target domain \( Q_{\X} \). 

\begin{definition}[\epr{} - empirical version]
\label{def:EPA-iproj}
Let $\Phi:\Xcal\to\R^k$ a feature map and $\widehat\mu_Q:=\E_{\hat Q_X}[\Phi(X)]$ be the empirical target moment. Assume  $\widehat\mu_Q$ is in the relative interior of the convex hull of $\Phi_n = \{\Phi(X_i)\}_{i=1}^n$, and  $\Phi_n^\intercal$ are linearly independent. \epr{} solves
\vspace{-0.2cm}
\begin{equation*}
\label{eq:EPA-iproj}
\hat{P}^\star\ \in\ \arg\min_{\tilde P}\ \KL\!\left(\tilde P\,\Vert\,\hat P\right)\quad\text{s.t.}\quad \E_{\tilde P}[\Phi(X)]\ =\ \widehat\mu_Q.
\end{equation*}
\end{definition}
\vspace{-0.4cm}
If $\widehat\mu_Q$ lies outside $\operatorname{conv}\{\Phi(X_i)\}_{i=1}^n$, the equality-constrained problem is infeasible. In that case, we project $\widehat\mu_Q$ onto the source convex hull and run \epr{} with the projected moment; \S\ref{sec:epa-outside-hull} shows that this is equivalent to imposing a hard proximity constraint around the target moment.

In our setting, both \emph{estimation} and \emph{explanation} revolve around finding the alignment (via the matching function $\Phi$) that most accurately estimates and explains the model’s performance. Then, we leverage the distribution $P^\star$, drawn from the best estimate, to \emph{improve} the model.

\begin{proposition}
\label{prop:exp-tilt}
Under the assumptions of Definition \ref{def:EPA-iproj}, the optimal distribution $P^\star$ is a reweighting of the source data, $\hat{P}^\star = \sum_{i=1}^n \lambda_i \delta_{\{\mathbf{X}_i, Y_i\}}$, with unique weights given by
\[
\lambda_i = \frac{\exp\{\ip{\xi}{\Phi(X_i)}\}}{\sum_{j=1}^n \exp\{\ip{\xi}{\Phi(X_j)}\}}, \qquad i=1,\dots,n,
\]
where $\xi \in \R^k$ is the unique minimizer of the strictly convex function:
\begin{equation*} \label{eq:loss}
    \mathcal{J}(\xi) = \log \left(\frac{1}{n} \sum_{i=1}^n e^{\ip{\Phi(\mathbf{X}_i)}{\xi}}\right) - \ip{\xi}{\E_{\hat{Q}_\mathbf{X}}[\Phi(\mathbf{X})]}.
\end{equation*}
\end{proposition}
\vspace{-0.3cm}
As noted by \citet{bachoc2023explaining}, this result follows as a corollary of the theorems in \citet{csiszar1984sanov}. In~\S\ref{proof:prop-epa}, we provide a simpler direct proof that avoids relying on those general results. Moreover, in ~\S\ref{proof:convergence} we establish new finite-sample concentration bounds comparing the empirical \epr{} in ~\S\ref{def:EPA-iproj} with its population analogue.

The \epr{} approach offers several advantages:  
\begin{enumerate}
    \item \emph{Uniqueness and implicit regularisation.} The solution is unique and corresponds to the maximum-entropy set of weights satisfying the constraints, since $\mathrm{KL}(\tilde P \,\Vert\, \hat P) = \log n - H(\lambda)$.  This encourages weights close to uniform, acting as a strong regulariser that penalises concentration and reduces the variance of importance-weighted estimators. A formal variance analysis is provided in \S\ref{proof:variance}.  

    \item \emph{Scalability.} The optimisation involves only \(k = \dim(\Phi)\) parameters, whereas KMM requires solving a quadratic program scaling with \(n\), which is prohibitive for large datasets.  

    \item \emph{Distribution-free formulation.} \epr{} requires no explicit assumptions on the underlying data distribution to approximate KL divergence, unlike prior approaches such as SEES.  
\end{enumerate}

% \vspace{-0.4cm}
\section{Estimating Model Performance} \label{sec:estimating}
% \vspace{-0.3cm}
The previous section introduced the \epr{} framework, which produces a reweighted source distribution $P^\star$ that matches the target moments of $\Phi$ under $Q$. The ultimate goal is to ensure that the risk on the reweighted source, $\epsilon_{P^\star}(f)$, is a reliable estimator for the target risk $\epsilon_Q(f)$. This raises the central question:
\begin{center}
\textit{Under what conditions does matching the moments of $\Phi$ lead to a small estimation gap $|\epsilon_Q(f) - \epsilon_{P^\star}(f)|$?}
\end{center}

To address this, we draw on the theory of domain adaptation, which provides formal tools for reasoning about performance under distribution shift. 

\begin{definition}[$\Loss$-Discrepancy]
For distributions $P^\star_X, Q_X$ on $\Xcal$, a hypothesis class $\Hcal$, and a bounded loss $\Loss \in [0,1]$, the \textbf{$\Loss$-discrepancy} is $\mathrm{disc}_\Loss(P^\star_X,Q_X) := \\ \sup_{f,g\in\Hcal} \left| \E_{Q_X}[\Loss(f(X),g(X))] - \E_{P^\star_X}[\Loss(f(X),g(X))] \right|.$
\end{definition}

This measure captures the maximum disagreement between the two distributions over the worst-case pair of hypotheses from $\Hcal$. A standard result \citep{mansour2009domain, ben2006analysis}, shown below, bounds the estimation gap using this discrepancy.

\begin{proposition}\label{prop:disc_gap}
For any hypothesis $h \in \Hcal$, assuming that the loss $\Loss$ satisfies the triangle inequality (e.g., 0-1 loss, absolute loss), the following inequality holds:
\[
\big|\epsilon_Q(h) - \epsilon_{P^\star}(h)\big| \le \mathrm{disc}_\Loss(P^\star_X,Q_X) + \lambda_{P^\star_,Q},
\]
where $\lambda_{P^\star,Q} := \inf_{f \in \Hcal} (\epsilon_{P^\star}(f) + \epsilon_Q(f))$ is the combined irreducible error of the best possible hypothesis on both distributions.
\end{proposition}

See derivations in~\ref{proof:disc_gap}. This bound clarifies our objective: to make the estimation gap small, we must choose our feature map $\Phi$ such that the resulting $P^\star$ has small discrepancy $\mathrm{disc}_\Loss(P^\star_X, Q_X)$. The following example shows that naively matching simple moments is not enough in general.

\begin{example}[Mismatch between Mean and Thresholds]
\label{ex:counterexample}
Let $\Xcal = \R$ and $\Phi(x)=x$, so that moment matching enforces $\E_{P^\star_X}[X] = \E_{Q_X}[X]$. Let $\Hcal$ be the class of threshold classifiers, $\Hcal = \{h_c(x) = \mathbf{1}\{x > c\} \mid c \in \R\}$. The associated discrepancy with binary 0-1 loss is $\mathrm{disc}_\Loss(P_X, Q_X) = 2\, \sup_{I \subseteq \R\ \text{interval}} \big| \Prb_{P^\star_X}(X \in I) - \Prb_{Q_X}(X \in I) \big|$. Consider the distributions $Q_X = \tfrac{1}{2}\delta_{-2} + \tfrac{1}{2}\delta_{+2}$ and $
P^\star_X = \tfrac{1}{2}\delta_{-1} + \tfrac{1}{2}\delta_{+1}.$
Both distributions have a mean of 0, satisfying the moment constraint. However, for the interval $I=(1.5, 2.5)$, we have $\Prb_{Q_X}(X\in I)=0.5$ while $\Prb_{P^\star_X}(X\in I)=0$. This implies $\sup_I |\Prb_{P^\star_X}(I)-\Prb_{Q_X}(I)| \ge 0.5$, and $\mathrm{disc}_\Loss(P^\star_X,Q_X) = 1$, rendering the bound vacuous. 
\end{example}

\subsection{Alignment}
The failure of the previous example demonstrates that the feature map $\Phi$ must be chosen in a way that is sensitive to the hypothesis class $\Hcal$ and loss function $\Loss$. We formalise this intuition with the concept of alignment. 

\begin{definition}[Disagreement Function Space]
Given a hypothesis class $\Hcal$ and a bounded loss $\Loss:\Ycal\times\Ycal\to[0,1]$, the \emph{disagreement function space} $\mathcal{G}_{\Hcal, \Loss}$ is defined as $\mathcal{G}_{\Hcal, \Loss} = \big\{ g:\Xcal \to [0,1] \,\big|\, \exists f,h \in \Hcal \text{ s.t. } g(x) = \Loss(f(x),h(x)) \big\}$.
For the binary $0$-$1$ loss, this space consists of indicator functions of the form $\{ \mathbf{1}\{f(x)\neq h(x)\} \mid f,h \in \Hcal \}$.
\end{definition}

\begin{definition}[Alignment]
\label{def:alignment}
A feature map $\Phi=(\phi_1,\dots,\phi_k):\Xcal \to \R^k$ is \textbf{aligned} with $(\Hcal, \Loss)$ if the span of the disagreement functions is contained within the affine span of the components of $\Phi$: $\Span(\mathcal{G}_{\Hcal, \Loss}) \subseteq \Span\big(\{1, \phi_1, \dots, \phi_k\}\big).$
Equivalently, for every $g \in \mathcal{G}_{\Hcal,\Loss}$, there exist $\beta_0\in\R$ and $\beta\in\R^k$ such that $g(x)=\beta_0+\beta^\top \Phi(x)$ for all $x \in \Xcal$.
\end{definition}

Intuitively, alignment means that $\Phi$ is expressive enough to exactly represent any disagreement pattern between hypotheses. When alignment holds, moment matching eliminates discrepancy entirely.

\begin{theorem}[Alignment Implies Zero Discrepancy]
\label{thm:alignment}
Let $P^\star_X$ and $Q_X$ such that $\E_{P^\star_X}[\Phi(X)] = \E_{Q_X}[\Phi(X)]$. If $\Phi$ is aligned with $(\Hcal,\Loss)$, then $
\mathrm{disc}_\Loss(P^\star_X, Q_X) = 0$. Consequently, for any loss covered by Proposition \ref{prop:disc_gap}, $\big|\epsilon_Q(h) - \epsilon_P^\star(h)\big| \le \lambda_{P^\star,Q}$, where $\lambda_{P^\star,Q}$ is small if a model exists that performs well on both domains.
\end{theorem}

The remaining term $\lambda_{P^\star,Q}$ captures the shared-good-hypothesis assumption. Under standard covariate or label shift, it is zero whenever the hypothesis class can represent the common labelling rule. If source and target are genuinely incompatible, then $\lambda_{P^\star,Q}$ is large, and neither EPA nor any other reweighting method can be expected to provide accurate target risk estimation or reliable adaptation. 

To illustrate how this abstract condition manifests in practice, we now present several concrete examples.
\newpage
\begin{example}[Alignment for Linear Models]
Consider linear predictors $f_w(x)=w^\top x$ on $x\in\R^d$ with squared loss $\Loss(\hat y,y)=(\hat y-y)^2$. A disagreement function is $g(x)=\Loss(f_{w_1}(x),f_{w_2}(x))=((w_1-w_2)^\top x)^2$, which is a quadratic form in $x$. The expectation of $g(x)$ depends on the second raw moments $\E[x_ix_j]$. If $\Phi$ includes all degree-2 monomials, e.g., $\Phi(x)=\big(x_1,\dots,x_d,\,x_1^2,\,x_1x_2,\dots,x_d^2\big)^\top,$
then every $g \in \mathcal{G}_{\Hcal, \Loss}$ is an affine function of $\Phi(x)$, alignment holds, and discrepancy vanishes under moment matching.
\end{example}

\begin{remark}
For linear models with a squared loss, \citet{mansour2009domain} obtain the reweighted source distribution that minimises $\disc$ via a semi-definite program. Under our assumption that $\hat{\mu}_Q$ lies in the convex hull of $\Phi_n$, however, no such optimization is required.
\end{remark}

\begin{example}[Alignment for Covariate Shift]
    Under covariate shift, suppose the distributional change is driven entirely by the prevalence of a subgroup $s(X) \in \{0, 1\}$.  
If the conditional risk is approximately constant within the subgroup and its complement, then choosing $\Phi(X) = s(X)$ provides effective control of the estimation gap.

\end{example}

See proofs and additional examples in \ref{proof:align}. Our focus is on tabular data, where tree-based models are state-of-the-art and thus the models we use in our experiments. Because these models can be represented as piecewise-constant functions, they also admit alignment, as shown below.

\begin{theorem}[Alignment for Decision Trees]\label{theo:decisiontree}
Let $\Hcal$ be the class of decision trees on $\Xcal$ and let $\Loss$ be the 0-1 loss. Let $\Rcal_\Hcal$ be the set of all possible leaf regions generated by any tree in $\Hcal$. Then the feature map $\Phi(x) = (\mathbf{1}\{x \in R\})_{R \in \Rcal_\Hcal}$ is aligned with $(\Hcal, \Loss)$.
\end{theorem}

See proof in \ref{proof:align_decision}. This result shows that, in principle, exact alignment can be achieved for decision trees: the indicator features $\mathbf{1}\{x \in R\}$ span the full disagreement space. In practice, however, such a construction is infeasible as $\Rcal_\Hcal$ can be extremely large, and most regions will be empty.

To make the approach tractable, we turn to approximations. Instead of explicitly representing all tree regions, we approximate the disagreement space using quantile binning of the feature space. Concretely, motivated by the examples and our empirical findings, we use the following general-purpose matching function:
\begin{equation}
\Phi(\X)=\big(\Psi_1(X^{(1)}),\ldots,\Psi_d(X^{(d)}),\,\Psi_f(f(\X))\big),
\end{equation}
where each $\Psi_j$ is a low-cardinality quantile binning. We use $10$ bins in all our experiments. This approach has proven empirically effective for most common tabular models and datasets. See \ref{disc:binning} for further discussion of alternative binning strategies for decision trees that may reduce approximation gap.

\begin{remark}
   The key takeaway is that one does not need to recover the full density ratio -- a notoriously difficult task. Instead, it is enough to match the moments of a $\Phi$ that captures model disagreement. The quality of performance estimation depends on how well $\Phi$ approximates the disagreement space (see \ref{sec:approx_align} for approximate alignment analysis).
\end{remark}

\begin{remark}
Although our results are stated for population distributions ($P^\star$ and $Q$), the derivations are purely algebraic and therefore apply to any pair of distributions, including the empirical ones $\hat{P}^\star$ and $\hat{Q}$. This is a crucial feature of the framework: given an aligned feature map $\Phi$, performance on a finite target sample $\hat{Q}$ can be estimated simply by matching empirical moments between $\hat{P}^\star$ and $\hat{Q}$. The resulting gap between the empirical risk $\epsilon_{\hat{P}^\star}$ and the population risk $\epsilon_{Q}$ is governed only by the sampling deviation between $\hat{\mu}_{Q}$ and $\mu_{Q}$ (See \ref{lem:concentration} for derivations). Thus, unlike standard analyses that require generalisation bounds on distributional discrepancy (e.g., Corollary~7 in \citet{mansour2009domain}), our setting reduces entirely to moment matching.
\end{remark}

\section{Explaining Model Performance} \label{sec:explaining}
% \vspace{-0.3cm}
Understanding differences between source and target domains is crucial for providing guidelines to monitor similar changes or to anticipate further shifts. This is often approached by learning a constrained map {\small\( T \)} that transforms the source data to the target data {\small $T(\hat{P}_{\X}) = \hat{Q}_{\X}.$}, e.g., via optimal transport \citep{koebler2023towards}.
% or a causal graph \cite{budhathoki2021did, zhang2210did}.

A challenge with this problem is identifiability: without assumptions of invariance between the source and target data or knowledge about the causal structure, there are infinitely many plausible mappings {\small\( T \)}, making it impossible to identify which one generated the given target data. Therefore, we need to specify certain invariance assumptions to make this problem identifiable. Unlike the estimating and improving components, the explaining aspect of our methodology focuses specifically on explaining sparse shifts, which assumes there exists a subset {\small\( S \subset \{1, \dots, d\} \)} such that {\small$P(\mathbf{X}_{\bar{S}} | \mathbf{X}_S, Y) = Q(\mathbf{X}_{\bar{S}} | \mathbf{X}_S, Y).$}
%\vspace{-0.35cm}
Notably, the sparse joint shift generalizes both classical label shift and sparse covariate shift. See \citet{tasche2023sparse, tasche2023invariance} for a comprehensive analysis of these shifts. In this context, explaining the causes involves identifying the subset of features responsible for generating the shift \citep{chen2022estimating}.
\begin{proposition}[Identification of Sparse Shifts] \label{prop:sparse_shift}
Let a sparse shift from a source $P$ to a target $Q$ be generated by a unique minimal set $S^*$. For any candidate subset $S$, define a reweighted source distribution $P'_S$ via the density ratio restricted to the variables in $S$ (and $Y$ for joint shift): $w_S^*(\x_S) = q(\x_S)/p(\x_S)$ or $w_S^*(\x_S, y) = q(\x_S, y)/p(\x_S, y)$. The generating set $S^*$ is the unique set of minimal cardinality such that: $S^* = \argmin_{S \subseteq \{1, \dots, d\}} \{|S| \mid \KL(P'_S \,\|\, Q) = 0\}.$
% \begin{equation*}
    
% \end{equation*}
\end{proposition}
% \vspace{-0.2cm}
See derivation in~\ref{proof:sparse_shift}. Our approach is to operationalise this idea using \epr{}. We iteratively test candidate feature subsets. For each subset, we use \epr{} to perform a histogram matching, reweighting the source data to align the distribution of those specific features with the target's. We then select the subset that yields the best overall approximation of the target data, as measured by an estimated KL distance. Specifically, among all subsets within $p\%$ (default $p=5$) of the minimum, we then select the subset with the smallest size. We can also combine feature matchings with the target variable to capture shifts involving both inputs and labels. Refer to \S\ref{algo:kl_divergence} and \S\ref{algo:shift_explanation} for a step-by-step description of histogram matching and details on \epr{}'s explanations.

\section{Improving Model Performance} \label{sec:improving_model_performance}

To enhance the model’s performance under the target distribution $Q$, we propose to leverage the reweighted data generated by the matching $\Phi$ used for target error estimation, i.e., matching the binned version of all features and model prediction.

The use of reweighting to improve model performance has been extensively studied. \citet{shimodaira2000improving} demonstrated that, for parametric models, reweighting can be beneficial when the model is misspecified (e.g., using a linear model when the true relationship is quadratic), but offers no advantage when the model is well-specified. However, recent works such as \citet{byrd2019effect} and \citet{zhai2022understanding} report negative results, indicating that reweighting does not outperform classical training with uniform weights for high-capacity models. In contrast, \citet{gogolashvili2023importance} argue that these negative findings are due to evaluations conducted on well-specified cases, and they show that reweighting can be effective even for high-capacity models if the model is misspecified.

Additionally, it is well-known that directly using the exact ratio of source and target densities as weights may not be optimal in practice. \citet{shimodaira2000improving} discusses this issue in detail, highlighting that while such weights are asymptotically unbiased, they can suffer from high variance in practice. For example, these weights may disproportionately emphasize a small subset of the data, potentially leading to poor model. This phenomenon reflects the classic bias-variance trade-off: accepting a small bias can lead to more stable and effective learning. Our reweighting technique addresses this trade-off by ensuring that the reweighted distribution remains close to the original distribution, thereby preventing the weights from concentrating  on few points while still matching key statistics of the target data useful for estimation and adaptation.

When learning with reweighted data, most prior work focuses on training a completely new model \citep{byrd2019effect, zhai2022understanding, gogolashvili2023importance}. In contrast, we propose adding a corrective term to the existing model \( f \). We find this approach more effective than retraining a new model from scratch under reweighting, as it builds upon the valuable information already captured by \( f \). Specifically, we employ a boosting procedure to minimize the error under the reweighted distribution \( P^\star \).

Starting with \( f_0 = f \), we iteratively update the model as $f_t = f_{t-1} + \alpha_t h_t,$ where at each iteration \( t \),
\( h_t \in \mathcal{H} \) is a base learner chosen from a class of functions $\mathcal{H}$ to minimize the weighted loss $h_t = \arg\min_{h \in \mathcal{H}} \sum_{i=1}^n \lambda_i \, \mathcal{L}\big( f_{t-1}(\mathbf{X}_i) + h(\mathbf{X}_i), \, Y_i \big),$ and \( \alpha_t \) is a learning rate determined by: $\alpha_t = \arg\min_{\alpha} \sum_{i=1}^n \lambda_i \, \mathcal{L}\big( f_{t-1}(\mathbf{X}_i) + \alpha h_t(\mathbf{X}_i), \, Y_i \big).$ After \( M \) iterations, we obtain the improved model \(\tilde{f} = f_M\), which is expected to perform better under \( P^\star \) and, consequently, in the target domain $Q$. This procedure can be efficiently implemented using gradient boosting libraries like XGBoost \citep{chen2016xgboost}, which natively support instance weights. 

The following theorem provides a formal guarantee: if the performance gain from boosting on the reweighted source data exceeds a certain threshold, the adapted model is provably better than the baseline on the target distribution.

\begin{theorem} \label{theo:improvement}
Let $\tilde{f}$ be the boosted model and define the empirical performance gain as $\Delta_{\mathrm{emp}} = \epsilon_{\hat{P}^\star}(f) - \epsilon_{\hat{P}^\star}(\tilde f)$. If
\begin{equation*}
    \Delta_{\mathrm{emp}} > 2 (\mathrm{disc}_\mathcal{L}(\hat{P}^\star_X, \hat{Q}_X) + \lambda_{\hat{P}^\star,\hat{Q}}),
\end{equation*} then $\epsilon_{\hat{Q}}(\tilde f) \le \epsilon_{\hat{Q}}(f)$. Furthermore, if $\Delta_{\mathrm{emp}}$ also exceeds a constant $B_{m,k,\delta}$ (see details in \ref{proof:improvement_theo}), then with probability at least $1-\delta$, we have
\begin{equation*}
    \epsilon_{Q}(\tilde f) < \epsilon_{Q}(f).
\end{equation*}
\end{theorem}
% \vspace{-0.2cm}

Our matching $\Phi$ is designed to make the discrepancy term $\mathrm{disc}_\Loss(P^\star_X, Q_X)$ small, while boosting decreases the reweighted error $\epsilon_{\hat{P}^\star}(\tilde f)$. Theorem~\ref{theo:improvement} then ensures these two effects combine into a provable improvement guarantee over the baseline model $f$. See proof in \ref{proof:improvement_theo}.

% \vspace{-0.4cm}
\section{Experimental Setup}
% \vspace{-0.3cm}
In this section, we detail the experimental setup used to determine which method -- relying only on source data and target domain features -- yields the best weights across three core objectives: \annot{1} estimating the target domain performance, \annot{2} identifying the features responsible for performance changes in sparse shift scenarios, and \annot{3} adapting the model to the target data.

We use three shift settings: \textit{sparse covariate}, \textit{sparse joint} and \textit{natural shifts}. For each case, we have source and target datasets. We split the source data into a training set, used to train the initial model \( f \), and a calibration set, which is made accessible to each method for weight estimation and model adaptation. Similarly, we split the target data into an evaluation and calibration sets.

\textbf{Sparse Covariate Shifts:}  We use the Adult \citep{uci_dataset}, HELOC \citep{helocdata}, and NHANES I \citep{nhanes} datasets. To create different source and target domains, we split the data based on the values of one feature at a time. For each continuous feature, we split the data at its median value.
For each categorical feature, we exclude data belonging to one category.
In either case, splitting results in two subsets, \( A \) and \( B \). We then randomly select 10\% of \( B \) and add it to \( A \) to form our \textit{source data}, while the remaining 90\% of \( B \) constitutes our \textit{target data}. 

For each dataset, the number of distribution shifts studied equals the number of features times the number of splits: two for continuous features (since we split at the median) and one for each category of discrete features. This setup results in \textbf{292 different distribution shifts}. A similar methodology was used in prior work \citep{chen2022estimating} on a smaller scale with selected features, whereas we evaluate all possible feature combinations, aligning with recent tabular shift benchmarks \citep{gardner2023benchmarking}.

\textbf{Sparse Joint Shifts:}
In this setting, both features and labels shift. We apply the same subset-splitting strategy from the covariate-shift setup but additionally modify the label distribution by randomly undersampling the source label mean to $0.2$ and oversampling the target label mean to $0.5$.

\textbf{Natural Shifts:}
We use the Folktables dataset \citep{folktables}, which comprises U.S.\ Census data (50 states plus Puerto Rico) from multiple years (2014-2018), capturing both geographic and temporal shifts. The predictive task is to determine whether income exceeds \$50{,}000.

As with the previous datasets, we define subsets $A$ and $B$ as different state-year combinations, thereby leveraging the inherent distribution shifts across time and location.

For further details on the datasets, such as size and feature count, see \S\ref{sec:dataset_details}.

Given our focus on tabular data, we utilize XGBoost \cite{chen2016xgboost}, a well-established and high-performing model for this type of data. We perform extensive fine-tuning to ensure optimal performance at every step involving training. Details of the tuning process can be found in  \S\ref{app:hyperparameters}.

In \ref{sec:epa_algorithms}, we give a detailed algorithmic description of all components of \epr{}. We use the following baselines for comparison:

\begin{itemize}
    \item \seesc{} and \seesd{}: The two versions of the method introduced by \citet{chen2022estimating}. SEES-C uses continuous features, and SEES-D uses a discretized version of the features to estimate the sample weights. We use the original implementation available on \href{https://github.com/stanford-futuredata/SparseJointShift}{Github}.
    \item \kmm{} \citep{cortes2008sample, gretton2008covariate} using a linear kernel to match the mean of all features.
    \item \textbf{Domain Classifier Weighting (\covariate)}: A commonly used method \citep{sugiyama2012density} for covariate shift that trains a classifier to distinguish between source and target domains. The predicted probabilities \( \hat{p}(\mathbf{X}) \) are used to derive sample weights as \( \hat{p}(\mathbf{X}) / (1 - \hat{p}(\mathbf{X})) \). We use XGBoost as classifier.
\end{itemize}
\vspace{-0.2cm}

\section{Experiments: Estimating}
\begin{table*}[!t]
\centering
\caption{Performance across different shift scenarios. 
\textbf{Tasks:} Estimation, Explanation, Improvement.}
\vspace{0.2cm}
\scriptsize
\setlength{\tabcolsep}{1.5pt}
\renewcommand{\arraystretch}{1.0}
\resizebox{\textwidth}{!}{
\begin{tabular}{ll|cccc|c|c|cccc|c|c|cccc|c|c}
\toprule
& & \multicolumn{6}{c|}{\textbf{Sparse Covariate}} & \multicolumn{6}{c|}{\textbf{Sparse Joint}} & \multicolumn{6}{c}{\textbf{Natural}} \\
\cmidrule(lr){3-8} \cmidrule(lr){9-14} \cmidrule(lr){15-20}
Task & Method 
& B1 & B2 & B3 & B4 & Power/Prop & FP 
& B1 & B2 & B3 & B4 & Power/Prop & FP
& B1 & B2 & B3 & B4 & Power/Prop & FP \\
\midrule
\multirow{5}{*}{\rotatebox{90}{Estimation$\downarrow$}} 
& \epr & \textbf{0.007} & \textbf{0.010} & \textbf{0.014} & \textbf{0.009} & 0.765 & 0.056 
        & \textbf{0.031} & \textbf{0.034} & \textbf{0.034} & \textbf{0.064} & \textbf{1.000} & 0.000 
        & \textbf{0.047} & \textbf{0.049} & \textbf{0.055} & \textbf{0.045} & \textbf{0.978} & 0.000 \\
& \kmm & 0.014 & 0.020 & 0.158 & 0.111 & 0.009 & \textbf{0.000} 
      & 0.172 & 0.130 & 0.227 & 0.310 & 0.442 & 0.000 
      & -- & -- & -- & -- & -- & -- \\
& \seesd & 0.016 & 0.022 & 0.053 & 0.044 & \textbf{0.791} & 0.178 
        & 0.084 & 0.048 & 0.141 & 0.235 & 0.976 & 0.000 
        & 0.119 & 0.116 & 0.163 & 0.160 & 0.717 & 0.000 \\
& \seesc & 0.014 & 0.016 & 0.012 & 0.017 & 0.696 & 0.100 
        & 0.057 & 0.029 & 0.127 & 0.222 & 0.976 & 0.000 
        & 0.134 & 0.135 & 0.186 & 0.323 & 0.607 & 0.000 \\
& \covariate & 0.015 & 0.015 & 0.028 & 0.026 & 0.757 & 0.171 
            & 0.128 & 0.098 & 0.203 & 0.249 & 0.935 & 0.000 
            & 0.107 & 0.098 & 0.167 & 0.281 & 0.835 & 0.000 \\
\midrule
\multirow{3}{*}{\rotatebox{90}{Explan.}} 
& \epr & \textbf{0.70/0.75} & \textbf{0.71/0.82} & \textbf{1.00/1.00} & \textbf{0.86/0.98} & -- & -- 
      & \textbf{0.65/0.58} & \textbf{0.86}/0.73 & \textbf{1.00/1.00} & 0.63/\textbf{0.58} & -- & -- 
      & -- & -- & -- & -- & -- & -- \\
& \seesd & 0.28/0.35 & 0.29/0.35 & 0.50/0.59 & 0.86/0.83 & -- & -- 
        & 0.25/0.20 & 0.23/0.09 & 0.56/0.50 & 0.25/0.18 & -- & -- 
        & -- & -- & -- & -- & -- & -- \\
& \seesc & 0.47/0.69 & 0.60/0.81 & 1.00/1.00 & 0.62/\textbf{0.98} & -- & -- 
        & 0.61/0.40 & 0.83/\textbf{0.82} & \textbf{1.00/1.00} & \textbf{0.87}/0.19 & -- & -- 
        & -- & -- & -- & -- & -- & -- \\
\midrule
\multirow{7}{*}{\rotatebox{90}{Improvement$\uparrow$}} 
& \epr & \underline{13.86} & 11.89 & \underline{10.71} & 24.48 & \underline{0.979} & -- 
      & \underline{17.65} & \underline{17.74} & \underline{8.90} & \underline{6.21} & \underline{0.983} & -- 
      & \underline{13.67} & \underline{13.90} & \underline{20.07} & \underline{42.27} & 0.981 & -- \\
& \kmm & 13.52 & \textbf{13.09} & 9.77 & 23.56 & \textbf{0.997} & -- 
      & 6.50 & 17.21 & 8.20 & 4.25 & 0.973 & -- 
      & -- & -- & -- & -- & -- & -- \\
& \seesd & \textbf{14.05} & 12.03 & 9.15 & 23.52 & 0.976 & -- 
        & 16.59 & 17.12 & 8.41 & 4.40 & 0.973 & -- 
        & 9.68 & 12.46 & 19.27 & 41.23 & 0.879 & -- \\
& \seesc & 13.72 & \underline{12.22} & 9.28 & 24.51 & 0.969 & -- 
        & 16.20 & 17.07 & 7.82 & 4.10 & 0.969 & -- 
        & 9.91 & 11.59 & 20.76 & 42.08 & 0.780 & -- \\
& \covariate & 13.80 & 12.09 & 8.45 & \underline{24.66} & \underline{0.979} & -- 
            & 16.50 & 16.56 & 8.32 & 4.20 & 0.966 & -- 
            & 12.85 & 13.09 & 10.28 & 36.85 & 0.978 & -- \\
\cmidrule[0.05pt](lr){2-20}  % Thinner line to indicate separation
& \caluni & 14.01 & 12.05 & 9.27 & 24.17 & 0.969 & -- 
         & 16.55 & 16.92 & 8.45 & 4.18 & 0.973 & -- 
         & 12.76 & 12.90 & 10.22 & 37.40 & \underline{0.984} & -- \\
& \produni & 13.37 & 11.02 & \textbf{11.93} & \textbf{26.57} & \underline{0.979} & -- 
          & \textbf{37.27} & \textbf{49.14} & \textbf{42.07} & \textbf{49.65} & \textbf{1.000} & -- 
          & \textbf{24.87} & \textbf{27.43} & \textbf{34.68} & \textbf{50.62} & \textbf{1.000} & -- \\
\bottomrule
\end{tabular}\label{results}
}
\small \\
\vspace{0.1cm}
\parbox{\textwidth}{\tiny 
\textbf{Bin ranges:} Sparse Covariate: B1=(-0.149,-0.075], B2=(-0.075,0], B3=(0,0.106], B4=(0.106,0.211]; 
Sparse Joint: B1=(0.082,0.146], B2=(0.146,0.209], B3=(0.209,0.272], B4=(0.272,0.336]; 
Natural: B1=(0.042,0.141], B2=(0.141,0.24], B3=(0.24,0.338], B4=(0.338,0.437].
\textbf{Notes:} Estimation shows inaccuracy (lower is better); Explanation shows Acc/Corr pairs; Improvement shows average improvement (higher is better). 
Power/Prop column shows Power for Estimation and I-prop for Improvement. FP column shows false positives for Estimation only. 
Bold indicates best performance and underline second best performance within each task-shift. -- indicates unavailable data. Note that \caluni{} and \produni{} are not baselines but serve as control group}
\end{table*}

We evaluate each method's ability to accurately estimate model performance on target domains and detect harmful distribution shifts using the following metrics:
\begin{enumerate}
    \item \emph{Estimation Inaccuracy}: Mean absolute error between estimated and true target errors (lower is better; $\downarrow$). True errors are defined as the absolute difference between predicted probabilities and true labels.   The results are organized into increasing bins (B1-B4), which reflect the difference between the model error in the target data and the source data. A positive difference indicates a \emph{harmful} shift; otherwise, the shift is considered \emph{benign}.
    \item \emph{Detection Capability}: Assesses the method's ability to detect when a harmful shift occurs—that is, when the target error exceeds the source error. Power (proportion of harmful shifts correctly detected; higher is better $\uparrow$) and False Positive rate (FP; lower is better $\downarrow$).
\end{enumerate}
 
\textbf{Takeaway} \annot{1}: As shown in Table~\ref{results}, \epr{} significantly outperforms the baselines in both estimation inaccuracy and detection capability across all shift scenarios, achieving 30–70\% better performance, particularly in the more harmful bins (B3, B4). It is followed by \seesc{}, \seesd{}, and then \covariate{}. \kmm{} is the poorest-performing method. In the natural shift setting, which reflects real-world data variability, \epr{} maintains strong accuracy and the highest power, while other methods degrade more noticeably. Due to scalability issues with larger datasets, \kmm{} is omitted under natural shifts.

% \vspace{-0.1cm}
\section{Experiments: Explaining}
Next, we evaluate each method’s explanation performance, focusing on \emph{sparse covariate} and \emph{sparse joint} shifts. In contrast to the natural shift scenario, these cases have well-defined ground truth explanations, as they correspond to the splitting features used to generate the shifts, allowing straightforward evaluation. Among the methods considered, only our approach (\epr{}) and the two baselines (\seesc{} and \seesd{}) provide feature-level explanations.

We measure the average detection rate of the features responsible for the shifts (\textit{Acc}). To account for potential biases arising from correlated features, we also compute the correlation (\textit{Corr}), which quantifies the relationship between the selected features and the ground truth features across the dataset (see details in  \S\ref{app:details}). 

% It is important to note that while \textit{Acc} = 1 implies \textit{Corr} = 1, the reverse is not necessarily true.

\textbf{Takeaway} \annot{2}. As shown in Table~\ref{results}, \epr{}  outperforms the baselines across most bins and shift types, with up to 60\% higher accuracy under sparse joint shifts. \seesc{} generally outperforms \seesd{}.
% \vspace{-0.3cm}
\section{Experiments: Improving}

We now turn to adaptation effectiveness, which measures how effectively the reweighted source data from each method can be used to adapt the model and improve its performance on the target domain. We report two relative improvement metrics to isolate the contribution of the reweighting scheme from the strenght of the starting model:

\begin{itemize}
    \item \textit{Average Improvement}: The relative percentage difference in performance between the adapted model and the original model on the target data (higher is better; $\uparrow$). Results are averaged across bins representing different performance shift intensities.
    \item \textit{Improvement Proportion (I-prop)}: The proportion of shifts in which the adapted model outperforms the original model ranges from 0 to 1 (higher is better; $\uparrow$).
\end{itemize}

For model adaptation comparison, we include two additional baselines as \textbf{control methods}:

\begin{itemize}
    \item \textbf{Source Calibration with Uniform Weights} (\caluni): Uses the source calibration set with uniform weights as a naive approach without reweighting.
    \item \textbf{Target Calibration with Uniform Weights} (\produni): Uses the target calibration set with uniform weights as an idealized or strong baseline as it leverages ground truth labels from the target domain. Note that the target calibration set is not used for evaluation but comes from the same distribution as the one used for evaluation. It provides a reference for what can be achived with the same calibration budget when target labels are available.
\end{itemize}

Table~\ref{results} summarizes the results across three shift types: Sparse Covariate, Sparse Joint, and Natural shifts. A horizontal line separates the control baselines {\small (\caluni, \produni)} from the actual baselines.

\subsection{Sparse Covariate Shifts}

\textbf{Benign Shifts (Negative Bins)} When distribution shifts benefit performance, most methods perform similarly, with differences in average improvement typically under $0.5\%$ for bin B1.

\textbf{Harmful Shifts (Positive Bins)} When shifts become harmful, \produni{} (which leverages labeled target data) achieves the highest improvements. Among methods using only unlabeled target data, \epr{} stands out in the moderate bin $B_3=(0, 0.106]$ with improvements close to \produni. In the highest bin $B_4=(0.106, 0.211]$, all methods show comparable results, with \covariate{} having a slight edge in average improvement.

Interestingly, \caluni{} (the naive approach) also performs well, supporting the findings from \citep{gogolashvili2023importance} that uniform weights may often suffice when the model is well specified. Next, we examine performance when the model is misspecified, where naive training provides less benefit, allowing a clearer comparison of each method's reweighting effectiveness.

\begin{remark}[Model Misspecification]
    To isolate the benefits of reweighting, we conducted a study with a misspecified linear model, a scenario where the model class cannot easily approximate the true data-generating process \citep{shimodaira2000improving}. The results show: naive fine-tuning with uniform weights (\caluni{}) failed, yielding negligible or even negative gains on harmful shifts. In contrast, all importance reweighting methods delivered substantial improvements, demonstrating their ability to compensate for the model-data mismatch. Among them, \textbf{\epr{} provided the most consistent and significant gains} (see Table~\ref{app:linear_model}), underscoring the critical role of effective reweighting for robust adaptation, especially when a model is misspecified.
\end{remark}

\subsection{Sparse Joint and Natural Shifts}

Under \emph{sparse joint shifts}, both the features and the labels change, creating more pronounced and inherently harmful distribution shifts even for a strong model like XGBoost. As expected, \produni{} achieves the highest performance gains by leveraging labeled target data. Among methods that only use unlabeled target data, \epr{} achieves the largest overall improvements, followed closely by \seesd{}, \seesc{}, and \covariate{}.

 For \emph{natural shifts}, \produni{} again demonstrates superior performance, especially under severe shifts. \epr{} maintains its position as the strongest method among those not using target labels across all bins, while \seesd{}, \seesc{}, perform similarly, with \covariate{} showing slightly better results.

% \vspace{-0.1cm}
\textbf{Takeaway} \annot{3}: When using boosting approaches, even simple uniform weighting on a new independent source dataset can yield modest improvements as the algorithm naturally adapts. However, as distributional shifts become more pronounced relative to the model class, our results demonstrate that reweighting becomes particularly beneficial, aligning with findings from \cite{gogolashvili2023importance}. Among the methods using only unlabelled target data, \epr{} performs best overall across diverse types of shifts, especially when the shift becomes harmful. Another observation from these experiments is that even suboptimal weights for estimating model performance can still lead to substantial model improvement.
% \vspace{-0.3cm}
\section{Ablation Studies}
Appendices~\ref{app:stability}-\ref{app:target_data_estimation} present additional ablation experiments. These experiments demonstrate that the \emph{correction} or \emph{boosting} strategies (outlined in Section~\ref{sec:improving_model_performance}) consistently outperform \emph{learning from scratch} under the studied shifts. We analyze which method generates reweighted data that aligns most closely with the target data. We conduct simulations demonstrating \epr{}’s stability advantage over \kmm{} in importance weight estimation. Further experiments are included on linear models and regression problems.
% \vspace{-0.4cm}
\section{Complexity Analysis and Running Time}
We compare the computational costs of computing the sample weights for our method (\epr{}) against the direct baselines: \seesc{} and \seesd{}. In terms of complexity, \epr{} requires estimating a vector proportional to the number of features by minimizing the convex loss in~\eqref{eq:loss}. In contrast, the baselines aim to learn basis functions to approximate the density ratio $w_S(x, y) = \frac{dQ}{dP}(x_S, y)$. This approach involves a number of parameters dependent on the size of the basis class function, which can often exceed the dimensionality of the feature space.

Table~\ref{tab:runtime} provides the observed runtimes for each method during sparse-shift experiments. These results show that \epr{} is significantly faster than its closest baselines, particularly the best baseline \seesc{}, which requires an order of magnitude more time.  Notably, \covariate{}, which simply trains a domain classifier to distinguish between source and target labels and converts the predictions into sample weights, is the quickest, requiring only 2 seconds at most. 

\begin{table}[!ht]
\centering
\caption{Runtime per method (in seconds), reporting the \emph{mean}, standard deviation (\emph{Std}), minimum (\emph{Min}), and maximum (\emph{Max}) across sparse-shift experiments on a MacBook Pro M2.}
\vspace{0.2cm}
\footnotesize % Use a smaller font size for compactness
\setlength{\tabcolsep}{8pt} % Reduce horizontal padding for compactness
\renewcommand{\arraystretch}{1.0} % Adjust vertical spacing for compactness
\begin{tabular}{lcccc}
\toprule
\textbf{Metrics} & \emph{Mean} & \emph{Std} & \emph{Min} & \emph{Max} \\
\midrule
\epr{}        & \underline{4.6}   & \underline{2.3}   & \underline{1.4}   & \underline{10.1}    \\
\kmm & 238 & 320 & 1.0 & 927 \\
\seesd{}      & 17.3  & 10.6  & 1.8   & 37.6    \\
\seesc{}      & 448 &  452  & 30 & 2026  \\
\covariate{}  & \textbf{2.4}   & \textbf{1.2}   & \textbf{0.5}   & \textbf{4.4}     \\
\bottomrule
\end{tabular}
\label{tab:runtime}
\end{table}

\section{Conclusion}
We introduced \epr{}, an entropic projection-based framework to address three core challenges of distribution shift: \annot{1} estimating target domain performance, \annot{2} identifying the features driving the shift, and \annot{3} adapting models to improve target domain performance, all using only unlabeled target data and the source data. Empirical evaluations across various shifts demonstrate that \epr{} outperforms state-of-the-art baselines while providing improved computational efficiency. 

A key element of our framework is the choice of a feature map $\Phi$ that is well aligned with the model disagreement class under study. Future work could extend this framework beyond the tabular setting to unstructured domains, such as images, and to richer model classes, such as deep neural networks, where alignment may be achieved through learned representations. This direction is especially relevant for cases where explicit mappings are difficult or intractable, such as tree-based models. Another promising avenue is to employ hypothesis-specific discrepancy bounds rather than worst-case bounds over the entire function class, which could lead to tighter guarantees and improved adaptability.

% \newpage

\section*{Disclaimer}
This paper was prepared for informational
purposes by the Artificial Intelligence Research group of
JPMorgan Chase \& Co., and its affiliates (``JP Morgan'')
and is not a product of the Research Department of JP
Morgan. JP Morgan makes no representation, and warranty
whatsoever, and disclaims all liability, for the completeness,
accuracy or reliability of the information contained herein.
This document is not intended as investment research or
investment advice, or a recommendation, offer or solicitation for the purchase or sale of any security, financial
instrument, financial product or service, or to be used in
any way for evaluating the merits of participating in any transaction, and shall not constitute a solicitation under any jurisdiction or to any person, if such solicitation under such jurisdiction or to such person would be unlawful.

\textit{© 2026 JP Morgan Chase \& Co. All rights reserved.}

% \subsubsection*{Acknowledgements}
% All acknowledgments go at the end of the paper, including thanks to reviewers who gave useful comments, to colleagues who contributed to the ideas, and to funding agencies and corporate sponsors that provided financial support. 
% To preserve the anonymity, please include acknowledgments \emph{only} in the camera-ready papers. The acknowledgements do not count against the 9-page page limit in the camera-ready.

% \subsubsection*{References}

% References follow the acknowledgements. Use an unnumbered third level
% heading for the references section.  Please use the same font
% size for references as for the body of the paper---remember that
% references do not count against your page length total.

% \begin{thebibliography}{}
% \setlength{\itemindent}{-\leftmargin}
% \makeatletter\renewcommand{\@biblabel}[1]{}\makeatother
% \bibitem{} J.~Alspector, B.~Gupta, and R.~B.~Allen (1989).
%     \newblock Performance of a stochastic learning microchip.
%     \newblock In D. S. Touretzky (ed.),
%     \textit{Advances in Neural Information Processing Systems 1}, 748--760.
%     San Mateo, Calif.: Morgan Kaufmann.

% \bibitem{} F.~Rosenblatt (1962).
%     \newblock \textit{Principles of Neurodynamics.}
%     \newblock Washington, D.C.: Spartan Books.

% \bibitem{} G.~Tesauro (1989).
%     \newblock Neurogammon wins computer Olympiad.
%     \newblock \textit{Neural Computation} \textbf{1}(3):321--323.
% \end{thebibliography}
\bibliography{bib}

%%%%%%%%%%%%%%%%%%%%%%%%%%%%%%%%%%%%%%%%%%%%%%%%%%%%%%%%%%%%
\section*{Checklist}

\begin{enumerate}

  \item For all models and algorithms presented, check if you include:
  \begin{enumerate}
    \item A clear description of the mathematical setting, assumptions, algorithm, and/or model. [Yes. Proofs and detailed statements of the theorems, including their assumptions, are provided in the Appendix. In Section~\ref{sec:epa_algorithms}, we give a detailed algorithmic description of all components of \epr{}.]
    \item An analysis of the properties and complexity (time, space, sample size) of any algorithm. [Yes]
    \item (Optional) Anonymized source code, with specification of all dependencies, including external libraries. [Code will be released after acceptance]
  \end{enumerate}

  \item For any theoretical claim, check if you include:
  \begin{enumerate}
    \item Statements of the full set of assumptions of all theoretical results. [Yes]
    \item Complete proofs of all theoretical results. [Yes]
    \item Clear explanations of any assumptions. [Yes]     
  \end{enumerate}

  \item For all figures and tables that present empirical results, check if you include:
  \begin{enumerate}
    \item The code, data, and instructions needed to reproduce the main experimental results (either in the supplemental material or as a URL). [Code to reproduce all experiments will be released after acceptance]
    \item All the training details (e.g., data splits, hyperparameters, how they were chosen). [Yes. Additional details can be found in Appendix, and code to reproduce all experiments will be released after acceptance]
    \item A clear definition of the specific measure or statistics and error bars (e.g., with respect to the random seed after running experiments multiple times). [Yes. We report mean, standard deviation, min, and max for runtime statistics related to weight computation (Table 2). However, due to the extensive hyperparameter tuning detailed in Appendix~\ref{app:hyperparameters} and the large number of shift scenarios evaluated (292 for sparse covariate shifts alone), computing comprehensive error bars (e.g., via multiple runs with full tuning) for the primary performance metrics in Tables 1 was computationally prohibitive: each full evaluation (with hyperparameter tuning) took approximately one week per method. However, we ensured robustness by evaluating a wide range of settings and reporting aggregated results across multiples shifts and several datasets.]
    \item A description of the computing infrastructure used. (e.g., type of GPUs, internal cluster, or cloud provider). [Yes. The weight computation experiments were conducted on a MacBook Pro M2, as noted in the caption of Table 2. This figure provides mean, std, min, and max runtimes. Other experiments were run on AWS c5.4xlarge (16 vCPU, 32 GB RAM).]
  \end{enumerate}

  \item If you are using existing assets (e.g., code, data, models) or curating/releasing new assets, check if you include:
  \begin{enumerate}
    \item Citations of the creator If your work uses existing assets. [Yes]
    \item The license information of the assets, if applicable. [Not Applicable]
    \item New assets either in the supplemental material or as a URL, if applicable. [Not Applicable]
    \item Information about consent from data providers/curators. [Not Applicable]
    \item Discussion of sensible content if applicable, e.g., personally identifiable information or offensive content. [Not Applicable]
  \end{enumerate}

  \item If you used crowdsourcing or conducted research with human subjects, check if you include:
  \begin{enumerate}
    \item The full text of instructions given to participants and screenshots. [Not Applicable]
    \item Descriptions of potential participant risks, with links to Institutional Review Board (IRB) approvals if applicable. [Not Applicable]
    \item The estimated hourly wage paid to participants and the total amount spent on participant compensation. [Not Applicable]
  \end{enumerate}

\end{enumerate}

\clearpage
\appendix
\thispagestyle{empty}

% Supplementary material: To improve readability, you must use a single-column format for the supplementary material.
\onecolumn
\aistatstitle{Supplementary Materials}

\section*{Appendix: Table of Contents}
\begin{description}
    \item \ref{app:stability}. \textbf{Stability: \epr{} vs \kmm{}}
    \begin{itemize}
        \item \textbf{Highlight:} \epr{} produces sample weights with lower variance and closer to the uniform distribution than \kmm{}.
    \end{itemize}
    \item \ref{app:scratch}. \textbf{Learning from Scratch vs. Correcting the Existing Model}
    \begin{itemize}
        \item Comparison of training new models from scratch vs. using correction terms.
        \item \textbf{Highlight:} Correction approach outperforms training from scratch.
    \end{itemize}
    
    \item \ref{app:linear_model}. \textbf{Experiments: Estimation and Improvement using a Linear Base}
    \begin{itemize}
        \item Evaluation of methods under sparse covariate shifts, joint shifts, and natural shifts using a Linear model.
        \item \textbf{Highlight:} \epr{} excels in estimation accurary; reweighting improves performance significantly over the traditional uniform weighting of the calibration set (\caluni), in contrast to the case with tree base for the XGBoost model.
    \end{itemize}
    
    \item \ref{app:regression_exp}. \textbf{Experiments on Regression Model}
    \begin{itemize}
        \item Evaluation on the California Housing Pricing dataset \citep{uci_dataset} using a linear base.
        \item \textbf{Highlight:} While the baselines \seesd{}, and \seesc{} are not available for regression,  \epr{} significantly outperforms \covariate{} in estimation and improvement.
    \end{itemize}

    \item \ref{app:target_data_estimation}. \textbf{Target Data Estimation}
    \begin{itemize}
        \item Comparison of methods in reweighting source data to approximate target data.
        \item \textbf{Highlight:} \epr{} consistently achieves the lowest mean and histogram difference norms.
    \end{itemize}

    \item \ref{app:hyperparameters}. \textbf{Hyperparameter Optimization for XGBoost Model}
    \begin{itemize}
        \item \textbf{Highlight:} Details of hyperparameter tuning strategies for XGBoost. We use the same strategy for all parts of our methodology requiring learning a model.
    \end{itemize}
    \item \ref{app:details}. \textbf{Additional experiments details }
    \item \ref{sec:related_works} \textbf{Additional related works}
     \item \ref{sec:epa_algorithms}. \textbf{Algorithmic Description of \epr{}}
     \item \ref{theo:epa_a}-\ref{disc:binning} \textbf{corresponds to all the theoretical analyses and proofs.}
\end{description}

\newpage
\section{Stability: \epr{} vs \kmm{}} \label{app:stability}

In this section, we present a simulation to demonstrate the advantages of the \epr{} framework over \kmm{}, complementing our earlier experiments on real-world datasets. As discussed, the regularization in \epr{} reduces variance in importance weight estimation. To validate this, we generate datasets multiple times under controlled conditions.

We define two distributions, \( P \sim \mathcal{N}(0_p, 5) \) and \( Q \sim \mathcal{N}(2\times I_p, 1) \), and sample 1000 points from each. The target output is modeled as \( Y = f(X) = \sum_{i=1}^{p} X_i \), with \( p = 100 \).

Both \kmm{} and \epr{} are used to estimate importance weights, and we fit a linear model using the reweighted samples. Note that the optimal model remains \( f \), even after the distribution shift. Performance is evaluated by the Relative Absolute Difference with respect to the optimal model by 
\[
\text{Relative Absolute Difference} = \frac{\text{MSE}(\text{linear\_reweighted\_method}) - \text{MSE}(\text{true\_model})}{\text{MSE}(\text{true\_model})}.
\]

The experiment is repeated 500 times, and the results, illustrated in figure \ref{fig:stability}, highlight that \epr{} achieves zero variance, consistently recovering the true model, while \kmm{} exhibits high variance. Moreover, in figure \ref{fig:entropy} we show the distribution of the KL distance between the sample weights of each method and the uniform weights. It shows that, as expected, the \epr{} weights are closer to the uniform distribution than those estimated by \kmm{}.

\begin{figure*}[h]
    \centering
    \begin{minipage}{0.3\linewidth}
        \centering
        \includegraphics[width=\linewidth]{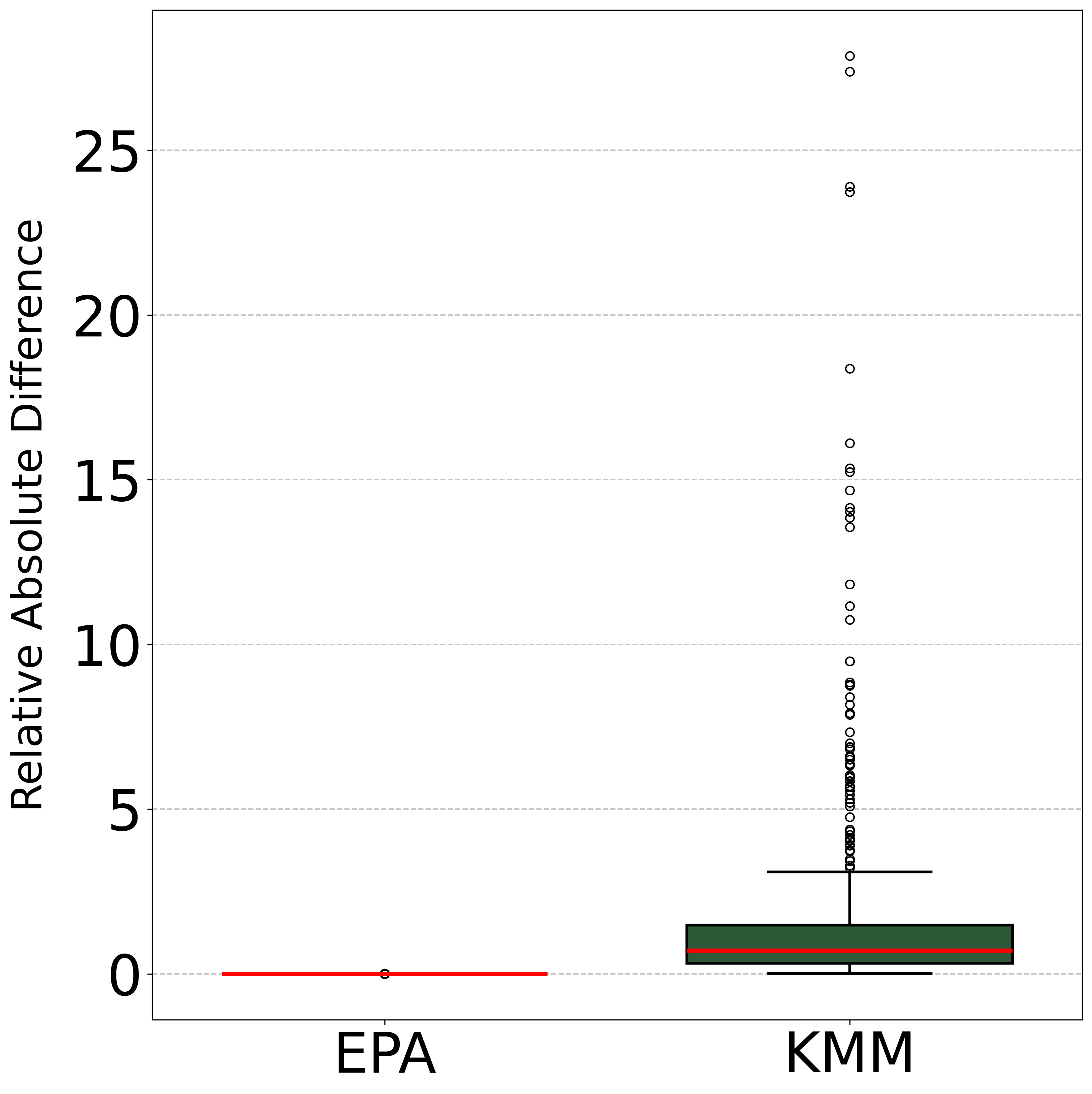}
        \caption{Comparison of each reweighted method’s MSE with the true model’s MSE. \epr{}’s variance is negligible, while \kmm{}’s variance is large.}
        \label{fig:stability}
    \end{minipage}
    \hspace{0.05\linewidth}  % Adjust spacing between images
    \begin{minipage}{0.3\linewidth}
        \centering
        \includegraphics[width=\linewidth]{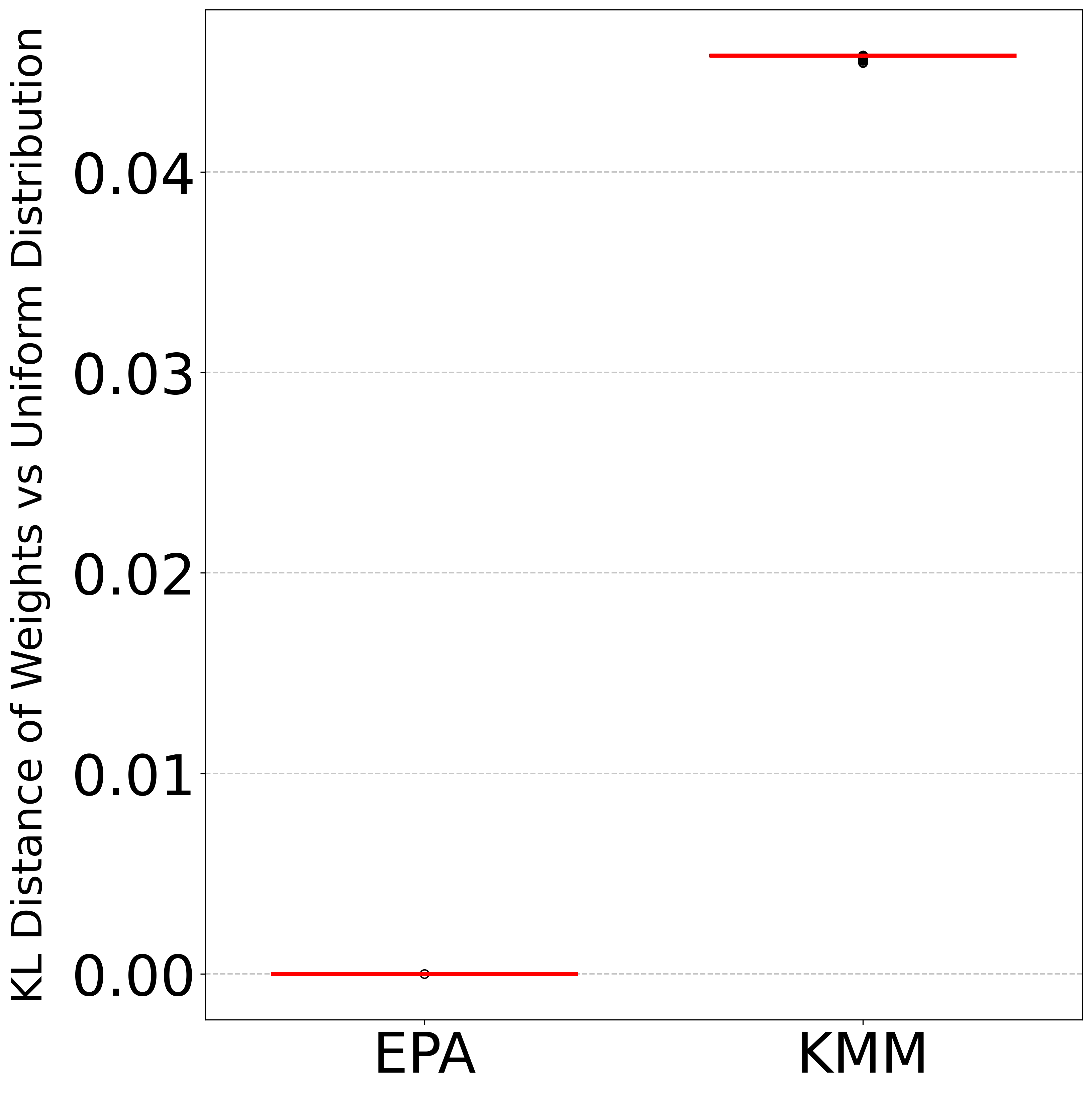}
        \caption{Distribution of the KL distance between sample weights and uniform weights. \epr{} is closer to uniform, as expected.}
        \label{fig:entropy}
    \end{minipage}
\end{figure*}

\newpage

\section{Learning from Scratch vs. Correcting the Existing Model} \label{app:scratch}

We evaluate the effectiveness of training a new model from scratch compared to our proposed method of adding corrective terms to an existing model (as described in Section~\ref{sec:improving_model_performance}). Table~\ref{tab:average_improvements_sc_bins} presents the results. When comparing these findings with Table~\ref{results} (where correction was applied), we observe that training from scratch is significantly less effective across all experiments.

In the “learning from scratch” setting, most methods fail to improve performance and even degrade performance in most cases. The best-performing approach in this case is calibration with uniform weights (\caluni), which produces results similar to the original model. This outcome is expected, as the calibration data follows the same distribution as the data used to train the original model.

This pattern holds across all other experiments. These findings underscore the limitations of training a new model on reweighted data in sparse shifts. In contrast, our correction approach effectively utilizes reweighted data to achieve better adaptation and performance improvements.

\begin{table}[!ht]
\centering
\caption{Model improvement metrics for Sparse Covariate Shift when Learning a new model instead of correction on the reweighted data.}
\vspace{0.2cm}
\small
\setlength{\tabcolsep}{2pt} % Reduce horizontal padding for compactness
\renewcommand{\arraystretch}{1.0} % Adjust vertical spacing for compactness
% \resizebox{\columnwidth}{!}{ % Scale table to fit a single column
\begin{tabular}{lccccc}
\toprule
\textbf{Metrics} & \multicolumn{4}{c}{\emph{Average Improvement} ($\uparrow$)} & \emph{I-prop} ($\uparrow$) \\
\midrule
\textbf{Bins} & (-0.091, -0.045] & (-0.045, 0] & (0, 0.153) & (0.153, 0.305) & All-Bins \\
\midrule
\epr                    &-64.623         &-62.970        &-54.696        &-51.008&  -63.225\\
\seesd                   & -58.976         &-64.096        &-61.102        &-28.211&   -60.621 \\
\seesc                   &-62.311         &-59.761        &-54.342        &-32.398&   -60.086  \\
\covariate        &-67.654         &-61.408        &-55.903        &-82.166 &-64.797  \\
\produni                    &-4.520          &-3.448          &3.130         &12.302&     -3.366 \\
\caluni                   & -0.755           &0.088         &-0.593          &1.825&   -0.29  \\
\bottomrule
\end{tabular}
% }
\label{tab:average_improvements_sc_bins}
\end{table}

\section{Experiments: Estimation and Improvement using a Linear base}  \label{app:linear_model}

This section evaluates the performance of the methods in terms of estimation and improvement under sparse covariate shifts, sparse joint shifts, and natural shifts using a linear base model. This setup highlights scenarios where the model class is underperforming, and reweighting is known to help \citep{gogolashvili2023importance} in mitigating the performance degradation.

\subsection{Sparse Covariate Shift: Linear Model}

\paragraph{Key Observations:}
\begin{itemize}
    \item \textbf{Estimation Inaccuracy (Table~\ref{tab:estimation_errors_bins_linear}):}
    \begin{itemize}
        \item \epr{} achieves the best estimation accuracy across all bins, with significantly lower errors compared to other methods. For detection power, \epr{} significantly outperforms baselines, achieving 0.848 compared to approximately 0.14 for others.
        \item Baselines \seesc{}, \seesd{}, and \covariate{} struggle, particularly in cases of large performance shifts.
    \end{itemize}
    \item \textbf{Model Improvement (Table~\ref{tab:average_improvements_bins_linear}):}
    \begin{itemize}
        \item While most methods perform similarly in lower shift bins, \epr{} demonstrates the highest average improvement in bins with the largest shifts.
        \item Uniform calibration (\caluni) fails to deliver meaningful improvements, often resulting in negligible or negative performance gains.
    \end{itemize}
\end{itemize}

\begin{table}[!ht]
\centering
\caption{Estimation metrics for sparse covariate shifts using linear models.}
\vspace{0.2cm}
\small
\setlength{\tabcolsep}{2pt} % Reduce horizontal padding for compactness
\renewcommand{\arraystretch}{1.0} % Adjust vertical spacing for compactness
% \resizebox{\columnwidth}{!}{ % Scale table to fit a single column
\begin{tabular}{lccccccc}
\toprule
\textbf{Metrics} & \multicolumn{4}{c}{\emph{Estimation Inaccuracy} ($\downarrow$)} & \emph{Power} ($\uparrow$) & \emph{FP} ($\downarrow$) \\
\midrule
\textbf{Bins} & (-0.091, -0.045] & (-0.045, 0] & (0, 0.153) & (0.153, 0.305)] & All-Bins & All-Bins \\
\midrule
\epr               &0.004           &0.005          &0.004           &0.015 &  0.848 & 0\\
\kmm   & 0.012          & 0.024         & 0.067           & 0.230 & 0. & 0. \\
\seesd             &0.048           &0.068          &0.025           &0.189 & 0.141 & 0\\
\seesc             &0.054           &0.066          &0.023           &0.242 & 0.141 &  0\\
\covariate         &0.012           &0.024          &0.063           &0.217 &0.130 & 0\\
\bottomrule
\end{tabular}
% }
\label{tab:estimation_errors_bins_linear}
\end{table}

\begin{table}[!ht]
\centering
\caption{Model improvement metrics for sparse covariate shift using linear models.}
\vspace{0.2cm}
\small
\setlength{\tabcolsep}{2pt} % Reduce horizontal padding for compactness
\renewcommand{\arraystretch}{1.0} % Adjust vertical spacing for compactness
% \resizebox{\columnwidth}{!}{ % Scale table to fit a single column
\begin{tabular}{lccccc}
\toprule
\textbf{Metrics} & \multicolumn{4}{c}{\emph{Average Improvement} ($\uparrow$)} & \emph{I-prop} ($\uparrow$) \\
\midrule
\textbf{Bins} & (-0.091, -0.045] & (-0.045, 0] & (0, 0.153) & (0.153, 0.305) & All-Bins \\
\midrule
\epr                     &56.662          &32.590         & \textbf{40.587}          & \textbf{25.072} & \textbf{1.00}  \\
\kmm & 56.645          &32.586         &\underline{39.592}          &\underline{23.067} & 0.97\\
\seesd                  &56.664          & \textbf{35.914}         &39.171          &21.187 & 0.712  \\
\seesc                  &\underline{56.679}          &34.569         &39.353          &21.138 &  0.801 \\
\covariate       & \textbf{56.690}          &\underline{35.908}         &39.175          &21.124  & 0.712  \\
\produni                  &33.243           &1.044          &6.061          &-8.091  & 0.805  \\
\caluni                   & 4.380           &1.764         & 2.241          &-0.502  & 0.928  \\
\bottomrule
\end{tabular}
% }
\label{tab:average_improvements_bins_linear}
\end{table}

\newpage
\subsection{Sparse Joint Shift: Linear Model}

\paragraph{Key Observations:}
\begin{itemize}
    \item \textbf{Estimation Inaccuracy (Table~\ref{tab:estimation_errors_sjs_bins_linear}):}
    \begin{itemize}
        \item \epr{} achieves near-perfect accuracy with negligible errors across all bins.
        \item Baselines show large errors and low detection power (near $0$).
    \end{itemize}
    \item \textbf{Model Improvement (Table~\ref{tab:average_improvements_sjs_bins_linear}):}
    \begin{itemize}
        \item All methods demonstrate similar improvement levels, but \epr{} lags slightly in improvement proportion (\textit{I-prop}) compared to the baselines.
        \item \caluni{} continues to underperform, offering limited or negligible improvements.
    \end{itemize}
\end{itemize}

\begin{table}[!ht]
\centering
\caption{Estimation metrics for sparse joint shifts using linear models.}
\vspace{0.2cm}
\small
\setlength{\tabcolsep}{2pt} % Reduce horizontal padding for compactness
\renewcommand{\arraystretch}{1.0} % Adjust vertical spacing for compactness
% \resizebox{\columnwidth}{!}{ % Scale table to fit a single column
\begin{tabular}{lccccccc}
\toprule
\textbf{Metrics} & \multicolumn{4}{c}{\emph{Estimation Inaccuracy} ($\downarrow$)} & \emph{Power} ($\uparrow$) & \emph{FP} ($\downarrow$) \\
\midrule
\textbf{Bins} & (0.159, 0.167) & (0.167, 0.174) & (0.174, 0.182) &  (0.182, 0.189) & All-Bins & All-Bins \\
\midrule
\epr               &0.000            &0.000            &0.000             &0.001 & 1.00 & 0\\
\seesd             &0.235           &0.240           &0.225            &0.241 & 0.120 & 0\\
\seesc             &0.222           &0.242           &0.216            &0.210 & 0.027 &  0\\
\covariate         &0.175           &0.183           &0.169           &0.160 &0.045 & 0\\

\bottomrule
\end{tabular}
% }
\label{tab:estimation_errors_sjs_bins_linear}
\end{table}

\begin{table}[!ht]
\centering
\caption{Model improvement metrics for sparse joint shift using linear models.}
\vspace{0.2cm}
\small
\setlength{\tabcolsep}{2pt} % Reduce horizontal padding for compactness
\renewcommand{\arraystretch}{1.0} % Adjust vertical spacing for compactness
% \resizebox{\columnwidth}{!}{ % Scale table to fit a single column
\begin{tabular}{lccccc}
\toprule
\textbf{Metrics} & \multicolumn{4}{c}{\emph{Average Improvement} ($\uparrow$)} & \emph{I-prop} ($\uparrow$) \\
\midrule
\textbf{Bins}  & (0.159, 0.167) & (0.167, 0.174) & (0.174, 0.182) &  (0.182, 0.189) & All-Bins \\
\midrule
\epr                      &42.263          &39.622          &45.551          &39.964 & 0.712  \\
\seesd                   &42.273         & 39.636          &45.562          &39.977 & 0.712  \\
\seesc                   &40.088          &35.798          &43.399          &39.962 &  0.795 \\
\covariate        &36.673          &30.403          &39.310          &39.928 & 1.00  \\
\caluni                    &1.775           &1.256           &1.698           &2.358  & 1.00  \\
\produni                   &4.709           &4.328           &3.586           &3.182 & 0.945  \\
\bottomrule
\end{tabular}
% }
\label{tab:average_improvements_sjs_bins_linear}
\end{table}

\subsection{Natural Shift: Linear model}

\paragraph{Key Observations:}  
\begin{itemize}  
\item \textbf{Estimation Inaccuracy (Table~\ref{tab:estimation_errors_nat_linear}):}  
  \begin{itemize}  
    \item \epr{} demonstrates consistently strong performance, achieving low estimation errors across all bins.  
    \item Baselines such as \seesc{}, \seesd{}, and \covariate{} exhibit significantly higher errors.  
  \end{itemize}  
\item \textbf{Model Improvement (Table~\ref{tab:average_improvements_nat_linear}):}  
  \begin{itemize}  
    \item All methods show comparable levels of improvement.  
    \item \produni{} and \caluni{} fail to provide meaningful enhancements in lower bins. \produni{} gives the highest improvement on the most harmful bin.
  \end{itemize}  
\end{itemize}  

\begin{table}[!ht]
\centering
\caption{Estimation metrics for natural shifts using linear models.}
\vspace{0.2cm}
\small
\setlength{\tabcolsep}{2pt} % Reduce horizontal padding for compactness
\renewcommand{\arraystretch}{1.0} % Adjust vertical spacing for compactness
% \resizebox{\columnwidth}{!}{ % Scale table to fit a single column
\begin{tabular}{lccccccc}
\toprule
\textbf{Metrics} & \multicolumn{4}{c}{\emph{Estimation Inaccuracy} ($\downarrow$)} & \emph{Power} ($\uparrow$) & \emph{FP} ($\downarrow$) \\
\midrule
\textbf{Bins} & (-0.082, -0.041) &  (-0.041, -0.001) & (-0.001, 0.16) & (0.16, 0.319) & All-Bins & All-Bins \\
\midrule
\epr               &0.018            & 0.034           &0.005          &0.009 & 0.963 & 0\\
\seesd             &0.157             &0.278           &0.017          &0.044 & 0. & 0\\
\seesc             & 0.146            & 0.274          & 0.019           &0.017 & 0. &  0\\
\covariate         & 0.087             &0.211           &0.015          &0.081 &0. & 0\\
\bottomrule
\end{tabular}
% }
\label{tab:estimation_errors_nat_linear}
\end{table}

\begin{table}[!ht]
\centering
\caption{Model improvement metrics for natural shifts using linear models.}
\vspace{0.2cm}
\small
\setlength{\tabcolsep}{2pt} % Reduce horizontal padding for compactness
\renewcommand{\arraystretch}{1.0} % Adjust vertical spacing for compactness
% \resizebox{\columnwidth}{!}{ % Scale table to fit a single column
\begin{tabular}{lccccc}
\toprule
\textbf{Metrics} & \multicolumn{4}{c}{\emph{Average Improvement} ($\uparrow$)} & \emph{I-prop} ($\uparrow$) \\
\midrule
\textbf{Bins}  & (-0.082, -0.041) &  (-0.041, -0.001) & (-0.001, 0.16) & (0.16, 0.319) & All-Bins \\
\midrule
\produni                   &-7.162           &-14.898          &17.471         &52.100 & 0.242  \\
\epr                      &25.824            &18.047          &27.189         &38.826 & 1.00  \\
\seesd                   &25.821            &18.034          &27.178         &38.588 & 1.00  \\
\seesc                   &25.819            &18.043          &27.161         &38.720 &  1.00 \\
\covariate        &25.814            &18.025          &27.161         &38.575 & 1.00  \\
\caluni                     &1.317             &0.866           &0.568         &-4.069  & 0.934  \\
\bottomrule
\end{tabular}
% }
\label{tab:average_improvements_nat_linear}
\end{table}

\paragraph{Takeaway:} In terms of estimation, \epr{} is by far the most accurate method, significantly outperforming all baselines. Regarding improvement, all methods perform similarly overall, highlighting that reweighting can still drive performance gains even when error estimation is suboptimal—showing the learning algorithm’s ability to benefit from reweighted data itself. Additionally, when the model is not high-performing, classical training with uniform weighting (\caluni) fails to improve and can even degrade performance. Interestingly, revealing labels (\produni) can also underperform in this context, emphasizing the challenges of adaptation when the model class and data distribution are misaligned.

\newpage
\section{Experiments on Regression Model}  \label{app:regression_exp}
In this section, we extend the evaluation of our method to regression problems using the California Housing Pricing dataset \cite{uci_dataset} following the previous section by using a Linear base.  Among the baselines, only \covariate{} is applicable, as \seesc{} and \seesd{} are not designed for regression settings.

\paragraph{Key Observations:}
\begin{itemize}
    \item \textbf{Estimation Inaccuracy (Table~\ref{tab:estimation_errors_nat_linear}):}
    \begin{itemize}
        \item \epr{} achieves the lowest estimation errors across all bins, significantly outperforming the \covariate{} baseline.
        \item The detection power of \epr{} is high (0.889), while \covariate{} fails to detect shifts at all (\emph{Power} ($\uparrow$) = 0).
    \end{itemize}
    \item \textbf{Model Improvement (Table~\ref{tab:average_improvements_bins_linear}):}
    \begin{itemize}
        \item \epr{} demonstrates consistent improvements across all bins, slightly outperforming the \covariate{} baseline.
    \end{itemize}
\end{itemize}

\begin{table}[!ht]
\caption{Estimation metrics for natural shifts using linear models.}
\vspace{0.2cm}
\centering
\small
\setlength{\tabcolsep}{2pt} % Reduce horizontal padding for compactness
\renewcommand{\arraystretch}{1.0} % Adjust vertical spacing for compactness
% \resizebox{\columnwidth}{!}{ % Scale table to fit a single column
\begin{tabular}{lccccccc}
\toprule
\textbf{Metrics} & \multicolumn{4}{c}{\emph{Estimation Inaccuracy} ($\downarrow$)} & \emph{Power} ($\uparrow$) & \emph{FP} ($\downarrow$) \\
\midrule
\textbf{Bins} & (-0.025, -0.018] & (-0.018, 0]  & (0, 0.054] &  (0.054, 0.092] & All-Bins & All-Bins \\
\midrule
\epr                &0.015            &0.019           &0.012           &0.017 & 0.889 & 0\\
\covariate          &0.023            &0.013           &0.074           &0.028 &0. & 0\\
\bottomrule
\end{tabular}
% }
\label{tab:estimation_errors_nat_linear}
\end{table}

\begin{table}[!ht]
\centering
\caption{Model improvement metrics for sparse covariate shift on California House Price using linear models.}
\vspace{0.2cm}
\small
\setlength{\tabcolsep}{2pt} % Reduce horizontal padding for compactness
\renewcommand{\arraystretch}{1.0} % Adjust vertical spacing for compactness
% \resizebox{\columnwidth}{!}{ % Scale table to fit a single column
\begin{tabular}{lccccc}
\toprule
\textbf{Metrics} & \multicolumn{4}{c}{\emph{Average Improvement} ($\uparrow$)} & \emph{I-prop} ($\uparrow$) \\
\midrule
\textbf{Bins} & (-0.025, -0.018] & (-0.018, 0]  & (0, 0.054] &  (0.054, 0.092] & All-Bins \\
\midrule
\produni                  & 5.993            &4.343         & 30.014           &1.980 & 0.625  \\
\epr                      &35.803           &38.933          &51.553          & 40.051 & 1.0 \\
\covariate       &35.701           &38.693          &51.498          &40.004 & 1.0  \\
\caluni                   &0.680            &1.256           &1.730           &0.666  & 0.938  \\
\bottomrule
\end{tabular}
% }
\label{tab:average_improvements_bins_linear}
\end{table}
\paragraph{Takeaway:} Once again, \epr{} significantly outperforms the baseline in terms of error estimation. In terms of model improvement, \epr{} shows a slight advantage. This example also demonstrates that it is possible to have weights that are not necessarily accurate for estimation but can still contribute to improving model performance.

\newpage
\section{Target Data Estimation}    \label{app:target_data_estimation}

In this section, we evaluate how well each method reweights the source data to approximate the target data. Two metrics are used for evaluation:  
\begin{itemize}
    \item \emph{Mean Matching for All Features}: Measures the discrepancy in feature means between the reweighted source data and the target data.  
    \item \emph{Histogram Matching for All Features}: Measures the discrepancy in feature distributions (histogram binning) computed marginally for each features independently and averaged. The number of bins is set at 10.
\end{itemize}

\subsection{Sparse Covariate Shift}
Table~\ref{tab:mean_diff_norms} summarizes the results for sparse covariate shifts. \epr{} achieves the lowest values for both metrics, indicating superior alignment with the target data. \seesc{} performs moderately well, followed by \covariate{} and \seesd{}, which show larger discrepancies.  

\begin{table}[!ht]
\centering
\caption{Comparison of mean and histogram difference norms for sparse covariate shift.}
\vspace{0.2cm}
\small
\setlength{\tabcolsep}{8pt} % Adjust horizontal padding
\renewcommand{\arraystretch}{1.2} % Adjust vertical spacing
\begin{tabular}{lcc}
\toprule
\textbf{Method} & \emph{Mean Matching for All Features} & \emph{Histogram Matching for All Features} \\
\midrule
\epr{}          & 0.03               & 0.07                    \\
\seesd{}        & 0.08               & 0.18                     \\
\seesc{}        & 0.04                & 0.08                      \\
\covariate{}    & 0.06              & 0.11                     \\
\bottomrule
\end{tabular}
\label{tab:mean_diff_norms}
\end{table}

\subsection{Sparse Joint Shift}
Table~\ref{tab:mean_diff_norms_updated} presents the results for sparse joint shifts. Similar to the covariate shift setting, \epr{} demonstrates the best performance, achieving the smallest mean approximation. Regarding histogram differences \seesc{} is slightly better, while \covariate{} and \seesd{} perform worse, with higher values for both metrics.  

\begin{table}[!ht]
\centering
\caption{Comparison of mean and histogram difference norms for sparse joint shifts.}
\vspace{0.2cm}
\small
\setlength{\tabcolsep}{8pt} % Adjust horizontal padding
\renewcommand{\arraystretch}{1.2} % Adjust vertical spacing
\begin{tabular}{lcc}
\toprule
\textbf{Method} & \emph{Mean Matching for All Features} & \emph{Histogram Matching for All Features} \\
\midrule
\epr{}          & 0.04               & 0.12                     \\
\seesd{}        & 0.10              & 0.20                     \\
\seesc{}        & 0.06               & 0.11                    \\
\covariate{}    & 0.12               & 0.17                    \\
\bottomrule
\end{tabular}
\label{tab:mean_diff_norms_updated}
\end{table}

\paragraph{Takeaway:} These findings reinforce the effectiveness of \epr{} in approximating target distributions, making it a reliable choice for reweighting source data under distribution shifts.

\newpage
\section{Hyperparameter Optimization for XGBoost Model} \label{app:hyperparameters}

To fine-tune the XGBoost model, a hyperparameter optimization process was conducted using 1000 iterations. The optimization strategy varied depending on the booster type (\texttt{gbtree} vs \texttt{gblinear}). The following details the process used to suggest and train the optimal hyperparameters. 
The hyperparameters were dynamically suggested during training based on the booster type:

\begin{itemize}
    \item For \texttt{gbtree} boosters, the following hyperparameters were optimized:
    \begin{itemize}
        \item \texttt{eta} (learning rate): Log-uniformly sampled in the range [0.01, 0.3].
        \item \texttt{max\_depth}: Integer value between 3 and 10.
        \item \texttt{subsample}: Float value between 0.6 and 1.0.
        \item \texttt{colsample\_bytree}: Float value between 0.6 and 1.0.
        \item \texttt{lambda} (L2 regularization): Log-uniformly sampled in the range [1e-3, 10.0].
        \item \texttt{alpha} (L1 regularization): Log-uniformly sampled in the range [1e-3, 10.0].
        \item \texttt{min\_child\_weight}: Integer value between 1 and 10.
    \end{itemize}
    \item For the \texttt{gblinear} booster, the following hyperparameters were optimized:
    \begin{itemize}
        \item \texttt{lambda} (L2 regularization): Log-uniformly sampled in the range [1e-3, 10.0].
        \item \texttt{alpha} (L1 regularization): Log-uniformly sampled in the range [1e-3, 10.0].
    \end{itemize}
\end{itemize}
The trained model was optimized for performance on the validation dataset, leveraging early stopping and dynamically selected hyperparameters to achieve the best results using the package Optuna \cite{optuna_2019}.

\section{Additional experiments details} \label{app:details}
\subsection{Correlation metric for the explainability part}
Given two features, \( A \) and \( B \), we determine their types and apply an appropriate correlation metric:

\begin{itemize}
    \item \textbf{Both Continuous:} We use the \emph{Pearson correlation coefficient}, which measures the linear relationship between two continuous variables:
    \[
    \rho(A, B) = \frac{\mathrm{cov}(A, B)}{\sigma_A \sigma_B},
    \]
    where \( \mathrm{cov}(A, B) \) is the covariance of \( A \) and \( B \), and \( \sigma_A \) and \( \sigma_B \) are their respective standard deviations.

    \item \textbf{Both Categorical:} We use \emph{Cramér’s V}, which is derived from the chi-squared (\( \chi^2 \)) statistic computed from the contingency table:
    \[
    V = \sqrt{\frac{\chi^2/n}{\min(k_A - 1, k_B - 1)}},
    \]
    where \( \chi^2 \) is the chi-squared statistic, \( n \) is the total number of observations, and \( k_A \) and \( k_B \) are the number of categories in \( A \) and \( B \), respectively.

    \item \textbf{One Continuous, One Binary:} We use the \emph{Point-Biserial correlation coefficient}, which measures the correlation between a continuous variable and a binary categorical variable:
    \[
    r_{pb} = \frac{\overline{X}_1 - \overline{X}_0}{s} \sqrt{\frac{n_1 n_0}{n(n - 1)}},
    \]
    where \( \overline{X}_1 \) and \( \overline{X}_0 \) are the means of the continuous variable for the two binary groups, \( s \) is the standard deviation of the continuous variable, and \( n_1 \), \( n_0 \), and \( n \) are the sizes of the respective groups and the total sample.

\end{itemize}

Finally, we take the absolute value of each metric to ensure that all correlation values fall within the range \([0,1]\). We compute the values using all the dataset.

\subsection{Dataset Details}
\label{sec:dataset_details}

Our experiments leverage a diverse array of publicly available, real-world datasets. These datasets vary in terms of sample size, feature dimensionality, and application domain, facilitating a comprehensive evaluation of our proposed framework under different conditions. Specific details for each dataset are provided below. For all datasets, categorical features were one-hot encoded, and the resulting feature dimensions are reported.

\begin{itemize}
    \item \textbf{Adult:}
    This dataset, commonly used for income prediction (binary classification of whether income exceeds \$50K/yr), is sourced from the UCI Machine Learning Repository \cite{uci_dataset}. It originally contains \textbf{14 features} (a mix of continuous and categorical) for \textbf{48,842 samples}. After one-hot encoding of categorical variables, the feature space expands to \textbf{109 dimensions}.

    \item \textbf{HELOC (Home Equity Line of Credit):}
    Provided by \cite{helocdata}, this dataset is used for credit risk assessment. It comprises \textbf{9,871 samples} and \textbf{24 features}. The task is to predict whether a homeowner will default on their HELOC.

    \item \textbf{NHANES I (National Health and Nutrition Examination Survey):}
    This dataset consists of health-related data from the first National Health and Nutrition Examination Survey, made available by the \cite{nhanes}. In our experiments, we use a version with \textbf{8,593 samples} and \textbf{18 features}, focusing on a mortality prediction task.

    \item \textbf{California Housing Prices:}
    Also from the UCI Machine Learning Repository \cite{uci_dataset}, this dataset contains \textbf{20,640 samples} and \textbf{9 features}. It is a standard benchmark for regression tasks, where the objective is to predict the median house value in California districts.
\end{itemize}

\subsection*{Folktables (ACS-based Tasks)}
The Folktables suite \cite{folktables} provides benchmark tasks derived from the U.S. Census Bureau's American Community Survey (ACS). These datasets allow for the study of model performance under various demographic shifts. Table~\ref{tab:folktables_stats} summarizes the characteristics of several tasks available within this suite.

\begin{table}[H] % Using [H] from float package for "here if possible"
\centering
\begin{tabular}{|l|c|r|} % Changed 'c' to 'r' for datapoints for better alignment
\hline
\textbf{Task Name} & \textbf{No. of Features (Original)} & \textbf{No. of Datapoints} \\
\hline
ACSIncome           & 10 & 1,664,500 \\
ACSPublicCoverage   & 19 & 1,138,289 \\
ACSMobility         & 21 &   620,937 \\
ACSEmployment       & 17 & 3,236,107 \\
ACSTravelTime       & 16 & 1,466,648 \\
\hline
\end{tabular}
\caption{Overview of selected Folktables tasks based on ACS data. Feature counts are prior to one-hot encoding.}
\label{tab:folktables_stats}
\end{table}

For our experiments involving natural distribution shifts, we primarily utilize the \textbf{ACSIncome} task from Folktables. While the original ACSIncome task includes 10 features, after one-hot encoding of its categorical variables, the dimensionality increases to approximately \textbf{284 features}.
The full ACSIncome dataset across all available years and states contains over 1.6 million records. In our experimental setup, we generate source and target domains by selecting data for specific U.S. states and years. This typically results in experimental datasets ranging from \textbf{30,000 to 100,000 samples}, though some state-year combinations may yield larger datasets.

\newpage
%%%%%%%%%%%%%%%%%%%%%%%%%%%%%%%%%%%%%%%%%%%%%%%%%%%%%%%%%%%%
\section{Additional Related Works Discussion} \label{sec:related_works}

One important related body of literature is representation learning. Methods within this field have progressed in managing unstructured data such as images by using mean-matching techniques to align learned embedding features for adapting to distributional changes, as highlighted in works like \citep{long2015learning, tzeng2014deep, tzeng2017adversarial, long2015learning, ganin2016domain, sun2016deep} or sub-domain alignment methods \citep{liu2023domain}. However, the EPA approach differs from these works. These approaches primarily address domain adaptation alone. In contrast, EPA is specifically designed to tackle three interrelated problems simultaneously: (1) performance evaluation, (2) explanation of shifts, and (3) model adaptation. Thus, these earlier methods do not fully compare across all three aspects. In addition, direct empirical comparisons with these methods are challenging, as many are tailored to specific model architectures, such as deep neural networks with unique loss functions and feature alignment strategies—whereas EPA is not restricted to a particular model type. Rather than replacing existing models, EPA integrates a corrective component using estimated importance weights, setting it apart in its goals and operation. Putting differently, EPA is a post-hoc method. However, extending EPA to fields like computer vision or textual data would require an additional phase for identifying meaningful feature representations where mean-matching could be effectively applied. Developing this extension is complex and separate from EPA's core function and thus is not covered in this paper. As stated in our conclusion, we have left this extension for future exploration.

Given EPA's core mechanism of deriving importance weights, its contributions are particularly relevant to the literature on reweighting techniques for distribution shift. Within this domain, SEES \citep{chen2022estimating} is our main competitor and best among the reweighting techniques as demonstrated in their paper. Our work therefore emphasizes comparison with such state-of-the-art reweighting approaches to clearly benchmark EPA's advancements in robust weight estimation, along several additional dimensions such as performance prediction, explanation of shifts, and model adaptation.

Lastly, EPA can be contrasted with distributionally robust optimization (DRO) strategies \citep{sagawa2019distributionally}. DRO methods typically aim to optimize model performance under worst-case distributional shifts, often by reweighting data from predefined groups. While effective for enhancing worst-case generalization, DRO approaches require these group definitions. Specifying such groups comprehensively can be challenging, may not capture all relevant types of shifts, or the necessary groupings might not be apparent in all domains (for instance, while image datasets sometimes allow for grouping by attributes like gender or hair color, this approach has its limitations and is not universally applicable). Furthermore, DRO techniques typically do not offer the direct target performance estimation or the explanatory insights inherent to the EPA framework.

\newpage
\section{Algorithmic Description of \epr{}} \label{sec:epa_algorithms} % 

This section provides a detailed algorithmic description of the proposed Entropic Projection Alignment (\epr{}) framework. The following algorithms outline the core components and applications of \epr{}:

\begin{itemize}
    \item Algorithm~\ref{algo:core_epa} describes the core \epr{} procedure for computing importance weights and the reweighted source distribution.
    \item Algorithm~\ref{algo:target_estimation} details how \epr{} is used to estimate model performance on an unlabeled target domain.
    \item Algorithm~\ref{algo:shift_explanation} outlines the process for identifying feature subsets responsible for distribution shifts, utilizing Algorithm~\ref{algo:kl_divergence} for discrepancy measurement.
    \item Algorithm~\ref{algo:model_adaptation} shows the methodology for adapting an existing model to improve its performance in the target domain using \epr{}-derived weights.
    \item Algorithm~\ref{algo:kl_divergence} provides the method for computing the histogram-based symmetric KL divergence used in the shift explanation process.
    \item Algorithm~\ref{algo:target_estimation_with_prevalence} provide the method for computing the target estimation with label prevalence.
\end{itemize}

\begin{algorithm}[H] % Consider [H] for "here" placement if needed, or [!ht]
\caption{Core \epr{} (Entropic Projection Alignment)}
\begin{algorithmic}[1]
\Statex \textbf{Input:}
\Statex \quad Source dataset: $D_{\text{src}} = \{(X_i, Y_i)\}_{i=1}^{n}$ % Corrected indexing
\Statex \quad Target dataset: $D_{\text{tar}} = \{X_i^t\}_{i=1}^{m}$ % Corrected indexing
\Statex \quad Feature mapping function: $\Phi(X)$
\Statex \textbf{Output:}
\Statex \quad Importance weights: $\lambda = (\lambda_1, \ldots, \lambda_n)$
\Statex \quad Reweighted source distribution: $P^*$ % Clarified it's a reweighted source distribution
\Statex \textbf{Procedure:}
\State \textbf{Step 1:} Define the convex loss function:
$$\text{Loss}(\xi) = \log\left(\frac{1}{n} \sum_{i=1}^n \exp(\langle \Phi(X_i), \xi \rangle)\right) - \langle \xi, \mu_t \rangle$$
where $\mu_t = \frac{1}{m} \sum_{j=1}^m \Phi(X_j^t)$ is the mean of the feature mapping over the target data. % Corrected index to j for target sum
\State \textbf{Step 2:} Optimize to find the optimal parameter vector $\xi^*$:
$$\xi^* = \arg\min_{\xi} \text{Loss}(\xi)$$
using convex optimization (e.g., L-BFGS optimizer).
\State \textbf{Step 3:} Compute importance weights and define the reweighted source distribution:
$$\lambda_i = \frac{\exp(\langle \xi^*, \Phi(X_i) \rangle)}{\sum_{j=1}^n \exp(\langle \xi^*, \Phi(X_j) \rangle)}, \quad \text{for } i = 1,\ldots,n$$
$$P^* = \sum_{i=1}^n \lambda_i \delta_{(X_i, Y_i)}$$
\State \textbf{Return:} Importance weights $\lambda$ and reweighted source distribution $P^*$.
\end{algorithmic}
\label{algo:core_epa} % Changed label for clarity
\end{algorithm}

\begin{remark}
    Assuming $d$ features and the scalar score $f(X)$ are discretized into $B$ bins, then EPA optimizes $k=(d+1)(B-1)$ parameters. Evaluating the objective in~\eqref{eq:loss} and its gradient costs $O(nk)=O(n\,d\,B)$, so $I$ L-BFGS iterations yield total complexity $O(I\,n\,d\,B)$, linear in the number of source samples.
\end{remark}
\begin{algorithm}[H]
\caption{Target Performance Estimation using \epr{}}
\begin{algorithmic}[1]
\Statex \textbf{Input:}
\Statex \quad Source dataset: $D_{\text{src}} = \{(X_i, Y_i)\}_{i=1}^{n}$
\Statex \quad Target dataset: $D_{\text{tar}} = \{X_i^t\}_{i=1}^{m}$
\Statex \quad Predictive model: $f: \mathcal{X} \rightarrow \mathcal{Y}$ % Used mathcal for spaces
\Statex \quad Loss function: $\mathcal{L}: \mathcal{Y} \times \mathcal{Y} \rightarrow [0, 1]$
\Statex \textbf{Output:}
\Statex \quad Estimated target error: $\hat{E}_{\text{tar}}(f)$
\Statex \textbf{Procedure:}
\State \textbf{Step 1:} Choose quantile bins for each feature and for the scalar score $f(X)$. Let $\psi_j(X^{(j)})$ and $\psi_f(f(X))$ denote the corresponding bin-indicator blocks with one redundant bin removed, and define:
$$\Phi_{\text{est}}(X) = [\psi_1(X^{(1)}), \ldots, \psi_d(X^{(d)}), \psi_f(f(X))]$$

\State \textbf{Step 2:} Apply Algorithm~\ref{algo:core_epa} with $\Phi_{\text{est}}$ to obtain importance weights:
$$\lambda = (\lambda_1, \ldots, \lambda_n) \leftarrow \text{Core-\epr{}}(D_{\text{src}}, D_{\text{tar}}, \Phi_{\text{est}})$$ % Used arrow and named call
\State \textbf{Step 3:} Calculate the estimated target error using importance weights:
$$\hat{E}_{\text{tar}}(f) = \sum_{i=1}^n \lambda_i \cdot \mathcal{L}(f(X_i), Y_i)$$
\State \textbf{Return:} Estimated target error $\hat{E}_{\text{tar}}(f)$.
\end{algorithmic}
\label{algo:target_estimation} % Changed label
\end{algorithm}

\begin{algorithm}[H]
\caption{Shift Explanation using \epr{}}
\begin{algorithmic}[1]
\Statex \textbf{Input:}
\Statex \quad Source dataset: $D_{\text{src}} = \{(X_i, Y_i)\}_{i=1}^{n}$
\Statex \quad Target dataset (unlabeled features): $D_{\text{tar}} = \{X_i^t\}_{i=1}^{m}$
\Statex \quad Candidate feature subsets: $\mathcal{S} = \{S_1, S_2, \ldots, S_k\}$, where $S_j \subseteq \{1, \ldots, d\}$
\Statex \textbf{Output:}
\Statex \quad Feature subset $S^*$ that best explains the distribution shift.
\Statex \textbf{Procedure:}
\State \textbf{Step 1:} For each candidate subset $S_j \in \mathcal{S}$:
\Statex \quad \textbf{1.1} Define feature mapping using only features in subset $S_j$:
$$\Phi_{S_j}(X) = [\psi_s(X^{(s)})]_{s \in S_j}$$
\Statex \quad \textbf{1.2} Apply Algorithm~\ref{algo:core_epa} with $\Phi_{S_j}$ to obtain the reweighted source distribution $P_{S_j}^*(\text{features only})$ and corresponding weights $\lambda_{S_j}$:
$$(P_{S_j}^*(\text{features only}), \lambda_{S_j}) \leftarrow \text{Core-\epr{}}(D_{\text{src}}, D_{\text{tar}}, \Phi_{S_j})$$
\Statex \quad \textbf{1.3} Compute discrepancy $D(S_j)$ between the reweighted source feature distribution (using $P_{S_j}^*$ or weights $\lambda_{S_j}$ on $D_{\text{src}}$ features) and the target feature distribution $Q_X^{\text{tar}}$ (empirical distribution of $D_{\text{tar}}$). Typically using Algorithm~\ref{algo:kl_divergence}:
$$D(S_j) = \text{SymmetricKL}( \{(X_i, (\lambda_{S_j})_i)\}_{i=1}^n, D_{\text{tar}} \text{ for features in } S_j )$$ % Be more explicit
\State \textbf{Step 2:} Select the feature subset that minimizes the discrepancy:
$$S^* = \arg\min_{S_j \in \mathcal{S}} D(S_j)$$
\State \textbf{Return:} Optimal feature subset $S^*$.
\end{algorithmic}
\label{algo:shift_explanation} % Changed label
\end{algorithm}

\begin{algorithm}[H]
\caption{Model Adaptation using \epr{} (Boosting Approach)}
\begin{algorithmic}[1]
\Statex \textbf{Input:}
\Statex \quad Source dataset: $D_{\text{src}} = \{(X_i, Y_i)\}_{i=1}^{n}$
\Statex \quad Target dataset: $D_{\text{tar}} = \{X_i^t\}_{i=1}^{m}$
\Statex \quad Initial model: $f: \mathcal{X} \rightarrow \mathcal{Y}$
\Statex \quad Loss function: $\mathcal{L}: \mathcal{Y} \times \mathcal{Y} \rightarrow [0, 1]$
\Statex \quad Base learner class: $\mathcal{H}$
\Statex \quad Number of boosting iterations: $M$
\Statex \textbf{Output:}
\Statex \quad Adapted model: $\tilde{f}: \mathcal{X} \rightarrow \mathcal{Y}$
\Statex \textbf{Procedure:}
\State \textbf{Step 1:} Use the same histogram-based matching map as in estimation:
$$\Phi_{\text{adapt}}(X) = [\psi_1(X^{(1)}), \ldots, \psi_d(X^{(d)}), \psi_f(f(X))]$$
\State \textbf{Step 2:} Apply Algorithm~\ref{algo:core_epa} with $\Phi_{\text{adapt}}$ to obtain importance weights:
$$\lambda = (\lambda_1, \ldots, \lambda_n) \leftarrow \text{Core-\epr{}}(D_{\text{src}}, D_{\text{tar}}, \Phi_{\text{adapt}})$$
\State \textbf{Step 3:} Initialize the adapted model:
$$f_0 = f$$
\State \textbf{Step 4:} For each iteration $t = 1, 2, \ldots, M$:
\Statex \quad \textbf{4.1} Fit base learner $h_t \in \mathcal{H}$ by minimizing the weighted loss:
$$h_t = \arg\min_{h \in \mathcal{H}} \sum_{i=1}^n \lambda_i \cdot \mathcal{L}(f_{t-1}(X_i) + h(X_i), Y_i)$$
\Statex \quad \textbf{4.2} Determine the learning rate (or step size) $\alpha_t$:
$$\alpha_t = \arg\min_{\alpha \in \mathbb{R}} \sum_{i=1}^n \lambda_i \cdot \mathcal{L}(f_{t-1}(X_i) + \alpha \cdot h_t(X_i), Y_i)$$
\Statex \quad \textbf{4.3} Update the model:
$$f_t = f_{t-1} + \alpha_t \cdot h_t$$
\State \textbf{Step 5:} Set the final adapted model:
$$\tilde{f} = f_M$$
\State \textbf{Return:} Adapted model $\tilde{f}$.
\Statex \textit{Note: In practice, gradient boosting libraries like XGBoost can implement this procedure by accepting instance weights $\lambda_i$.}
\end{algorithmic}
\label{algo:model_adaptation} % Changed label
\end{algorithm}

\begin{algorithm}[H]
\caption{Histogram-based Symmetric KL Divergence}
\begin{algorithmic}[1]
\Statex \textbf{Input:}
\Statex \quad Reweighted source data (features and weights): $D_s = \{(X_i, \lambda_i)\}_{i=1}^n$, where $X_i \in \mathbb{R}^d$
\Statex \quad Target data (features): $D_t = \{X^t_j\}_{j=1}^m$, where $X^t_j \in \mathbb{R}^d$
\Statex \quad Number of bins for histogram: $B$ (default: 10)
\Statex \quad Smoothing parameter: $\varepsilon$ (default: $10^{-6}$)
\Statex \textbf{Output:}
\Statex \quad Average Symmetric KL divergence over all features: $D_{KL}^{\text{avg}}$
\Statex \textbf{Procedure:}
\State \textbf{Initialize:} Total divergence $D_{KL}^{\text{total}} \leftarrow 0$
\State \textbf{For} each feature dimension $k \in \{1, \ldots, d\}$:
\Statex \quad \textbf{Step 1:} Extract $k$-th feature values from source and target:
\Statex \quad \quad Source values: $x_s^{(k)} = [X_1^{(k)}, X_2^{(k)}, \ldots, X_n^{(k)}]$ (with corresponding weights $\lambda_i$)
\Statex \quad \quad Target values: $x_t^{(k)} = [X_1^{t,(k)}, X_2^{t,(k)}, \ldots, X_m^{t,(k)}]$
\Statex
\Statex \quad \textbf{Step 2:} Compute empirical probability distributions $\mathbf{p}^{(k)}$ and $\mathbf{q}^{(k)}$ using histograms:
\Statex \quad \quad Determine common bin edges for the $k$-th feature based on the combined range of $x_s^{(k)}$ and $x_t^{(k)}$.
\Statex \quad \quad $\mathbf{p}^{(k)} = [p_1^{(k)}, \ldots, p_B^{(k)}]$: weighted histogram of $x_s^{(k)}$ using weights $\lambda_i$.
\Statex \quad \quad $\mathbf{q}^{(k)} = [q_1^{(k)}, \ldots, q_B^{(k)}]$: unweighted histogram of $x_t^{(k)}$.
\Statex
\Statex \quad \textbf{Step 3:} Apply smoothing and normalize distributions:
\Statex \quad \quad $p_b^{(k)} \leftarrow p_b^{(k)} + \varepsilon$ and $q_b^{(k)} \leftarrow q_b^{(k)} + \varepsilon$ for all bins $b=1, \ldots, B$.
\Statex \quad \quad Normalize $\mathbf{p}^{(k)}$ such that $\sum_b p_b^{(k)} = 1$.
\Statex \quad \quad Normalize $\mathbf{q}^{(k)}$ such that $\sum_b q_b^{(k)} = 1$.
\Statex
\Statex \quad \textbf{Step 4:} Calculate symmetric KL divergence for the $k$-th feature:
\Statex \quad \quad $KL(\mathbf{p}^{(k)} \parallel \mathbf{q}^{(k)}) = \sum_{b=1}^B p_b^{(k)} \log\left(\frac{p_b^{(k)}}{q_b^{(k)}}\right)$
\Statex \quad \quad $KL(\mathbf{q}^{(k)} \parallel \mathbf{p}^{(k)}) = \sum_{b=1}^B q_b^{(k)} \log\left(\frac{q_b^{(k)}}{p_b^{(k)}}\right)$
\Statex \quad \quad $D_{KL}^{(k)} = KL(\mathbf{p}^{(k)} \parallel \mathbf{q}^{(k)}) + KL(\mathbf{q}^{(k)} \parallel \mathbf{p}^{(k)})$
\Statex
\Statex \quad \textbf{Step 5:} Accumulate total divergence:
\Statex \quad \quad $D_{KL}^{\text{total}} \leftarrow D_{KL}^{\text{total}} + D_{KL}^{(k)}$
\State $D_{KL}^{\text{avg}} = D_{KL}^{\text{total}} / d$
\State \textbf{Return:} Average Symmetric KL divergence $D_{KL}^{\text{avg}}$.
\end{algorithmic}
\label{algo:kl_divergence} % Changed label
\end{algorithm}

\begin{algorithm}[H]
\caption{Target Performance Estimation with Label Prevalence Optimization (EPA-LPO)}
\begin{algorithmic}[1]
\Statex \textbf{Input:}
\Statex \quad Source dataset: $D_{\text{src}} = \{(X_i, Y_i)\}_{i=1}^{n}$ (where $Y_i \in \{0,1\}$)
\Statex \quad Target features: $D_{\text{tar\_X}} = \{X_j^t\}_{j=1}^{m}$ \Comment{Unlabeled target data}
\Statex \quad Predictive model: $f: \mathcal{X} \rightarrow \mathbb{R}$ (e.g., a scorer)
\Statex \quad Loss function: $\mathcal{L}: \mathcal{Y} \times \mathcal{Y} \rightarrow [0, 1]$
\Statex \quad Candidate target prevalences for $Y=1$: $\mathcal{Q}_1 = \{q_{1}^{(1)}, q_{1}^{(2)}, \ldots, q_{1}^{(K)}\}$ \Comment{Grid like $[0.05, 0.1, \ldots, 0.95]$}
\Statex \quad Discrepancy metric: $\text{Discrepancy}(\cdot, \cdot)$ \Comment{Typically Symmetric KL Divergence}
\Statex \textbf{Output:}
\Statex \quad Estimated target error: $\hat{E}_{\text{tar}}(f)$
\Statex \quad Estimated target label prevalence for $Y=1$: $q_1^*$
\Statex \textbf{Procedure:}
\State Define base feature mapping (features and model predictions):
       $\Phi_{\text{base}}(X) = [X^{(1)}, \ldots, X^{(d)}, f(X)]$
\State Compute target moments for $\Phi_{\text{base}}(X)$ from $D_{\text{tar\_X}}$:
       $\mu_{\text{base\_target}} = \frac{1}{m} \sum_{j=1}^m \Phi_{\text{base}}(X_j^t)$
\State Initialize $D_{\min} \leftarrow \infty$, $q_1^* \leftarrow \text{None}$, $\lambda^* \leftarrow \text{None}$
\For{each candidate prevalence $q_1 \in \mathcal{Q}_1$} \Comment{Try different label prevalence values}
    \State Define augmented feature mapping including the label $Y$:
           $\Phi_{\text{aug}}(X,Y) = [\Phi_{\text{base}}(X), Y]$
    \State Construct the target moment vector for $\Phi_{\text{aug}}$:
           $\mu_{\text{aug\_target}}(q_1) = [\mu_{\text{base\_target}}, q_1]$ \Comment{Hypothesize target moments with current $q_1$}
    \State Solve for $\xi^*(q_1)$ to get weights $\lambda_i(q_1)$: \Comment{Entropic projection step}
    \State \quad $\text{Loss}(\xi; q_1) = \log\left(\frac{1}{n} \sum_{i=1}^n \exp(\langle \Phi_{\text{aug}}(X_i, Y_i), \xi \rangle)\right) - \langle \xi, \mu_{\text{aug\_target}}(q_1) \rangle$
    \State \quad $\xi^*(q_1) = \arg\min_{\xi} \text{Loss}(\xi; q_1)$ \Comment{Convex optimization}
    \State \quad $\lambda_i(q_1) = \frac{\exp(\langle \xi^*(q_1), \Phi_{\text{aug}}(X_i, Y_i) \rangle)}{\sum_{k=1}^n \exp(\langle \xi^*(q_1), \Phi_{\text{aug}}(X_k, Y_k) \rangle)}$, for $i=1,\ldots,n$
    \State Calculate discrepancy $D(q_1)$ between reweighted source and target distributions:
    \State \quad $D(q_1) = \text{Discrepancy}(\{(X_i, \lambda_i(q_1))\}_{i=1}^n, D_{\text{tar\_X}})$ \Comment{Measure alignment quality}
    \If{$D(q_1) < D_{\min}$}
        \State $D_{\min} \leftarrow D(q_1)$
        \State $q_1^* \leftarrow q_1$ \Comment{Store best prevalence estimate}
        \State $\lambda^* \leftarrow \lambda(q_1)$
    \EndIf
\EndFor
\State Calculate the final estimated target error using optimal weights:
       $\hat{E}_{\text{tar}}(f) = \sum_{i=1}^n \lambda_i^* \cdot \mathcal{L}(f(X_i), Y_i)$ \Comment{Weighted average of errors}
\State \textbf{Return:} Estimated target error $\hat{E}_{\text{tar}}(f)$ and estimated prevalence $q_1^*$.
\end{algorithmic}
\label{algo:target_estimation_with_prevalence}
\end{algorithm}

\newpage
\section{Theoretical analysis of \epr{}} \label{theo:epa_a}

\subsection{Proof of Proposition \ref{prop:exp-tilt}} \label{proof:prop-epa}

To prove Proposition \ref{prop:exp-tilt}, we begin with the fundamental theorem governing the minimization of KL divergence relative to a prior measure under linear moment constraints. This result corresponds to Theorem A.1 in \cite{bachoc2023explaining}, stated there without proof as a consequence of the theorems of \cite{csiszar1984sanov}.

\begin{theorem}[Minimum KL Divergence Solution]\label{thm:main}
Let $(E, \mathcal{B}(E))$ be a measurable space and $Q$ a probability measure on $E$. Consider $t \in \mathbb{R}^k$ and a measurable function $\Phi : E \to \mathbb{R}^k$. We assume that, for $v \in \mathbb{R}^k, b \in \mathbb{R}$, $Q(\{x \in E; \langle v, \Phi(x) \rangle = b\}) = 1$ if and only if $v=0$ and $b=0$. Let $\mathbb{P}_{\Phi, t}$ be the set of all probability measures $P$ on $(E, \mathcal{B}(E))$ such that $\int_E \Phi(x) dP(x) = t$. Assume that $\mathbb{P}_{\Phi, t}$ contains a probability measure that is mutually absolutely continuous with respect to $Q$.

For a vector $\xi \in \mathbb{R}^k$, let $Z(\xi) := \int_E e^{\langle \xi, \Phi(x) \rangle} dQ(x)$. We assume that the set on which $Z$ is finite is open. Define now $\xi(t)$ as the unique minimizer of the strictly convex function $H(\xi) - \langle \xi, t \rangle$, where $H(\xi) := \log Z(\xi)$.

Then, the solution to the optimization problem
$$
Q_t := \arg\inf_{P \in \mathbb{P}_{\Phi,t}} \text{KL}(P,Q)
$$
exists and is unique. Furthermore, it can be computed as
$$
Q_t = \frac{\exp(\langle \xi(t), \Phi \rangle)}{Z(\xi(t))} Q.
$$
This means its Radon-Nikodym derivative is $\frac{dQ_t}{dQ}(x) = \frac{\exp(\langle \xi(t), \Phi(x) \rangle)}{Z(\xi(t))}$.
\end{theorem}

\begin{proof}
    Assume $P$ is absolutely continuous with respect to $Q$, we can frame the problem as minimising the functional $I[p] = \int_E p \log p \, dQ$ subject to the constraints $\int_E p \, dQ = 1$ and $\int_E \Phi p \, dQ = t$. We form the Lagrangian functional with multipliers $\lambda$ and $\xi \in \mathbb{R}^k$:
$$ L[p] = \int_E p \log p \, dQ - \lambda\left(\int_E p \, dQ - 1\right) - \left\langle \xi, \int_E \Phi p \, dQ - t \right\rangle. $$
Combining terms under a single integral:
$$ L[p] = \int_E \left[ p(x)\log p(x) - \lambda p(x) - \langle \xi, \Phi(x) \rangle p(x) \right] dQ(x) + \lambda + \langle \xi, t \rangle. $$
To find the stationary point, we take the functional derivative with respect to $p(x)$ and set it to zero. This is equivalent to finding the point-wise minimum of the integrand:
$$ \frac{\partial}{\partial p} \left[ p \log p - \lambda p - \langle \xi, \Phi \rangle p \right] = \log p + 1 - \lambda - \langle \xi, \Phi \rangle = 0. $$
Solving for $p(x)$ reveals its exponential form:
$$ p(x) = \exp(\lambda - 1 + \langle \xi, \Phi(x) \rangle). $$
The multipliers are determined by the constraints. The normalization constraint $\int p(x) dQ(x) = 1$ implies $e^{\lambda - 1} = 1/Z(\xi)$, where $Z(\xi)$ is the partition function. This gives the density form:
$$ p_\xi(x) = \frac{e^{\langle \xi, \Phi(x) \rangle}}{Z(\xi)}. $$

Next, consider  $\int_E \Phi(x) p_\xi(x) \, dQ(x) = t$.
\[ \int_E \Phi(x) \frac{\exp(\langle \xi, \Phi(x) \rangle)}{Z(\xi)} \, dQ(x) = t. \]
We can relate the integral to the gradient of $Z(\xi)$. Since the set where $Z(\xi)$ is finite is open, we can differentiate under the integral sign:
\[ \nabla_{\xi} Z(\xi) = \nabla_{\xi} \int_E e^{\langle \xi, \Phi(x) \rangle} dQ(x) = \int_E \Phi(x) e^{\langle \xi, \Phi(x) \rangle} dQ(x). \]
The moment constraint becomes:
\[ \frac{1}{Z(\xi)} \nabla_{\xi} Z(\xi) = t. \]
Now, consider $H(\xi) = \log Z(\xi)$. Its gradient is $\nabla_{\xi} H(\xi) = \frac{\nabla_{\xi} Z(\xi)}{Z(\xi)}$. Thus, the Lagrange multiplier vector $\xi$, which we now denote $\xi(t)$, must satisfy:
\[ \nabla_{\xi} H(\xi) = t. \]
This is precisely the first-order condition for minimizing $H(\xi) - \langle \xi, t \rangle$. By the theorem's assumption, this function is strictly convex and thus has a unique minimizer $\xi(t)$.
\end{proof}

\begin{proof}[Proof of Proposition \ref{prop:exp-tilt}]
The \epr{} is a direct, discrete instance of Theorem \ref{thm:main}. The key is to replace the general measures and integrals with their empirical counterparts.

The solution from Theorem \ref{thm:main} translates directly. The optimal distribution $P^*$ must be absolutely continuous with respect to $\hat{P}$, meaning it is also a discrete distribution on the same $n$ samples, defined by a vector of weights $\lambda^* = (\lambda_1^*, \dots, \lambda_n^*)$. The Radon-Nikodym derivative $\frac{dP^*}{d\hat{P}}$ at a sample point $X_i$ is $\frac{\lambda_i^*}{1/n} = n \lambda_i^*$.
Applying the theorem's solution form:
$$ n \lambda_i^* = \frac{\exp(\langle \xi^*, \Phi(X_i) \rangle)}{Z(\xi^*)}, $$
where the partition function is now a sum: $Z(\xi^*) = \int e^{\langle \xi^*, \Phi(x) \rangle} d\hat{P}(x) = \frac{1}{n} \sum_{j=1}^n e^{\langle \xi^*, \Phi(X_j) \rangle}$.
Substituting $Z(\xi^*)$ gives the explicit formula for the optimal weights:
$$ \lambda_i^* = \frac{1}{n Z(\xi^*)} \exp(\langle \xi^*, \Phi(X_i) \rangle) = \frac{\exp(\langle \xi^*, \Phi(X_i) \rangle)}{\sum_{j=1}^n \exp(\langle \xi^*, \Phi(X_j) \rangle)}. $$
The vector $\xi^*$ is the unique parameter that satisfies the moment constraint: $\sum_{i=1}^n \lambda_i^* \Phi(X_i) = \hat{\mu}_Q$.
\end{proof}

\subsection{\epr{}'s weights minimises variance upper-bound by construction} \label{proof:variance}

Interestingly, \epr{}'s definition in \ref{def:EPA-iproj} can be rewritten as the following optimisation problem:

\begin{definition}[EPA - empirical weight version]\label{def:EPA}
EPA solves the I-projection of \(\hat P\) onto the affine moment set:
\begin{equation}\label{eq:EPA}
\lambda^*\;\in\; \arg\min_{\lambda\in\Delta_n} \ \KL(\lambda\Vert u)\quad\text{s.t.}\quad \sum_{i=1}^n \lambda_i\,\Phi(X_i)\;=\;\widehat\mu_Q.
\end{equation}
Equivalently, \eqref{eq:EPA} maximizes Shannon entropy on the feasible set, since \(\KL(\lambda\Vert u)=\log n-H(\lambda)\).
\end{definition}

Let $S=(X_{1:n},\tilde X_{1:m})$ denote the (unlabeled) features from source and target. EPA produces weights
\[
\lambda(S)\in\Delta_n,\qquad 
\sum_{i=1}^n \lambda_i\,\Phi(X_i)=\widehat\mu_Q,\qquad
\lambda_i\propto \exp\{\langle\xi,\Phi(X_i)\rangle\},
\]
which depend only on $S$ (not on labels). For any $f\in\Hcal$, define the centered label-noise term
\[
\Delta_f\ :=\ \sum_{i=1}^n \lambda_i\Big(\Loss(f(X_i),Y_i)-\E[\Loss(f(X_i),Y_i)\mid X_i]\Big).
\]
Write $u$ for the uniform weights $u_i=1/n$, and set $N_{\mathrm{eff}}(\lambda):=1/\sum_{i=1}^n\lambda_i^2$.

\begin{theorem}[Variance upper bound via KL; EPA minimises the bound]
\label{thm:epa-variance-min}
Assume $\Loss\in[0,1]$. Conditionally on $S$, for any mean-matching $w\in\Delta_n$ satisfying $\sum_i \lambda_i\,\Phi(X_i)=\widehat\mu_Q$,
\begin{equation}
\label{eq:var-chain}
\Var(\Delta_f\mid S;\lambda)\ \le\ \tfrac14\sum_{i=1}^n \lambda_i^2
\ \le\ \tfrac14\!\left(\frac1n + 2\,\KL(\lambda\|u)\right)
\ \le\ \tfrac14\!\left(\frac1n + 2\,\KL(w\|u)\right).
\end{equation}
Consequently, EPA’s $\lambda$ minimises the KL-based \emph{upper bound} in \eqref{eq:var-chain} over the feasible set. In particular, EPA favors high effective sample size $N_{\mathrm{eff}}(\lambda)$, although it does not necessarily maximise $N_{\mathrm{eff}}(\lambda)$ exactly.
\end{theorem}

\begin{proof}[Proof sketch]
Given $S$, the $\lambda_i$ are fixed and label-independent. The summands in $\Delta_f$ are independent, mean-zero, and bounded by $\lambda_i$, so
$\Var(\Delta_f\mid S)=\sum_i \lambda_i^2\,\Var(\cdot\mid X_i)\le \tfrac14\sum_i \lambda_i^2$.
Next, relate dispersion to KL via
$\sum_i \lambda_i^2=\tfrac1n+\|\lambda-u\|_2^2 \le \tfrac1n+\|\lambda-u\|_1^2 \le \tfrac1n+2\,\KL(\lambda\|u)$ (Pinsker). EPA minimises $\KL(\cdot\|u)$ subject to the same moment constraint, giving the final inequality for any feasible $w$.
\end{proof}
\begin{proof}
\textbf{Step 1: centring, independence and a basic variance bound.}
For each \(i\), write
\[
\ell_i \ :=\ \Loss\big(f(X_i),Y_i\big)\in[0,1],\qquad
m_i \ :=\ \E[\ell_i\mid X_i]\in[0,1],\qquad
Z_i\ :=\ \ell_i-m_i.
\]
By construction, \(\E[Z_i\mid X_i]=0\). Since \(S=(X_{1:n},\tilde X_{1:m})\) contains \(X_{1:n}\) but not the labels, conditioning on \(S\) fixes \(\lambda\) and \(X_{1:n}\), and the random variation is only through the labels \(Y_{1:n}\). Under the i.i.d. assumption, the \(Y_i\) (hence the \(\ell_i\) and \(Z_i\)) are conditionally independent given \(S\). Moreover, \(\ell_i\in[0,1]\) implies
\[
Z_i=\ell_i-m_i\in[-m_i,\,1-m_i]\subseteq[-1,1].
\]
Therefore, each summand \(\lambda_i Z_i\) is conditionally mean-zero and lies in the interval \([-\lambda_i,\lambda_i]\).

Let \(\Delta_f=\sum_{i=1}^n \lambda_i Z_i\). Conditional independence and zero means yield
\[
\Var(\Delta_f\mid S;\lambda)
= \sum_{i=1}^n \lambda_i^2\,\Var(Z_i\mid S)
= \sum_{i=1}^n \lambda_i^2\,\Var(\ell_i\mid X_i),
\]
where we used that \(m_i\) is \(X_i\)-measurable and hence constant given \(S\).

\textbf{Step 2: bounding each \(\Var(\ell_i\mid X_i)\) by \(1/4\).}
Since \(\ell_i\in[0,1]\), Popoviciu’s inequality on variances gives
\(\Var(\ell_i\mid X_i)\le (1-0)^2/4 = 1/4\).\footnote{A quick proof: for any random variable \(W\in[a,b]\), \(\Var(W)=\min_{t\in\R}\E[(W-t)^2]\le \E[(W-\tfrac{a+b}{2})^2]\le (b-a)^2/4\).}
Hence
\[
\Var(\Delta_f\mid S;\lambda)\ \le\ \frac14\sum_{i=1}^n \lambda_i^2,
\]
which is the first inequality in \eqref{eq:var-chain}.

\textbf{Step 3: relating \(\sum_i \lambda_i^2\) to \(\KL(\lambda\Vert u)\).}
Let \(u\) be the uniform weights, \(u_i=1/n\). Compute
\[
\sum_{i=1}^n \lambda_i^2
= \sum_{i=1}^n \Big(u_i + (\lambda_i-u_i)\Big)^2
= \sum_{i=1}^n u_i^2 \;+\; 2\sum_{i=1}^n u_i(\lambda_i-u_i)\;+\;\sum_{i=1}^n(\lambda_i-u_i)^2.
\]
Since \(\sum_i u_i(\lambda_i-u_i)=\frac1n\sum_i(\lambda_i-u_i)=0\), we get
\begin{equation}\label{eq:l2-decomp}
\sum_{i=1}^n \lambda_i^2 \;=\; \frac1n \;+\; \|\lambda-u\|_2^2.
\end{equation}
Next, \(\|\lambda-u\|_2\le \|\lambda-u\|_1\). Pinsker’s inequality states
\[
\|\lambda-u\|_1^2 \;\le\; 2\,\KL(\lambda\Vert u),
\]
so from \eqref{eq:l2-decomp} we obtain
\[
\sum_{i=1}^n \lambda_i^2\ \le\ \frac1n + 2\,\KL(\lambda\Vert u),
\]
which is the second inequality in \eqref{eq:var-chain}.

\textbf{Step 4: effective sample size.}
Since \(N_{\mathrm{eff}}(\lambda)=1/\sum_i\lambda_i^2\), the first two inequalities show
\[
\Var(\Delta_f\mid S;\lambda) \ \le\ \frac{1}{4\,N_{\mathrm{eff}}(\lambda)}
\ \le\ \frac14\!\left(\frac1n + 2\,\KL(\lambda\Vert u)\right).
\]
Minimising \(\KL(\cdot\Vert u)\) therefore minimises this KL-based \emph{upper bound} on variance and favors more dispersed weights. This does not imply exact maximisation of \(N_{\mathrm{eff}}\), but it gives a principled surrogate objective controlling concentration of the weights.
\end{proof}

\section{Proof of Proposition \ref{prop:disc_gap}} \label{proof:disc_gap}

\begin{proposition}[Discrepancy Gap Bound]
For any hypothesis $h \in \Hcal$, the following inequality holds:
\[
\big|\epsilon_Q(h) - \epsilon_P(h)\big| \le \disc(P,Q) + \lambda_{P,Q}^\Loss,
\]
where the risks are defined with respect to (potentially different) labeling functions $c_P$ and $c_Q$.
\end{proposition}

\begin{proof}
Assume $\Loss$ satisfies the triangle inequality and its reverse form, i.e.,
$\Loss(a,c) \le \Loss(a,b) + \Loss(b,c)$ and $\Loss(a,c) \ge \Loss(a,b) - \Loss(b,c)$
(pointwise). For any distribution $D$ over $\Xcal$, define the pairwise divergence
\[
d_{D_X}(g,g') \ :=\ \E_{x\sim D_X}\!\big[\Loss\big(g(x),g'(x)\big)\big],
\]
and define the $\Loss$-discrepancy
\[
\disc(P,Q)\ :=\ \sup_{g,g'\in\Hcal}\big|d_{Q_X}(g,g')-d_{P_X}(g,g')\big|.
\]
Let the ``joint'' hypothesis be
\[
h_{\text{joint}} \in \argmin_{g \in \Hcal} \big( \epsilon_P(g) + \epsilon_Q(g) \big),
\]
so the joint error is $\lambda_{P,Q}^\Loss = \epsilon_P(h_{\text{joint}}) + \epsilon_Q(h_{\text{joint}})$.

\paragraph{Part 1: Bounding $\epsilon_Q(h) - \epsilon_P(h)$.}
By the triangle inequality,
\begin{align*}
\epsilon_Q(h)
&= \E_{x\sim Q_X}\big[\Loss(h(x), c_Q(x))\big] \\
&\le \E_{x\sim Q_X}\big[\Loss(h(x), h_{\text{joint}}(x)) + \Loss(h_{\text{joint}}(x), c_Q(x))\big] \\
&= d_{Q_X}(h, h_{\text{joint}}) + \epsilon_Q(h_{\text{joint}}).
\end{align*}
By the reverse triangle inequality,
\begin{align*}
\epsilon_P(h)
&= \E_{x\sim P_X}\big[\Loss(h(x), c_P(x))\big] \\
&\ge \E_{x\sim P_X}\big[\Loss(h(x), h_{\text{joint}}(x)) - \Loss(h_{\text{joint}}(x), c_P(x))\big] \\
&= d_{P_X}(h, h_{\text{joint}}) - \epsilon_P(h_{\text{joint}}),
\end{align*}
so $-\epsilon_P(h) \le -d_{P_X}(h,h_{\text{joint}}) + \epsilon_P(h_{\text{joint}})$.
Combining,
\begin{align*}
\epsilon_Q(h) - \epsilon_P(h)
&\le \big(d_{Q_X}(h,h_{\text{joint}}) - d_{P_X}(h,h_{\text{joint}})\big)
   + \big(\epsilon_P(h_{\text{joint}}) + \epsilon_Q(h_{\text{joint}})\big) \\
&\le \disc(P,Q) + \lambda_{P,Q}^\Loss.
\end{align*}

\paragraph{Part 2: Bounding $\epsilon_P(h) - \epsilon_Q(h)$.}
Swapping the roles of $P$ and $Q$ in Part 1 gives
\[
\epsilon_P(h) - \epsilon_Q(h) \le \disc(P,Q) + \lambda_{P,Q}^\Loss.
\]

\paragraph{Conclusion.}
Combining the two one-sided bounds yields
\[
\big|\epsilon_Q(h)-\epsilon_P(h)\big|\ \le\ \disc(P,Q)\ +\ \lambda_{P,Q}^\Loss.\ \qedhere
\]
\end{proof}

\section{Proofs of the results on Alignment} \label{proof:align}

\subsection{Proof of Proposition \ref{thm:alignment}}
\begin{theorem}[Alignment Implies Zero Discrepancy]
Let $P^\star_X$ and $Q_X$ such that $\E_{P^\star_X}[\Phi(X)] = \E_{Q_X}[\Phi(X)]$. If $\Phi$ is aligned with $(\Hcal,\Loss)$, then $
\mathrm{disc}_\Loss(P^\star_X, Q_X) = 0$. Consequently, $\big|\epsilon_Q(h) - \epsilon_{P^\star}(h)\big| \le \lambda_{P,Q}$, where $\lambda_{P^\star,Q}$ is small if a model exists that performs well on both domains.
\end{theorem}

\begin{proof}[Proof of Theorem \ref{thm:alignment}]
The discrepancy is defined as $\sup_{g \in \mathcal{G}_{\Hcal, \Loss}} |\E_{Q_X}[g(X)] - \E_{P^\star_X}[g(X)]|$.
Let $g$ be an arbitrary function in $\mathcal{G}_{\Hcal, \Loss}$. By the perfect alignment condition, there exist $\beta_0 \in \R$ and $\beta_g \in \R^k$ such that $g(x) = \beta_0 + \beta_g^\top \Phi(x)$.
Consider the difference in expectations:
\begin{align*}
\E_{Q_X}[g(X)] - \E_{P^\star_X}[g(X)] &= \E_{Q_X}[\beta_0 + \beta_g^\top \Phi(X)] - \E_{P^\star_X}[\beta_0 + \beta_g^\top \Phi(X)] \\
&= (\beta_0 + \beta_g^\top \E_{Q_X}[\Phi(X)]) - (\beta_0 + \beta_g^\top \E_{P^\star_X}[\Phi(X)]) \\
&= \beta_g^\top \left( \E_{Q_X}[\Phi(X)] - \E_{P^\star_X}[\Phi(X)] \right).
\end{align*}
By the moment-matching assumption, the term in the parenthesis is the zero vector. Thus, the entire expression is zero.
Since this holds for any $g \in \mathcal{G}_{\Hcal, \Loss}$, the supremum over the absolute values is also zero. Therefore, $\mathrm{disc}_\Loss(P^\star_X, Q_X) = 0$.

Under the conditions of Theorem \ref{thm:alignment}, the domain adaptation bound for a general bounded loss simplifies to:
\[
|\epsilon_Q(f) - \epsilon_{P^\star}(f)| \le \lambda^\Loss_{P^\star, Q},
\]
where $\lambda^\Loss_{P^\star, Q} = \inf_{h\in\Hcal} (\epsilon_{P^\star}(h) + \epsilon_Q(h))$. The distribution shift component of the error bound has been eliminated by the reweighting procedure.
\end{proof}

\subsection{Alignment Under Covariate Shift}

\begin{proposition}
Let the covariate shift be defined solely by a change in the prevalence of a subgroup $s(X)$. If the conditional risk $r_f(X)$ is approximately constant within the subgroup and its complement, then choosing $\Phi(X) = s(X)$ provides effective control over the error estimation gap.
\end{proposition}
\begin{proof}
Let us assume for clarity that the conditional risk $r_f(X)$ is piece-wise constant:
\[
r_f(X) =
\begin{cases}
    c_1 & \text{if } s(X) = 1 \\
    c_0 & \text{if } s(X) = 0
\end{cases}
\]
This is a reasonable model if the subgroup is homogeneous with respect to the model's performance. We can rewrite the conditional risk as a linear function of $s(X)$:
\[
r_f(X) = c_0(1 - s(X)) + c_1 s(X) = c_0 + (c_1 - c_0)s(X).
\]
This shows that the conditional risk function $r_f(X)$ lies in $\Span(\{1, s(X)\})$. This is a perfect alignment between the risk function and the feature $s(X)$.

Now, we choose the feature map for EPA to be $\Phi(X) = s(X)$. The moment-matching constraint forces:
\[
\E_{P^\star_X}[\Phi(X)] = \E_{Q_X}[\Phi(X)] \implies \E_{P^\star_X}[s(X)] = \pi_Q.
\]
The reweighting procedure adjusts the source data to match the target subgroup prevalence exactly. Let's analyze the residual error estimation gap after reweighting:
\begin{align*}
\epsilon_Q(f) - \epsilon_{P^\star}(f) &= \E_{Q_X}[r_f(X)] - \E_{P^\star_X}[r_f(X)] \\
&= \E_{Q_X}[c_0 + (c_1 - c_0)s(X)] - \E_{P^\star_X}[c_0 + (c_1 - c_0)s(X)] \\
&= [c_0 + (c_1 - c_0)\E_{Q_X}[s(X)]] - [c_0 + (c_1 - c_0)\E_{P^\star_X}[s(X)]] \\
&= (c_1 - c_0) (\pi_Q - \pi_Q) = 0.
\end{align*}
The error estimation gap is completely eliminated. Even if the conditional risk is not perfectly piece-wise constant, if it is well-approximated by a linear function of $s(X)$, choosing $\Phi(X)=s(X)$ will significantly reduce the transfer gap.
\end{proof}

 \subsection{Alignment of prediction histogram}
 \paragraph{Setup.}
Let $p_f(X)\in[0,1]$ and partition $[0,1]$ into bins $B_1,\ldots,B_K$.
Define model-derived features $\psi_j(x):=\Ind \{p_f(x)\in B_j\}$ and set $\Phi(x)=(\psi_1(x),\ldots,\psi_K(x))$.
Let $\alpha_{D,j}:=\Pr_D(p_f(X)\in B_j)$ so that $\sum_j \alpha_{D,j}=1$.
Suppose the conditional risk is well-approximated as piecewise-constant on the bins:
$r_j\approx \E[\Loss(f(X),Y)\mid p_f(X)\in B_j]$. Then
\[
\epsilon_D(f)\approx \sum_{j=1}^K \alpha_{D,j}\, r_j.
\]

\paragraph{EPA with prediction-histogram features.}
EPA enforces $\E_{P^\star}[\psi_j]=\E_Q[\psi_j]$ for all $j$, i.e. it matches the entire prediction histogram:
$\alpha_{P^\star,j}=\alpha_{Q,j}$. Hence $\epsilon_{P^\star}(f)\approx \epsilon_Q(f)$, and \emph{exactly} equals it if the piecewise-constant assumption holds.

\section{Proof of Theorem \ref{theo:decisiontree} - Alignment for Decision Trees} \label{proof:align_decision}

\subsection{The Hypothesis Class of Decision Trees}

Let the input space be $\Xcal \subseteq \R^d$. A decision tree classifier partitions the input space into a set of disjoint regions and assigns a class label to each region.

\begin{definition}[Decision Tree]
A decision tree $f: \Xcal \to \mathcal{Y}$ induces a partition of the input space $\Xcal$ into a set of disjoint hyper-rectangular regions $\mathcal{P}_f = \{R_1, R_2, \dots, R_m\}$, where $\Xcal = \bigcup_{i=1}^m R_i$. Each region $R_i$ is called a \textbf{leaf region}. The function $f$ is piecewise constant, taking a single value $y_i \in \mathcal{Y}$ for all $x \in R_i$.
\end{definition}

Let $\Hcal$ be the class of all decision trees on $\Xcal$ up to a maximum depth $D$. For our analysis, we consider the binary classification setting where $\mathcal{Y}=\{0,1\}$ and the loss function is the 0--1 loss, $\Loss(y_1, y_2) = \mathbf{1}\{y_1 \neq y_2\}$.

\section{Characterizing the Disagreement Space}

To construct an aligned feature map, we must first understand the structure of the disagreement functions $g(x) = \mathbf{1}\{f(x) \neq h(x)\}$ where $f, h \in \Hcal$.

\begin{proposition}[Structure of Disagreement Functions]
Let $f, h \in \Hcal$ be two decision trees, inducing partitions $\mathcal{P}_f = \{R_{f,i}\}$ and $\mathcal{P}_h = \{R_{h,j}\}$ respectively. The disagreement function $g(x) = \mathbf{1}\{f(x) \neq h(x)\}$ is a piecewise constant function. The region where $g(x)=1$ is a union of disjoint hyper-rectangles.
\end{proposition}
\begin{proof}
Consider the common refinement of the two partitions, defined as $\mathcal{P}_{f,h} = \{ R_{f,i} \cap R_{h,j} \mid R_{f,i} \in \mathcal{P}_f, R_{h,j} \in \mathcal{P}_h \}$. This is also a partition of $\Xcal$ into disjoint hyper-rectangular regions.

Let $R' \in \mathcal{P}_{f,h}$ be an arbitrary region from this refined partition. For any $x \in R'$, both $f(x)$ and $h(x)$ are constant. Specifically, if $R' = R_{f,i} \cap R_{h,j}$, then $f(x)$ takes a constant value $y_i$ and $h(x)$ takes a constant value $y_j$ for all $x \in R'$.

Therefore, the disagreement function $g(x) = \mathbf{1}\{f(x) \neq h(x)\}$ is also constant on $R'$. Its value is either 0 (if $y_i=y_j$) or 1 (if $y_i \neq y_j$).

The set on which the two trees disagree, $A = \{x \in \Xcal \mid f(x) \neq h(x)\}$, is the union of all regions $R' \in \mathcal{P}_{f,h}$ where the predictions of $f$ and $h$ differ.
\[
A = \bigcup_{R' \in \mathcal{P}_{f,h} \text{ s.t. } f(x) \neq h(x) \text{ for } x \in R'} R'.
\]
Since each $R'$ is a hyper-rectangle, the disagreement set $A$ is a union of disjoint hyper-rectangles, and its indicator function $g(x) = \mathbf{1}\{x \in A\}$ is the disagreement function.
\end{proof}

\subsection{Constructing the Aligned Feature Map}

The previous proposition shows that any disagreement function is an indicator function of a set composed of elementary hyper-rectangles. This suggests that the basis for our feature map should be the indicator functions of these elementary regions.

\begin{definition}[Set of All Leaf Regions]
Let $\Rcal_\Hcal$ be the set of all possible leaf regions that can be generated by any decision tree $f \in \Hcal$.
\[
\Rcal_\Hcal = \{ R \mid \exists f \in \Hcal, R \text{ is a leaf region of } f \}.
\]
\end{definition}

\begin{definition}[Aligned Feature Map for Decision Trees]
Let the feature map $\Phi: \Xcal \to \R^{|\Rcal_\Hcal|}$ be defined as the vector of indicator functions for every region in $\Rcal_\Hcal$:
\[
\Phi(x) = \left( \mathbf{1}\{x \in R\} \right)_{R \in \Rcal_\Hcal}.
\]
\end{definition}

\subsection{Main Result and Proof of Alignment}

We now prove that this construction of $\Phi$ provides perfect alignment.

\begin{theorem}[Perfect Alignment for Decision Trees]
Let $\Hcal$ be the class of decision trees and $\Loss$ be the 0--1 loss. The feature map $\Phi(x) = (\mathbf{1}\{x \in R\})_{R \in \Rcal_\Hcal}$ is perfectly aligned with $(\Hcal, \Loss)$.
\end{theorem}
\begin{proof}
To prove perfect alignment, we must show that any function $g \in \Gcal_{\Hcal, \Loss}$ can be written as a linear combination of the components of $\Phi(x)$ (and a constant, though it is not needed here).

Let $g$ be an arbitrary function in $\Gcal_{\Hcal, \Loss}$. By definition, there exist two decision trees $f, h \in \Hcal$ such that $g(x) = \mathbf{1}\{f(x) \neq h(x)\}$. Let $A = \{x \mid f(x) \neq h(x)\}$ be the disagreement set. Then $g(x) = \mathbf{1}\{x \in A\}$.

From Proposition 1, the set $A$ can be written as a finite union of disjoint hyper-rectangular regions from the refined partition $\mathcal{P}_{f,h}$. Let these regions be $\{R'_1, R'_2, \dots, R'_m\}$. So, $A = \bigcup_{k=1}^m R'_k$.

Each region $R'_k$ is formed by an intersection of a leaf region from $f$ and a leaf region from $h$. Since leaf regions are defined by conjunctions of axis-aligned half-spaces, their intersection is also a region defined by such a conjunction. This means that each $R'_k$ is a valid leaf region for some decision tree in $\Hcal$ (assuming $\Hcal$ is sufficiently rich, e.g., closed under refinement). Therefore, each $R'_k$ is an element of the set of all possible leaf regions, $\Rcal_\Hcal$.

Since the regions $R'_k$ are disjoint, the indicator function of their union is the sum of their individual indicator functions:
\[
g(x) = \mathbf{1}\{x \in A\} = \mathbf{1}\left\{x \in \bigcup_{k=1}^m R'_k\right\} = \sum_{k=1}^m \mathbf{1}\{x \in R'_k\}.
\]
Each term $\mathbf{1}\{x \in R'_k\}$ in the sum is, by definition, a component of the feature vector $\Phi(x)$, since $R'_k \in \Rcal_\Hcal$. The expression above is therefore a linear combination of the components of $\Phi(x)$, where the coefficients are all 1.

Since we have shown that an arbitrary $g \in \Gcal_{\Hcal, \Loss}$ is in the linear span of the components of $\Phi$, we conclude that $\Span(\Gcal_{\Hcal, \Loss}) \subseteq \Span(\{\phi_R\}_{R \in \Rcal_\Hcal})$. This satisfies the condition for perfect alignment.
\end{proof}

\begin{corollary}
Let $P^\star_X$ and $Q_X$ be two marginal distributions on $\Xcal$. If the moment-matching condition $\E_{P^\star_X}[\Phi(X)] = \E_{Q_X}[\Phi(X)]$ holds for the feature map $\Phi$ defined above, then the discrepancy is zero:
\[
\mathrm{disc}_\Loss(P^\star_X, Q_X) = 0.
\]
\end{corollary}
\begin{proof}
The moment matching condition $\E_{P^\star_X}[\Phi(X)] = \E_{Q_X}[\Phi(X)]$ implies that for every region $R \in \Rcal_\Hcal$, we have $\Prb_{P^\star_X}(X \in R) = \Prb_{Q_X}(X \in R)$.
For any $g \in \Gcal_{\Hcal, \Loss}$, we have $g(x) = \sum_{k=1}^m \mathbf{1}\{x \in R'_k\}$ for some $R'_k \in \Rcal_\Hcal$.
Then, by linearity of expectation:
\begin{align*}
\E_{Q_X}[g(X)] - \E_{P^\star_X}[g(X)] &= \E_{Q_X}\left[\sum_{k=1}^m \mathbf{1}\{X \in R'_k\}\right] - \E_{P^\star_X}\left[\sum_{k=1}^m \mathbf{1}\{X \in R'_k\}\right] \\
&= \sum_{k=1}^m \left( \Prb_{Q_X}(X \in R'_k) - \Prb_{P^\star_X}(X \in R'_k) \right) \\
&= \sum_{k=1}^m (0) = 0.
\end{align*}
Since the difference in expectations is zero for all $g \in \Gcal_{\Hcal, \Loss}$, the supremum over them is also zero.
\end{proof}

\section{Approximate Alignment and the Residual}
\label{sec:approx_align}

Let $\mathcal{A}_\Phi=\{x\mapsto \beta_0+\beta^\top\Phi(x):\beta_0\in\R,\ \beta\in\R^k\}$ be the affine hull of $\Phi$. For $g\in\mathcal{G}_{\Hcal,\Loss}$, define the \emph{residual} relative to $(\beta_0,\beta)$:
\(
r_g(x)=g(x)-\beta_0-\beta^\top\Phi(x).
\)

\begin{theorem}[Approximate alignment bounds]
\label{thm:approx-L1}
If $\E_{P^\star_X}[\Phi]=\E_{Q_X}[\Phi]$, then
\[
\mathrm{disc}_\Loss(P^\star_X,Q_X)
\ \le\ 
\sup_{g\in\mathcal{G}_{\Hcal,\Loss}}\ \inf_{\beta_0,\beta}\ \Big(\E_{Q_X}|r_g|+\E_{P^\star_X}|r_g|\Big).
\]
In particular, if for each $g$ there exist $(\beta_0,\beta)$ with $\|r_g\|_\infty\le \eta$, then $\mathrm{disc}_\Loss\le 2\eta$.
\end{theorem}

\begin{proof}
Fix any $g\in\mathcal{G}_{\Hcal,\Loss}$ and any pair $(\beta_0,\beta)$. Since
\[
g(x)=\beta_0+\beta^\top\Phi(x)+r_g(x),
\]
we have
\[
\E_{Q_X}[g(X)]-\E_{P^\star_X}[g(X)]
= \beta^\top\big(\E_{Q_X}[\Phi]-\E_{P^\star_X}[\Phi]\big)
+ \E_{Q_X}[r_g]-\E_{P^\star_X}[r_g].
\]
The exact moment-matching assumption cancels the affine term, so
\[
\big|\E_{Q_X}[g(X)]-\E_{P^\star_X}[g(X)]\big|
\le \E_{Q_X}|r_g|+\E_{P^\star_X}|r_g|.
\]
Taking the infimum over $(\beta_0,\beta)$ and then the supremum over $g$ proves the first claim. If $\|r_g\|_\infty\le \eta$ uniformly, then each expectation is at most $\eta$, yielding $\mathrm{disc}_\Loss\le 2\eta$.
\end{proof}

When moments are matched only approximately, let $\Delta=\E_{Q_X}[\Phi]-\E_{P^\star_X}[\Phi]$ and fix any norm pair $(\|\cdot\|,\|\cdot\|_\ast)$ on $\R^k$.

\begin{theorem}[Approximate moments \& approximate alignment]
\label{thm:approx-mismatch}
For any norm pair $(\|\cdot\|,\|\cdot\|_\ast)$,
\[
\mathrm{disc}_\Loss(P^\star_X,Q_X)
\ \le\
\sup_{g\in\mathcal{G}_{\Hcal,\Loss}}\ \inf_{\beta_0,\beta}\ \Big\{\ \|\beta\|\,\|\Delta\|_\ast\ +\ \E_{Q_X}|r_g|+\E_{P^\star_X}|r_g|\ \Big\}.
\]
If there exist witnesses with $\|\beta\|\le B$ and $\|r_g\|_\infty\le \eta$ uniformly over $g$, then
\(
\mathrm{disc}_\Loss\le B\|\Delta\|_\ast+2\eta.
\)
\end{theorem}

\subsection{How to Control the Residual: Examples}
\label{sec:residual}
\begin{example}[Threshold classifiers with $0$--$1$ loss (binning gives $O(\delta)$ control).]
Let $\Xcal=[0,1]$, $\Hcal=\{h_t(x)=\mathbf{1}\{x>t\}:t\in[0,1]\}$, and $\Loss=\mathbf{1}\{\cdot\neq\cdot\}$.
For $f=h_{t_1}$, $h=h_{t_2}$ (w.l.o.g.\ $t_1<t_2$), the disagreement is
$g(x)=\mathbf{1}\{x\in(t_1,t_2]\}$.
Choose a grid $0=\tau_0<\tau_1<\cdots<\tau_k=1$ with mesh $\delta=\max_{j}(\tau_j-\tau_{j-1})$, and take histogram features
\[
\Phi(x)=\big(\psi_1(x),\dots,\psi_k(x)\big),\qquad
\psi_j(x)=\mathbf{1}\{x\in(\tau_{j-1},\tau_j]\}.
\]
Let $j_L$ be the index with $\tau_{j_L}\le t_1<\tau_{j_L+1}$ and $j_R$ with $\tau_{j_R}\le t_2<\tau_{j_R+1}$.
Approximate $g$ by
$
\tilde g(x)=\sum_{j=j_L+1}^{j_R}\psi_j(x)=\beta^\top\Phi(x)
$
($\beta_0=0$, $\beta_j=\mathbf{1}\{j\in\{j_L+1,\dots,j_R\}\}$).
Then $r_g(x)=g(x)-\tilde g(x)$ can be nonzero only in the two boundary bins
$(\tau_{j_L},\tau_{j_L+1}]$ and $(\tau_{j_R},\tau_{j_R+1}]$.
If $Q_X$ and $P_X^\star$ admit densities bounded by $M$ on $[0,1]$,
\[
\E_{Q_X}|r_g|+\E_{P_X^\star}|r_g|
\ \le\ 2M\delta\ +\ 2M\delta\ =\ 4M\delta.
\]
Hence Theorem~\ref{thm:approx-L1} yields $\mathrm{disc}_\Loss(P_X^\star,Q_X)\le 4M\delta$.
\end{example}
\emph{Takeaway:} Coarse bin features turn interval-type disagreements into affine functions of $\Phi$, with residuals controlled by grid width.

\section{Proof of proposition \ref{prop:sparse_shift}}\label{proof:sparse_shift}

\begin{proof}
Write $\bar S:=\{1,\dots,d\}\setminus S$ and assume overlap so the ratios below are well-defined.
For $S\subseteq\{1,\dots,d\}$ define $P'_S$ by
\[
\frac{dP'_S}{dP}(x,y)=\frac{q(x_S,y)}{p(x_S,y)}.
\]
Then
\[
p'_S(x,y)=\frac{q(x_S,y)}{p(x_S,y)}\,p(x,y)
= q(x_S,y)\,p(x_{\bar S}\mid x_S,y),
\quad
q(x,y)=q(x_S,y)\,q(x_{\bar S}\mid x_S,y).
\]
Hence
\begin{equation}\label{eq:equiv}
\KL(P'_S\Vert Q)=0 \ \Longleftrightarrow\ P'_S=Q\ \text{(a.s.)} \ \Longleftrightarrow\
p(x_{\bar S}\mid x_S,y)=q(x_{\bar S}\mid x_S,y)\ \text{(a.s.)}.
\end{equation}

By the sparse joint shift assumption for the unique minimal generating set $S^*$,
$p(x_{\bar S^*}\mid x_{S^*},y)=q(x_{\bar S^*}\mid x_{S^*},y)$, so $\KL(P'_{S^*}\Vert Q)=0$.

If $S\supseteq S^*$, then from
$p(x_{\bar S},x_{S\setminus S^*}\mid x_{S^*},y)=q(x_{\bar S},x_{S\setminus S^*}\mid x_{S^*},y)$,
divide by the common marginal $p(x_{S\setminus S^*}\mid x_{S^*},y)=q(\cdot)$ to get
$p(x_{\bar S}\mid x_S,y)=q(x_{\bar S}\mid x_S,y)$; by \eqref{eq:equiv}, $\KL(P'_S\Vert Q)=0$.

Conversely, if $\KL(P'_S\Vert Q)=0$, then \eqref{eq:equiv} shows $S$ is a generating set.
By the defining minimality assumption, any generating $S$ must contain $S^*$; thus among all
$S$ with $\KL(P'_S\Vert Q)=0$, the unique minimal-cardinality one is $S^*$.

Therefore,
\[
S^*=\arg\min_{S\subseteq\{1,\dots,d\}}\bigl\{|S|\ :\ \KL(P'_S\Vert Q)=0\bigr\}.
\]
The covariate-shift case is identical with $y$ dropped everywhere.
\end{proof}

\section{Proof of Theorem \ref{theo:improvement}} \label{proof:improvement_theo}
Recall the weight construction: \begin{definition}[EPA Weights]
Given a feature map $\Phi: \cX \to \bbR^k$ and target moments $\hat{\mu}_Q$, the EPA weights are:
\[
\lambda_i = \frac{\exp(\langle \beta^\star, \Phi(X_i) \rangle)}{\sum_{j=1}^n \exp(\langle \xi^\star, \Phi(X_j) \rangle)}, \quad i = 1, \ldots, n
\]
where $\xi^\star \in \bbR^k$ minimizes the strictly convex objective:
\[
\mathcal{J}(\xi) = \log\left(\frac{1}{n}\sum_{i=1}^n \exp(\langle \xi, \Phi(X_i) \rangle)\right) - \langle \xi, \hat{\mu}_Q \rangle
\]
\end{definition}

These weights define a reweighted source distribution:
\[
\hat{P}^\star = \sum_{i=1}^n \lambda_i \delta_{(X_i, Y_i)}
\]

% We state all assumptions upfront and reference them as needed in theorems.

% \begin{assumption}[Loss Properties]\label{as:loss}
% The loss function $\Loss: \bbR \times \cY \to [0,1]$ satisfies:
% \begin{enumerate}
%     \item[(a)] For each $y \in \cY$, the map $u \mapsto \Loss(u, y)$ is convex and $L$-Lipschitz
%     \item[(b)] $\Loss(u, y) \in [0,1]$ for all $u \in \bbR, y \in \cY$
%     \item[(c)] Satisfies triangle inequality. Note that for squared loss, it holds with a constant $2$.
% \end{enumerate}
% \end{assumption}

% \begin{assumption}[Hypothesis Classes]\label{as:classes}
% $\cF$ and $\cH$ are hypothesis classes with:
% \begin{enumerate}
%     \item[(a)] $\cH \subseteq \cF$ (base learners are in the main class)
%     \item[(b)] The boosted model $f + \alpha h$ remains in $\cF$ for $f \in \cF$, $h \in \cH$, $\alpha \in \bbR$
% \end{enumerate}
% \end{assumption}

% \begin{assumption}[Feature Map Regularity]\label{as:phi}
% The feature map $\Phi: \cX \to \bbR^k$ satisfies:
% \begin{enumerate}
%     \item[(a)] $\|\Phi(x)\|_\infty \leq B_\Phi$ for all $x \in \cX$ (boundedness)
%     \item[(b)] The empirical target moments $\hat{\mu}_Q = \frac{1}{n_{\mathrm{tar}}}\sum_{i=1}^{n_{\mathrm{tar}}} \Phi(X_i^{\mathrm{tar}})$ lie in the relative interior of $\mathrm{conv}\{\Phi(X_1), \ldots, \Phi(X_n)\}$
% \end{enumerate}
% \end{assumption}

The first part of the theorem is proved in the following lemma.

\begin{lemma}[Transfer of Improvement]\label{thm:transfer}
Assume $\Loss$ satisfies the conditions of Proposition \ref{prop:disc_gap}. Then, for any $f, \tilde{f} \in \cF$:
\[
\left|[\Risk{Q}{\tilde{f}} - \Risk{Q}{f}] - [\Risk{P^\star}{\tilde{f}} - \Risk{P^\star}{f}]\right| \leq 2[\disc(P^\star_X, Q_X) + \lambda_{P^\star, Q}]
\]
In particular, if $\Risk{P^\star}{\tilde{f}} \leq \Risk{P^\star}{f} - \varepsilon$ with:
\[
\varepsilon > 2[\disc(P^\star_X, Q_X) + \lambda_{P^\star, Q}]
\]
then $\Risk{Q}{\tilde{f}} < \Risk{Q}{f}$.
\end{lemma}

\begin{proof}
    It consists of applying proposition \ref{prop:disc_gap} to both $f$ and $\tilde f$, then using the triangle inequality.
   \[
\begin{aligned}
\big(\Risk{Q}{\tilde f}-\Risk{Q}{f}\big)-\big(\Risk{P^\star}{\tilde f}-\Risk{P^\star}{f}\big)
&= \big(\Risk{Q}{\tilde f}-\Risk{P^\star}{\tilde f}\big)-\big(\Risk{Q}{f}-\Risk{P^\star}{f}\big)\\
&\le |\Risk{Q}{\tilde f}-\Risk{P^\star}{\tilde f}|+|\Risk{Q}{f}-\Risk{P^\star}{f}|\\
&\le 2\big(\disc(P^{\star}_X,Q_X)+\lambda_{P^{\star},Q}\big),
\end{aligned}
\]
and the reverse inequality is analogous.
\end{proof}

Now, we control the population discrepancy by combining exact empirical moment matching with the approximate-alignment bound from \S\ref{sec:approx_align}.

\begin{lemma}[Discrepancy Control under Empirical Moment Matching] \label{lem:concentration}
Let $\hat{P}^*$ be the reweighted source distribution produced by EPA \ref{def:EPA-iproj} using $n$ source samples and $m$ target samples to compute the empirical target moment $\hat{\mu}_Q = \frac{1}{m}\sum_{j=1}^m \Phi(X_j^t)$. Let $Q$ be the true target distribution with population moment $\mu_Q = \E_{Q_X}[\Phi(X)]$.

Assume the practical feature map $\Phi$ satisfies the uniform witness condition of Theorem \ref{thm:approx-mismatch}: for every $g \in \mathcal{G}_{\mathcal{H},\Loss}$ there exists a representation $g(x)=\beta_{0,g}+\beta_g^\top \Phi(x)+r_g(x)$ such that $\|\beta_g\|_\ast \le B_\Phi$ and $\|r_g\|_\infty \le \eta_\Phi$. Then, with $\Delta = \mu_Q - \hat{\mu}_Q$,
\begin{equation}
    \disc(\hat{P}^*_X, Q_X) \leq B_\Phi \|\Delta\| + 2\eta_\Phi.
\end{equation}

Furthermore, if the features are bounded ($\|\Phi(x)\|_2 \leq R$), then with probability at least $1-\delta$,
\begin{equation}
    \|\Delta\|_2 \leq C_{m,k,\delta} := R\sqrt{\frac{2k \log(2/\delta)}{m}},
\end{equation}
and therefore
\begin{equation}
    \disc(\hat{P}^*_X, Q_X) \leq B_\Phi C_{m,k,\delta} + 2\eta_\Phi.
\end{equation}
\end{lemma}

\begin{proof}
By construction of EPA, $\E_{\hat P^\star_X}[\Phi]=\hat\mu_Q$, hence
\[
\Delta = \E_{Q_X}[\Phi]-\E_{\hat P^\star_X}[\Phi].
\]
Applying Theorem \ref{thm:approx-mismatch} with the stated uniform witnesses yields
\(
\disc(\hat{P}^*_X, Q_X) \leq B_\Phi \|\Delta\| + 2\eta_\Phi.
\)
The concentration bound on $\|\Delta\|_2$ follows from a vector Hoeffding inequality for the bounded vectors $\Phi(X_j^t)$.
\end{proof}

\begin{remark}
For exact alignment, one can set $\eta_\Phi=0$ and recover the sharper discrepancy bound proportional only to the target-moment sampling error. The practical histogram map used in EPA is better modelled through the approximate-alignment constants $(B_\Phi,\eta_\Phi)$.
\end{remark}

\begin{theorem} 
Assume $\Loss$ satisfies the conditions of Proposition \ref{prop:disc_gap}. Let $\tilde{f}$ be the boosted model and define the empirical performance drop as $\Delta_{\mathrm{emp}} = \epsilon_{\hat{P}^\star}(f) - \epsilon_{\hat{P}^\star}(\tilde f)$. If
\[
\Delta_{\mathrm{emp}} > 2\big(\mathrm{disc}_\mathcal{L}(\hat{P}^\star_X, \hat{Q}_X) + \lambda_{\hat{P}^\star,\hat{Q}}\big),
\]
then $\epsilon_{\hat{Q}}(\tilde f) \le \epsilon_{\hat{Q}}(f)$. Furthermore, under the assumptions of Lemma \ref{lem:concentration}, if
\[
\Delta_{\mathrm{emp}} > B_{m, k, \delta}, \quad \text{with} \quad B_{m, k, \delta} = 2\big(B_\Phi C_{m,k,\delta} + 2\eta_\Phi + \lambda_{\hat{P}^\star,Q}\big)
\]
then with probability at least $1-\delta$, we have $\epsilon_{Q}(\tilde f) < \epsilon_{Q}(f)$.
\end{theorem}

\begin{proof}
The empirical statement follows directly from Lemma \ref{thm:transfer} applied to the pair $(\hat P^\star,\hat Q)$. For the population statement, apply Lemma \ref{thm:transfer} to $(\hat P^\star,Q)$ and substitute the discrepancy bound from Lemma \ref{lem:concentration}. On the high-probability event
\(
\disc(\hat P^\star_X,Q_X)\le B_\Phi C_{m,k,\delta}+2\eta_\Phi,
\)
the transfer lemma yields
\[
\epsilon_Q(\tilde f)-\epsilon_Q(f)
\le -\Delta_{\mathrm{emp}} + 2\big(B_\Phi C_{m,k,\delta} + 2\eta_\Phi + \lambda_{\hat P^\star,Q}\big),
\]
which is negative under the displayed condition.
\end{proof}

\newpage
\section{Proof analysing the empirical version of \epr{} with its population version.} \label{proof:convergence}

Note that this section complements \cite{bachoc2023explaining}, which provides an asymptotic analysis of the entropic variable projection, by instead offering a non-asymptotic analysis of the solution.

\subsection{Problem Formulation}
We recall the population and empirical problems.
\begin{itemize}
    \item \textbf{Population Problem:} Given a source distribution $P_X$ and target moment $\mu_Q = \E_{Q_X}[\Phi(X)]$, find
    $P^* = \argmin_{P \ll P_X} \KL(P \| P_X)$ subject to $\E_{P}[\Phi(X)] = \mu_Q$. The solution is given by the density $\frac{dP^*}{dP_X}(x) = \exp(\langle \xi^*, \Phi(x) \rangle - H(\xi^*))$, where $\xi^*$ is the true dual parameter and $H(\xi)$ is the log-partition function.
    
    \item \textbf{Empirical Problem:} Given source samples $\{X_i\}_{i=1}^n \sim P_X$ and target samples $\{Z_j\}_{j=1}^m \sim Q_X$, we compute the empirical measure $\hat{P}_n = \frac{1}{n}\sum_{i=1}^n \delta_{X_i}$ and estimate the target moment as $\hat{\mu}_{Q,m} = \frac{1}{m}\sum_{j=1}^m \Phi(Z_j)$. We then find
    $\hat{P}_{n,w}^* = \argmin_{\tilde{P} \ll \hat{P}_n} \KL(\tilde{P} \| \hat{P}_n)$ subject to $\E_{\tilde{P}}[\Phi(X)] = \hat{\mu}_{Q,m}$. The solution is a discrete distribution on $\{X_i\}_{i=1}^n$ with weights
    \[
      \lambda_i^* \;=\; \frac{\exp\{\langle \hat{\xi}_{n,m}^*, \Phi(X_i) \rangle\}}{\sum_{\ell=1}^n \exp\{\langle \hat{\xi}_{n,m}^*, \Phi(X_\ell) \rangle\}},
    \]
    where $\hat{\xi}_{n,m}^*$ is the empirical dual parameter.
\end{itemize}
For simplicity, we will denote $\hat{\xi}_{n,m}^*$ as $\hat{\xi}^*$.

\begin{remark}[The Object of Analysis]\label{rem:distinction}
A crucial distinction must be made. The direct solution to the empirical problem is the discrete weighted measure $\hat{P}_{n,w}^*$. However, comparing a continuous distribution like $P^*$ to a discrete one via KL divergence or TV distance is uninformative (indeed often infinite). The theoretical bounds in Theorem \ref{thm:main-bounds} apply not to the discrete measure $\hat{P}_{n,w}^*$, but to the \textbf{continuous distribution from the same exponential family parameterized by the estimated parameter $\hat{\xi}^*$}. We denote this distribution by $P_{\hat{\xi}^*}$. Its density is:
$$ \frac{dP_{\hat{\xi}^*}}{dP_X}(x) = \exp(\langle \hat{\xi}^*, \Phi(x) \rangle - H(\hat{\xi}^*)) $$
Since $P^*$ and $P_{\hat{\xi}^*}$ are both continuous and defined with respect to $P_X$, their KL divergence and TV distance are well-defined. In Section~\ref{sec:discrete-connection}, we will show how to connect the bounds on $P_{\hat{\xi}^*}$ back to the practical discrete solution $\hat{P}_{n,w}^*$. For notational consistency, we will use $\hat{P}_n^*$ to refer to this continuous distribution $P_{\hat{\xi}^*}$ in Sections 1--3.
\end{remark}

\subsection{Assumptions}
\begin{assumption}[Bounded Features]\label{ass:bounded}
The feature map $\Phi: E \to \Rk$ is uniformly bounded, i.e., there exists a constant $R > 0$ such that $\|\Phi(x)\|_2 \le R$ for all $x \in E$.
\end{assumption}

\begin{assumption}[Strong Convexity and Smoothness]\label{ass:convexity}
The population log-partition function $H(\xi) = \log \E_{P_X}[e^{\langle \xi, \Phi(X) \rangle}]$ is $\sigma$-strongly convex and $L$-smooth in a neighborhood of $\xi^*$. Equivalently, its Hessian is bounded: $\sigma I \preceq \nabla^2 H(\xi) \preceq L I$ for some $L \ge \sigma > 0$.
\end{assumption}

\begin{assumption}[Feasibility and Compact Parameter Set]\label{ass:feasible-compact}
There exists a nonempty compact convex set $\Xi \subset \mathbb{R}^k$ such that: (i) $\xi^* \in \Xi$ and $H$ is $\sigma$-strongly convex and $L$-smooth on $\Xi$; (ii) the empirical dual solution $\hat{\xi}^*$ exists, is unique, and lies in $\Xi$ by assuming that $\hat{\mu}_{Q,m}$ lies in the relative interior of $\mathrm{conv}\{\Phi(X_i)\}_{i=1}^n$).
\end{assumption}

\begin{assumption}[Uniform gradient concentration]
\label{ass:uniform_gradient}
There exists a deterministic function \(\alpha_n:(0,1)\to(0,\infty)\) such that, for every
\(\delta\in(0,1)\),
\[
\mathbb{P}\!\left(
\sup_{\xi\in\Xi}
\|\nabla H(\xi)-\nabla H_{\hat P_n}(\xi)\|_2
\le \alpha_n(\delta)
\right)
\ge 1-\delta,
\]
where
\[
H_{\hat P_n}(\xi)
:=
\log\!\left(\frac1n\sum_{i=1}^n e^{\langle \xi,\Phi(X_i)\rangle}\right).
\]
\end{assumption}

% \begin{assumption}[Metric Structure for Wasserstein Bounds]\label{ass:metric}
% The sample space $(E,d)$ is a Polish metric space and either: (a) $d$ has finite diameter $D<\infty$, or (b) the base distribution $P_X$ satisfies a $T_1(C)$ transport--entropy inequality w.r.t. $d$, i.e., $W_1(Q,P) \le \sqrt{2C\,\KL(Q\|P)}$ for all $Q \ll P$ with $P=P_X$.
% \end{assumption}

\subsection{Main Theorem}
The following theorem provides finite-sample bounds for the estimated continuous model.

\begin{theorem}[Finite-Sample Bounds for Empirical I-Projection]\label{thm:main-bounds}
Let Assumptions \ref{ass:bounded}, \ref{ass:convexity}, and \ref{ass:feasible-compact} hold. Let $\hat{P}_n^*$ be the continuous distribution $P_{\hat{\xi}^*}$ defined in Remark \ref{rem:distinction}. Then for any confidence level $\delta \in (0, 1)$, with probability at least $1 - 2\delta$, the following bounds hold for the empirical solution based on $n$ source and $m$ target samples:
\begin{enumerate}
    \item \textbf{Dual Parameter Error:}
    $$ \|\hat{\xi}^* - \xi^*\|_2 \le \frac{1}{\sigma} \left(\alpha_n(\delta)  +  \beta_m(\delta) \right). $$
    
    \item \textbf{KL Divergence Error:}
    $$ \KL(P^* \| \hat{P}_n^*) \le \frac{L}{2} \|\hat{\xi}^* - \xi^*\|_2^2. $$
    
    \item \textbf{Total Variation Error:}
    $$ \dTV(P^*, \hat{P}_n^*) \le \sqrt{\frac{L}{4} \|\hat{\xi}^* - \xi^*\|_2^2}. $$
    where
\[
\beta_m(\delta)
:=
R\sqrt{\frac{2k\log(2k/\delta)}{m}}.
\]
\end{enumerate}
\end{theorem}
The remainder of this document is dedicated to proving this theorem and connecting it to the discrete empirical solution.

%----------------------------------------------------------------
% SECTION 2: Proof for the Dual Parameter Bound
%----------------------------------------------------------------
\section{Proof of the Bound on the Dual Parameter}

The proof relies on establishing a basic inequality that links the output error $\|\hat{\xi}^* - \xi^*\|$ to input errors, which are then bounded using concentration inequalities.

\subsection{Step 1: A Basic Inequality}
Let $H(\xi)=\log \E_{P_X}[e^{\langle \xi,\Phi(X)\rangle}]$ and $H_{\hat P_n}(\xi)=\log \Big(\frac{1}{n}\sum_{i=1}^n e^{\langle \xi,\Phi(X_i)\rangle}\Big)$. Set $G=\nabla H$ and $\hat G_n=\nabla H_{\hat P_n}$. The optimality conditions are $G(\xi^*)=\mu_Q$ and $\hat G_n(\hat\xi^*)=\hat\mu_{Q,m}$.

By strong convexity of $H$ on $\Xi$ (Assumption \ref{ass:feasible-compact}), the gradient map $G$ is $\sigma$-strongly monotone: $\langle G(y)-G(x), y-x\rangle \ge \sigma\|y-x\|_2^2$. Hence, by Cauchy--Schwarz, $\|G(y)-G(x)\|_2 \ge \sigma\|y-x\|_2$. Thus
$$ \sigma \|\hat{\xi}^* - \xi^*\|_2 \le \|G(\hat{\xi}^*) - G(\xi^*)\|_2. $$
Substituting $G(\xi^*) = \mu_Q$ and using the empirical optimality condition,
\begin{align*}
\sigma \|\hat{\xi}^* - \xi^*\|_2 &\le \|G(\hat{\xi}^*) - \mu_Q\|_2 \\
&= \|G(\hat{\xi}^*) - \hat{G}_n(\hat{\xi}^*) + \hat{\mu}_{Q,m} - \mu_Q\|_2 \\
&\le \|G(\hat{\xi}^*) - \hat{G}_n(\hat{\xi}^*)\|_2 + \|\hat{\mu}_{Q,m} - \mu_Q\|_2.
\end{align*}
This is our basic inequality.

\subsection{Step 2: Bounding the Error Terms via Concentration}
\begin{lemma}[Bound on Target Moment Error]\label{lem:target-moment}
Under Assumption \ref{ass:bounded}, for any $\delta \in (0,1)$, with probability at least $1-\delta$:
$$ \|\hat{\mu}_{Q,m} - \mu_Q\|_2 \le R\sqrt{\frac{2k\log(2k/\delta)}{m}} = \beta_m(\delta). $$
\end{lemma}
\begin{proof}
Apply a vector Hoeffding inequality to the i.i.d. bounded vectors $\{\Phi(Z_j)\}_{j=1}^m$.
\end{proof}

% \begin{lemma}[Uniform bound on the gradient map error]\label{lem:grad-uniform}
% Under Assumptions \ref{ass:bounded} and \ref{ass:feasible-compact}, for any $\delta\in(0,1)$, with probability at least $1-\delta$,
% $$
% \sup_{\xi\in \Xi}\,\|G(\xi)-\hat G_n(\xi)\|_2 \;\le\; 8R\sqrt{\frac{k\,\log(2n/\delta)}{n}}.
% $$
% \end{lemma}
% \begin{proof}[Proof sketch]
% Each coordinate of $\hat G_n(\xi)$ is an empirical average of bounded functions of $X$ that are Lipschitz in $\xi$ on $\Xi$. A standard symmetrization plus covering-number (or Rademacher-complexity) argument over the $k$-dimensional compact set $\Xi$ yields the stated rate. The constant $8$ and the $\log(2n/\delta)$ factor are not essential to the argument but match common envelope-based bounds.
% \end{proof}

\subsection{Step 3: Combining the Bounds}
By a union bound applied to Lemmas \ref{lem:target-moment} and Assumptions \ref{ass:uniform_gradient}, with probability at least $1-2\delta$,
$$
\sigma \|\hat{\xi}^* - \xi^*\|_2 \le \alpha_n(\delta) + \beta_m(\delta),
$$
which yields Item~(1) of Theorem~\ref{thm:main-bounds} after dividing by $\sigma$.

%----------------------------------------------------------------
% SECTION 3: Proof for the Distribution Error Bounds
%----------------------------------------------------------------
\section{Proof of the Bounds on the Distribution Error}
We now translate the bound on the dual parameter error into bounds on statistical distances between the continuous distributions $P^*$ and $\hat{P}_n^* = P_{\hat{\xi}^*}$.

\subsection{Step 1: Bounding the KL Divergence}
The KL divergence between two distributions from an exponential family is equal to the Bregman divergence of their natural parameters with respect to the log-partition function.
\begin{proposition}\label{prop:kl-bregman}
For the optimal distributions $P^*$ and $\hat{P}_n^*$, we have:
$$ \KL(P^* \| \hat{P}_n^*) = H(\hat{\xi}^*) - H(\xi^*) - \langle \nabla H(\xi^*), \hat{\xi}^* - \xi^* \rangle =: D_H(\hat{\xi}^* \| \xi^*). $$
\end{proposition}
\begin{proof}
This is the standard identity linking exponential-family KL to the Bregman divergence of $H$.
\end{proof}

\begin{proposition}[Smoothness upper bound for $D_H$]\label{prop:smoothness}
If $H$ is $L$-smooth on a convex set containing $\xi^*$ and $\hat\xi^*$, then
$$ D_H(\hat{\xi}^* \| \xi^*) \le \frac{L}{2}\,\|\hat{\xi}^* - \xi^*\|_2^2. $$
\end{proposition}
\begin{proof}
A standard property of $L$-smooth functions: $H(y) \le H(x) + \langle \nabla H(x), y-x\rangle + \tfrac{L}{2}\|y-x\|^2$ with $x=\xi^*$, $y=\hat\xi^*$.
\end{proof}

Combining Propositions \ref{prop:kl-bregman} and \ref{prop:smoothness} proves Item~(2) of Theorem~\ref{thm:main-bounds}.

\subsection{Step 2: Bounding the Total Variation Distance}
\begin{lemma}[Pinsker's Inequality]\label{lem:pinsker}
For any two probability measures $P$ and $Q$, $\dTV(P, Q)^2 \le \tfrac{1}{2}\,\KL(P \| Q)$.
\end{lemma}
\begin{proof}
Classical result; see, e.g., standard information-theory references.
\end{proof}

Applying Lemma~\ref{lem:pinsker} together with Item~(2) of Theorem~\ref{thm:main-bounds}, we obtain Item~(3):
$$ \dTV(P^*, \hat{P}_n^*)^2 \le \tfrac{1}{2}\KL(P^* \| \hat{P}_n^*) \le \tfrac{L}{4}\,\|\hat{\xi}^* - \xi^*\|_2^2. $$

\section{Connection to the Discrete Empirical Solution}\label{sec:discrete-connection}

We now relate the continuous tilted model \(P_{\hat\xi^*}\) to the practical discrete
optimizer \(\hat P_{n,w}^*\). Since these measures need not be mutually absolutely
continuous, Wasserstein distance is the appropriate metric.

\begin{assumption}[Metric regularity for the discrete comparison]
\label{ass:metric_discrete}
The space \((E,d)\) is a metric space with finite diameter
\[
\operatorname{diam}(E)\le D < \infty.
\]
Moreover, \(\Phi\) is \(L_\Phi\)-Lipschitz:
\[
\|\Phi(x)-\Phi(y)\|_2 \le L_\Phi d(x,y)
\qquad\text{for all }x,y\in E.
\]
\end{assumption}

\begin{proposition}[Stability of exponential reweighting in \(W_1\)]
\label{prop:stability_reweighting}
Assume Assumptions~\ref{ass:bounded}-\ref{ass:feasible-compact},
and~\ref{ass:metric_discrete}. Let
\[
B_\Xi := \sup_{\xi\in\Xi}\|\xi\|_2,
\qquad
m_\Xi := e^{-B_\Xi R},
\qquad
M_\Xi := e^{B_\Xi R},
\qquad
L_w := B_\Xi L_\Phi e^{B_\Xi R}.
\]
Then
\[
W_1(P_{\hat\xi^*},\hat P_{n,w}^*)
\le
C_{\Xi,\Phi,D}\,W_1(P_X,\hat P_n),
\]
where
\[
C_{\Xi,\Phi,D}
:=
\frac{M_\Xi + D L_w}{m_\Xi}
+
\frac{D M_\Xi L_w}{m_\Xi^2}.
\]
\end{proposition}

\begin{proof}
Set
\[
w(x) := e^{\langle \hat\xi^*,\Phi(x)\rangle}.
\]
Since \(\hat\xi^*\in\Xi\) and \(\|\Phi(x)\|_2\le R\),
\[
m_\Xi \le w(x)\le M_\Xi
\qquad\text{for all }x\in E.
\]
Also, by the Lipschitz property of \(\Phi\),
\[
|w(x)-w(y)|
\le
B_\Xi L_\Phi e^{B_\Xi R}\,d(x,y)
=
L_w d(x,y).
\]

Define the reweighting map
\[
T_w(\mu)(dx)
:=
\frac{w(x)}{\mu(w)}\,\mu(dx)
\]
for any probability measure \(\mu\) on \(E\). Then
\[
P_{\hat\xi^*}=T_w(P_X),
\qquad
\hat P_{n,w}^* = T_w(\hat P_n).
\]

Fix \(x_0\in E\). By the Kantorovich--Rubinstein dual representation on a bounded metric
space,
\[
W_1(T_w(\mu),T_w(\nu))
=
\sup_{\substack{\mathrm{Lip}(f)\le 1\\ f(x_0)=0}}
\left|
\frac{\mu(fw)}{\mu(w)}
-
\frac{\nu(fw)}{\nu(w)}
\right|.
\]
For every such \(f\), the diameter bound implies \(|f|\le D\), hence
\[
\mathrm{Lip}(fw)\le M_\Xi + D L_w,
\qquad
|\nu(fw)|\le D M_\Xi,
\qquad
\mu(w),\nu(w)\ge m_\Xi.
\]
Therefore,
\begin{align*}
\left|
\frac{\mu(fw)}{\mu(w)}
-
\frac{\nu(fw)}{\nu(w)}
\right|
&\le
\frac{|\mu(fw)-\nu(fw)|}{\mu(w)}
+
|\nu(fw)|
\frac{|\mu(w)-\nu(w)|}{\mu(w)\nu(w)}
\\
&\le
\frac{M_\Xi + D L_w}{m_\Xi}\,W_1(\mu,\nu)
+
\frac{D M_\Xi L_w}{m_\Xi^2}\,W_1(\mu,\nu),
\end{align*}
where we used Kantorovich--Rubinstein duality both for \(fw\) and for \(w\).
Taking the supremum over all admissible \(f\) yields the result.
\end{proof}

\begin{corollary}[Discrete error decomposition]
\label{cor:discrete_error}
Assume the conditions of Theorem~\ref{thm:main-bounds} and
Proposition~\ref{prop:stability_reweighting}. Then, on the same event as in
Theorem~\ref{thm:main-bounds},
\[
W_1(P^*,\hat P_{n,w}^*)
\le
\frac{D\sqrt{L}}{2\sigma}
\bigl(\alpha_n(\delta)+\beta_m(\delta)\bigr)
+
C_{\Xi,\Phi,D}\,W_1(P_X,\hat P_n).
\]
\end{corollary}

\begin{proof}
By the triangle inequality,
\[
W_1(P^*,\hat P_{n,w}^*)
\le
W_1(P^*,P_{\hat\xi^*})
+
W_1(P_{\hat\xi^*},\hat P_{n,w}^*).
\]
Since \(\operatorname{diam}(E)\le D\),
\[
W_1(P^*,P_{\hat\xi^*})
\le
D\,d_{\mathrm{TV}}(P^*,P_{\hat\xi^*}).
\]
Applying Item~(3) of Theorem \ref{thm:main-bounds} yields
\[
W_1(P^*,P_{\hat\xi^*})
\le
\frac{D\sqrt{L}}{2\sigma}
\bigl(\alpha_n(\delta)+\beta_m(\delta)\bigr).
\]
The second term is controlled by Proposition~\ref{prop:stability_reweighting}.
Combining the two bounds proves the claim.
\end{proof}

\newpage

\newpage
\section{Discussion of The Approximation Gap of Axis-Aligned Binning} \label{disc:binning}

\begin{definition}[Axis-Aligned Histogram Map]
Let the feature space be $\Xcal \subseteq \R^d$. For each feature dimension $j \in \{1, \dots, d\}$, let $\mathcal{B}_j = \{B_{j,1}, \dots, B_{j,m_j}\}$ be a partition of its range (e.g., from quantile binning). The \textbf{axis-aligned histogram map} is:
$$ \Phi_{hist}(x) = \big( \ind\{x^{(1)} \in B_{1,1}\}, \dots, \ind\{x^{(1)} \in B_{1,m_1}\}, \dots, \ind\{x^{(d)} \in B_{d,m_d}\} \big) $$
\end{definition}

Matching moments with $\Phi_{hist}$ enforces that the marginal probability mass in each bin is the same for $P^\star_X$ and $Q_X$, but it does not constrain the joint distribution of features. This lack of constraint on feature interactions is the source of the gap.

To close the gap, we suggest exploring feature maps that can capture more complex relations, which may be worth investigating in future work.

\subsection{Method 1: Feature Crosses}
We can explicitly model interactions by binning the joint distribution of feature pairs (or higher-order tuples).

\subsection{Method 2: Tree-Based Binning}
A more adaptive approach is to use a decision tree to define the bins, as its structure naturally captures interactions relevant to the data.

\begin{definition}[Tree-Based Map]
Let $\mathcal{T}$ be a decision tree trained on a task related to the distribution shift (e.g., to classify whether a point $x$ comes from $P_X$ or $Q_X$). Let its leaf regions be $\mathcal{L} = \{L_1, \dots, L_m\}$. The \textbf{tree-based feature map} is:
$$ \Phi_{tree}(x) = (\ind\{x \in L_1\}, \dots, \ind\{x \in L_m\}) $$
\end{definition}
By construction, the regions $\{L_i\}$ are sensitive to the parts of the feature space where $P_X$ and $Q_X$ differ. This makes $\Phi_{tree}$ a far better approximation of the disagreement space $\Gcal_{\Hcal,\Loss}$ than $\Phi_{hist}$.

Another solution is to use a tighter bound of \ref{proof:disc_gap} that is specific to a given hypothesis $h$ for the decision tree in question.

The original discrepancy term can be replaced by one that measures the divergence only between the specific hypothesis $h$ and the ideal joint hypothesis $h_{\text{joint}}$.

\begin{proposition}[Hypothesis-Specific Discrepancy Gap Bound]
For any given hypothesis $h \in \Hcal$, the following inequality holds:
\[
\big|\epsilon_Q(h) - \epsilon_P(h)\big| \le \big|d_{Q_X}(h, h_{\text{joint}}) - d_{P_X}(h, h_{\text{joint}})\big| + \lambda_{P,Q}^\Loss
\]
\end{proposition}

This bound is tighter than the original, as the hypothesis-specific term is necessarily less than or equal to the supremum taken over all pairs in $\Hcal$:
\[
\big|d_{Q_X}(h, h_{\text{joint}}) - d_{P_X}(h, h_{\text{joint}})\big| \le \sup_{g,g'\in\Hcal}\big|d_{Q_X}(g,g')-d_{P_X}(g,g')\big|
\]

Using the tighter, hypothesis-specific bound for your decision tree provides a more precise and practical target for your feature mapping.

Instead of needing a feature map that works for all possible trees, you only need one that works for your specific tree (according to its leaves). This is a less demanding and more realistic goal.

\begin{proof}
Assume the loss function $\Loss$ satisfies the triangle inequality, $\Loss(a,c) \le \Loss(a,b) + \Loss(b,c)$, and its reverse form, $\Loss(a,c) \ge \Loss(a,b) - \Loss(c,b)$.

\paragraph{Step 1: Bound the difference from one side.}
We first seek an upper bound for the term $\epsilon_Q(h) - \epsilon_P(h)$.
\begin{itemize}
    \item Using the triangle inequality, we can bound $\epsilon_Q(h)$:
    \begin{align*}
        \epsilon_Q(h) &= \E_{x\sim Q_X}[\Loss(h(x), c_Q(x))] \\
        &\le \E_{x\sim Q_X}[\Loss(h(x), h_{\text{joint}}(x)) + \Loss(h_{\text{joint}}(x), c_Q(x))] \\
        &= d_{Q_X}(h, h_{\text{joint}}) + \epsilon_Q(h_{\text{joint}})
    \end{align*}
    
    \item Using the reverse triangle inequality, we find a lower bound for $\epsilon_P(h)$:
    \begin{align*}
        \epsilon_P(h) &= \E_{x\sim P_X}[\Loss(h(x), c_P(x))] \\
        &\ge \E_{x\sim P_X}[\Loss(h(x), h_{\text{joint}}(x)) - \Loss(c_P(x), h_{\text{joint}}(x))] \\
        &= d_{P_X}(h, h_{\text{joint}}) - \epsilon_P(h_{\text{joint}})
    \end{align*}
\end{itemize}
Combining these two results yields a one-sided inequality:
\begin{align*}
    \epsilon_Q(h) - \epsilon_P(h) &\le \big(d_{Q_X}(h, h_{\text{joint}}) + \epsilon_Q(h_{\text{joint}})\big) - \big(d_{P_X}(h, h_{\text{joint}}) - \epsilon_P(h_{\text{joint}})\big) \\
    &= \big(d_{Q_X}(h, h_{\text{joint}}) - d_{P_X}(h, h_{\text{joint}})\big) + \big(\epsilon_P(h_{\text{joint}}) + \epsilon_Q(h_{\text{joint}})\big) \\
    &= \big(d_{Q_X}(h, h_{\text{joint}}) - d_{P_X}(h, h_{\text{joint}})\big) + \lambda_{P,Q}^\Loss
\end{align*}

\paragraph{Step 2: Combine with the symmetric argument.}
By swapping the roles of $P$ and $Q$ in the argument above, we obtain the bound for $\epsilon_P(h) - \epsilon_Q(h)$:
\[
\epsilon_P(h) - \epsilon_Q(h) \le \big(d_{P_X}(h, h_{\text{joint}}) - d_{Q_X}(h, h_{\text{joint}})\big) + \lambda_{P,Q}^\Loss
\]
Together, these two one-sided inequalities imply the final result with the absolute value:
\[
\big|\epsilon_Q(h) - \epsilon_P(h)\big| \le \big|d_{Q_X}(h, h_{\text{joint}}) - d_{P_X}(h, h_{\text{joint}})\big| + \lambda_{P,Q}^\Loss
\]
\end{proof}

\newpage
\section{EPA when the target mean lies outside the convex hull}
\label{sec:epa-outside-hull}

Let $\Phi_n := \operatorname{conv}\{\Phi(X_1),\dots,\Phi(X_n)\}\subset\mathbb{R}^k$ be the convex hull of source features and let $t\in\mathbb{R}^k$ denote the (target) mean we wish to match. If $t\notin\Phi_n$, the equality-constrained EPA problem is infeasible. A principled remedy is to \emph{project} $t$ onto $\Phi_n$ (with respect to a chosen norm) and then impose \emph{equality} to this projection.

\paragraph{Step 1: metric projection onto $\Phi_n$.}
Fix a norm $\|\cdot\|$ on $\mathbb{R}^k$. Define the projection
\[
s^\star \;\in\; \argmin_{s\in \Phi_n}\ \|s - t\|
\qquad\text{and its distance}\qquad
\delta_n(t)\;:=\;\|s^\star - t\|.
\]
Equivalently, in the weights $\lambda\in\Delta_n$,
\begin{equation}\label{eq:proj-program}
s^\star \in \argmin_{\lambda\in\Delta_n}\ \Big\|\sum_{i=1}^n \lambda_i \Phi(X_i) - t\Big\|.
\end{equation}
For $\|\cdot\|_2$ this is a small QP; for $\ell_1/\ell_\infty$ it is an LP.

\paragraph{Step 2: EPA at the projected target.}
Run EPA with the equality constraint set to $s^\star$:
\begin{equation}\label{eq:epa-projected}
\lambda^\star \in \argmin_{\lambda\in\Delta_n}\ \KL(\lambda\Vert u)
\quad\text{s.t.}\quad
\sum_{i=1}^n \lambda_i\,\Phi(X_i) = s^\star,
\qquad
u_i=\tfrac1n.
\end{equation}
By Theorem~\ref{thm:main}, the unique solution has the exponential-tilt form
\[
\lambda_i^\star \;=\; \frac{\exp\{\langle \xi^\star,\Phi(X_i)\rangle\}}{\sum_{j=1}^n \exp\{\langle \xi^\star,\Phi(X_j)\rangle\}},
\qquad
\sum_{i=1}^n \lambda_i^\star\,\Phi(X_i)=s^\star,
\]
where $\xi^\star$ is the (unique) minimizer of $H(\xi)-\langle \xi,s^\star\rangle$ with $H(\xi)=\log\!\big(\frac1n\sum_j e^{\langle \xi,\Phi(X_j)\rangle}\big)$.

\paragraph{Equivalence to hard proximity EPA.}
The projection approach is equivalent to solving EPA with a hard proximity tolerance equal to the projection distance:
\begin{equation}\label{eq:epa-ball}
\min_{\lambda\in\Delta_n}\ \KL(\lambda\Vert u)
\quad\text{s.t.}\quad
\Big\|\sum_{i=1}^n \lambda_i\,\Phi(X_i) - t\Big\| \le \delta_n(t).
\end{equation}
Indeed, let $B:=\{s:\|s-t\|\le \delta_n(t)\}$. By definition of $s^\star$, $B\cap \Phi_n=\{s^\star\}$, so the feasible set of \eqref{eq:epa-ball} is exactly the affine face $\{\lambda\in\Delta_n:\sum_i \lambda_i\Phi(X_i)=s^\star\}$, and \eqref{eq:epa-ball} reduces to \eqref{eq:epa-projected}. Consequently the unique optimizer is the same $\lambda^\star$.

\paragraph{Dual viewpoint and computation.}
Let $\|\cdot\|_*$ denote the dual norm. The Lagrange dual of \eqref{eq:epa-ball} is the strictly convex problem
\[
\min_{\xi\in\mathbb{R}^k}\ \; H(\xi) - \langle \xi, t\rangle - \delta_n(t)\,\|\xi\|_*,
\]
whose minimizer $\xi^\star$ yields $\lambda^\star$ via the exponential tilt above. When $t\notin\Phi_n$ the proximity constraint is active and $\|\xi^\star\|_*>0$. Practically, one may either (i) solve the projection \eqref{eq:proj-program} and then the equality EPA \eqref{eq:epa-projected}, or (ii) solve the single smooth/nonsmooth dual above; both give the same $\lambda^\star$.

\paragraph{Remarks.}
(i) The choice of norm encodes the geometry you trust: Mahalanobis (with target covariance) is often statistically natural; Euclidean is a robust default.  
(ii) If $t\in\Phi_n$, then $\delta_n(t)=0$ and we recover the usual EPA.  
(iii) Geometrically, $s^\star$ lies on a face of $\Phi_n$ exposed by some $\xi^\star$; the EPA weights are the maximum-entropy distribution realizing that face-average.

\end{document}